\documentclass{article}

\usepackage{microtype}
\usepackage{graphicx}
\usepackage{subcaption}
\usepackage{booktabs} % for professional tables

\usepackage{hyperref}

\usepackage[toc,page,header]{appendix}
\usepackage{titletoc}
\usepackage[preprint]{icml2027}

\usepackage{amsmath}
\usepackage{amssymb}
\usepackage{mathtools}
\usepackage{amsthm}

\newcommand{\revision}[1]{\textcolor{black}{#1}}
\newcommand{\newrevision}[1]{\textcolor{black}{#1}}
\usepackage[capitalize,noabbrev]{cleveref}
\usepackage{etoc}
\theoremstyle{plain}
\newtheorem{theorem}{Theorem}[section]
\newtheorem{hypothesis}{Hypothesis}

\newtheorem{lemma}[theorem]{Lemma}

\theoremstyle{definition}

\theoremstyle{remark}

\usepackage{pifont}

\usepackage{bm}
\usepackage{soul}
\usepackage{multirow}
\usepackage{booktabs} 
\usepackage{makecell}
\usepackage{CJKutf8}

\usepackage{wrapfig}

\usepackage{amsthm,thmtools, thm-restate}

\usepackage[table]{xcolor}
\definecolor{lightgrey}{gray}{0.9}
\usepackage[framemethod=TikZ]{mdframed}
\mdfdefinestyle{hypobox}{
    backgroundcolor=gray!15, % Light grey
    linecolor=gray!15,       % Matches background for a seamless look
    innerleftmargin=10pt,
    innerrightmargin=10pt,
    innertopmargin=10pt,
    innerbottommargin=10pt,
    roundcorner=5pt          % Optional: adds slight rounding
}

\usepackage[most]{tcolorbox}
\usepackage{amssymb}
\usepackage{xcolor}

\usepackage[textsize=tiny]{todonotes}

\icmltitlerunning{Representation Misdirection elicits controllable side behaviors and capabilities}

\begin{document}

\twocolumn[
  \icmltitle{Beyond forgetting: Representation Misdirection elicits \\controllable side behaviors and capabilities}

  % the [accepted] option to the icml2027 package.


  \icmlsetsymbol{equal}{*}

  \begin{icmlauthorlist}
    \icmlauthor{Dang Huu-Tien}{equal,yyy}
    \icmlauthor{The-Hai Nguyen}{equal,yyy}
    \icmlauthor{Phuong Dinh}{yyy}
    \icmlauthor{Minh Phuong Nguyen}{yyy}
    \icmlauthor{Anh Bui}{sch}\\
    \icmlauthor{Hoang Thanh-Tung}{aaa}
    \icmlauthor{Le-Minh Nguyen}{yyy}
    \icmlauthor{Naoya Inoue}{yyy,comp}
  \end{icmlauthorlist}

  \icmlaffiliation{yyy}{JAIST}
  \icmlaffiliation{comp}{RIKEN}
  \icmlaffiliation{aaa}{VNU-UET}
  \icmlaffiliation{sch}{Monash University}

  \icmlcorrespondingauthor{Dang Huu-Tien}{tiendh@jaist.ac.jp}

  \icmlkeywords{Machine Learning, ICML}

  \vskip 0.3in
]



\printAffiliationsAndNotice{}  % no special notice (required even if empty)

\begin{abstract}
We consider Representation Misdirection (RM), a simple unlearning mechanism for large language models (LLMs) that achieves unlearning by redirecting the latent representations of forget samples toward a target vector. 
Despite being important, the role of the target vector used in RM remains insufficiently explored. 
Here, we approach and revisit RM through the lens of the Linear Representation Hypothesis (LRH). 
Specifically, if one can identify a one-dimensional representation corresponding to a high-level concept, the LRH enables linear operations on this concept vector within the representation space of the forget samples. 
From this perspective, we hypothesize that, beyond forgetting, machine unlearning via RM elicits  \emph{controllable side behaviors and capabilities} corresponding to the high-level concept. 
We empirically validate this hypothesis across a wide range of concepts and tasks, including controlling unlearned models' truthfulness, sentiment, stereotypical bias, refusal, language, and instruction-following.
Our findings reveal that this phenomenon could be either a hidden risk if misused or a mechanism that can be harnessed for developing unlearned models that require stronger capabilities and controllable behaviors. 
% Our implementation is available \href{https://github.com/huutiendang/llmu-analysis}{here}. 
\end{abstract}

\section{Introduction}
An important factor underlying the LLMs' knowledge and superior capabilities is the use of massive training corpora crawled from the web. However, the use of such massive corpora in training has raised critical concerns regarding the inclusion of private and sensitive information, copyrighted material, and content promoting malicious behaviors. Machine unlearning (MU;~\citet{7163042, bourtoule2021machine, 10.1145/3603620, nguyen2025survey, liu2025rethinking, barez2025open, ren2025sok}), a post-training paradigm, has been proposed to selectively unlearn the model's target knowledge and harmful capabilities while preserving the model's general knowledge and capabilities. 

Representation Misdirection (RM;~\citet{li2024wmdp}) is a simple yet effective MU mechanism for LLMs that redirects the latent representations of forget-samples (samples in the forget-set) toward a \emph{target vector}, while preserving the latent representations of retain-samples (samples in the retain-set).
In current works, the target vector is often chosen as a fixed, predefined \emph{random} vector~\citep{li2024wmdp, dang2025effects}. 
Prior work has argued that using a random vector as the target vector can inadvertently cause the unlearned model to produce incoherent outputs. 
% Prior work has argued that using a random vector as the target direction may lead to undesirable behaviors~\citep{shen2025llm}.
% However, explicitly injecting random noise into \revision{latent representations of forget-samples} in an uncontrolled manner can cause the unlearned model to produce incoherent or gibberish outputs. 
Such undesirable behaviors impede the reliability and applicability of unlearning methods in high-stakes domains. 
To mitigate this, \citet{shen2025llm} argue that reference prompts, such as questions about fictitious entities, can be used to redirect the model's representations into a region where the model is unable to answer given forget-samples while still maintaining coherent generation. Nevertheless, such a view may overlook the specific roles of the target vector in the unlearning behavior, which remain insufficiently explored. 
To bridge this gap, we investigate the effect of the target vector on the unlearned models' behaviors and capabilities. 
% and work toward a principled understanding of the unlearning mechanism. 
For this purpose, we pose the following question: 

\emph{``What is the effect of the target vector in RM, and how can we leverage this vector to control the unlearned model's behaviors and capabilities?''}

To this end, we make the following contributions: \ding{192} \textbf{A falsifiable hypothesis.} We approach and revisit RM through the lens of the \emph{Linear Representation Hypothesis}~\citep{park2024linear}, which posits that a high-level concept (concept for short) is encoded linearly in the model's latent representation space. Consequently, if there is a one-dimensional vector corresponding to a concept, it becomes possible to intervene on this concept vector via linear operations within the forget-samples' representation space. From this view, we propose the \emph{\textbf{Controllable Side Effect Hypothesis}}: \emph{Beyond forgetting, MU via RM elicits controllable side behaviors and capabilities corresponding to the concept.} 

\ding{193} \textbf{Two conceptual models of RM.} To validate this hypothesis, we propose \emph{Representational Addition (RAd)} and \emph{Representational Ablation (RAb)}. RAd guides the model to unlearn by steering forget-samples' representations toward the concept vector, thereby eliciting side behaviors and capabilities corresponding to the concept. In contrast, RAb guides the model to unlearn by projecting forget-samples' representations onto the null space of the concept vector, thereby ablating components aligned with the concept and suppressing the corresponding behaviors and capabilities. 

\ding{194} \textbf{Empirical evaluation.} Extensive experiments show strong evidence supporting our hypothesis. In addition to effectively preserving unlearning and retaining performance, RAd induces side behaviors and capabilities, such as the ability to control the unlearned model's truthfulness, sentiment, refusal, stereotypical bias, language, and instruction-following behaviors. Conversely, RAb effectively suppresses these behaviors and capabilities.

\section{Background and related work}
\textbf{LLM unlearning.}  Existing unlearning methods for LLMs can be broadly categorized into the following paradigms. Preference optimization methods perform unlearning by minimizing the model's likelihood for forget-samples. Representative methods include gradient ascent (GA)~\citep{maini2024tofu, yao2024large}, DPO~\citep{maini2024tofu, yuan2025a}, NPO~\citep{zhang2024negative}, and SimNPO~\citep{fan2025simplicity}. Localization-based methods find the important component (\emph{e.g.,} parameters or neurons) most responsible for the target knowledge and then prune or edit only those components~\citep{wu2023depn, jia2023model, foster2024fast, pochinkov2024dissecting, xiao2025right, zhang2025llm}. Influence-based methods use influence functions~\cite{cook1982residuals, koh2017understanding} to estimate and reverse the influence of specific training data points on model prediction~\cite{chen2023fast, li2024delta, gu2024second, jia-etal-2024-soul, ding2025unified}. Representation Misdirection, the focus of this work, intervenes in the model's internal representations to alter how target knowledge is propagated through layers.

\textbf{Representation Misdirection (RM;\textnormal{~\citet{li2024wmdp}}).} RM redirects latent representations of forget-samples toward a target vector, typically a fixed random vector. RM is closely related to representation engineering~\citep{zou2023representation}, which posits that concepts admit linear representations that can be steered via linear operators~\citep{turner2023steering, zou2023representation, park2024linear}. However, prior RM work treats the target vector primarily as a mechanism for suppressing access to forget knowledge, either via unstructured noise injection~\citep{li2024wmdp, dang2025effects} or by redirecting representations toward an ``unanswerable'' region~\citep{shen2025llm}, \emph{without examining how the target vector influences the model's behaviors and capabilities}. A concurrent work to ours is ReVa~\citep{gu2026unlearners}, a variant of RMU~\cite{li2024wmdp} that replaces a random target vector with a refusal vector to improve the model's honesty. We defer a broader discussion on related works to Appx.~\ref{appendix:related_works}.

\textbf{Notation and problem formulation.} Denote $f_{\bm\theta}$ an autoregressive LLM parameterized by $\bm\theta$. Let $\mathcal{D}_f$ and $\mathcal{D}_r$ be the forget-set and retain-set, respectively. Denote $\mathcal{L}_{\mathcal{D}_f,\bm\theta}$ the loss of $f_{\bm\theta}$ measured on $\mathcal{D}_f$, $\mathcal{L}_{\mathcal{D}_r,\bm\theta}$ the loss of $f_{\bm\theta}$ measured on $\mathcal{D}_r$. For operators, we denote $||\cdot||_2$ the Euclidean norm, $\langle\cdot,\cdot\rangle$ the dot product. The objective of LLM unlearning is to selectively minimize the model's performance on the forget-set $\mathcal{D}_f$ while preserving the model's general knowledge on the retain-set $\mathcal{D}_r$. 
A commonly used unlearning formulation involves minimizing the following two-term loss~\cite{li2024wmdp, maini2024tofu, liu2025rethinking}:
\begin{align}
    \mathcal{L}_{\mathcal{D}_f, \mathcal{D}_r, \bm\theta} = \alpha_f \mathcal{L}_{\mathcal{D}_f,\bm\theta} +  \alpha_r \mathcal{L}_{\mathcal{D}_r,\bm\theta} \label{eq:problem_formulation}
\end{align}
where $\alpha_f, \alpha_r \in \mathbb{R}_+$ are forget and retain hyperparameters. 
% that control the magnitude of the update gradients. 
% We note that other formulations have been explored. 
% For example, unlearning using forget-set only~\citep{wang2025llm}, or combination forget-loss with additional regularization terms~\citep{yao2024large, chen2023unlearn}. Since our focus is not on comparing unlearning objectives, we adopt the widely used formulation, \textit{i.e.,} Eqn.~\ref{eq:problem_formulation}, following previous works~\citep{li2024wmdp, liu2025rethinking}. 

\section{Representation Misdirection elicits controllable side effects}
\subsection{The Controllable Side Effect Hypothesis}
\label{sec:side_effect}
The idea of the \emph{Linear Representation Hypothesis}~\citep{mikolov-etal-2013-linguistic, pennington-etal-2014-glove, arora-etal-2016-latent, elhage2022toy, park2024linear,park2025the}, if true, motivates simple and effective methods for controlling LLMs' behaviors and capabilities. Indeed, recent works suggest that concepts can be effectively represented by one-dimensional representations, which can be controlled via linear operations in the model's representation space. For example, truthfulness~\citep{li2023inference, marks2024the}, sentiment~\citep{tigges2023linear}, refusal~\citep{arditi2024refusal}, and others~\citep{wolf2024tradeoffs, zheng2024on,zou2023universal, turner2023steering, chen2026persona}. 

In the context of LLM unlearning,~\citet{li2024wmdp} claim that unlearning effectiveness may not arise from \emph{\textbf{a specific direction}} in the latent space, but rather from increasing the norm of forget-samples' representations. Naturally, a scaled random vector can serve a similar role: \emph{flooding the residual stream with random noise, which obscures the model's ability to access the target knowledge}. Our intuition is that a specific concept vector can also flood the residual stream, but with a structured signal associated with that concept rather than random noise. Under this view, we hypothesize that RM with an \emph{off-the-shelf} concept vector as a target vector not only facilitates effective unlearning but also, \emph{\textbf{without any additional computations during unlearning}}, enables the unlearned model to elicit the controllable side behaviors and capabilities corresponding to the concept.
Formally, we propose the \emph{Controllable Side Effect Hypothesis}:
\begin{mdframed}[style=hypobox]
\begin{hypothesis}[\textbf{Controllable Side Effect Hypothesis}]
    By redirecting the representations of forget-samples toward a concept direction, the unlearned model suppresses target knowledge, effectively preserves general knowledge, and \textbf{elicits controlled side behaviors and capabilities corresponding to the concept}. \label{hypothesis1}
\end{hypothesis}
\end{mdframed}

\subsection{Two conceptual models of RM}
\label{sec:conceptual_model}
To empirically validate Hypothesis~\ref{hypothesis1}, we propose two conceptual models of RM. Denote $\bm\lambda_W\in \mathbb{R}^{d}$ the one-dimensional unit vector representing a target concept $W$ at an intermediate layer in the model. Denote $\bm\lambda^{f}_{\bm\theta}, \bm\lambda^{f}_{\bm\theta^{\text{ref}}} \in \mathbb{R}^{d}$ the latent representations of forget-sample $\mathbf{x}^{f} \in \mathcal{D}_f$ at layer $l$ in the update model $f_{\bm\theta}$ (update weights during finetuning) and reference model $f_{\bm\theta^{\text{ref}}}$ (frozen weights during finetuning), respectively. $\bm\lambda^{r}_{\bm\theta}, \bm\lambda^{r}_{\bm\theta^{\text{ref}}} \in \mathbb{R}^{d}$ the latent representations of retain-sample $\mathbf{x}^{r} \in \mathcal{D}_r$ in the update model and reference model, respectively.

\textbf{Representational Addition (RAd).} RAd consists of two objectives: redirecting forget-samples' representations toward the region of the representation space associated with the concept, while preserving retain-samples' representations.
More concretely, RAd forget-loss minimizes the mean squared error (MSE) between $\bm\lambda^{f}_{\bm\theta}$ and the target representation $\bm\lambda^{f}_{\bm\theta^{\text{ref}}} + c \cdot \bm\lambda_W$, while RAd retain-loss minimizes the MSE between $\bm\lambda^{r}_{\bm\theta}$ and $\bm\lambda^{r}_{\bm\theta^{\text{ref}}}$. The RAd loss is defined as
\begin{align}
\mathcal{L}^{\text{RAd}}&= \alpha_f\mathbb{E}_{\mathbf{x}^{f} \sim \mathcal{D}_{f}} \left[\left\|\bm\lambda^{f}_{\bm\theta}- \left(\bm\lambda^{f}_{\bm\theta^{\text{ref}}} + c \cdot\bm\lambda_W\right)\right\|_2^2\right]  \nonumber\\&+\alpha_r \mathbb{E}_{\mathbf{x}^{r} \sim \mathcal{D}_{r}} \left[\left\|\bm\lambda^{r}_{\bm\theta}-\bm\lambda^{r}_{\bm\theta^{\text{ref}}}\right\|_2^2\right],\label{rad}
\end{align}
where $c \in \mathbb{R}_{+}$ is a scaling coefficient, $\alpha_f, \alpha_r \in \mathbb{R}_{+}$ are the forget and retain weights of the losses, respectively.

\textbf{Representational Ablation (RAb).} RAb consists of two objectives: ablating the components of the forget-samples' representations aligned with the concept vector, while preserving the retain-samples' representations. More concretely, RAb forget-loss minimizes the MSE between $\bm\lambda^{f}_{\bm\theta}$ and the target representation $\bm\lambda^{f}_{\bm\theta^{\text{ref}}} - c\;\langle \bm\lambda^{f}_{\bm\theta^{\text{ref}}}, \bm\lambda_W\rangle \bm\lambda_W$, which corresponds to ablating the projection of the forget-samples' representation onto the concept vector $\bm\lambda_W$. RAb preserves retain-samples' representations by minimizing the MSE between $\bm\lambda^{r}_{\bm\theta}$ and $\bm\lambda^{r}_{\bm\theta^{\text{ref}}}$. The RAb loss is defined as
\begin{align}
\mathcal{L}^{\text{RAb}} &= \alpha_f\mathbb{E}_{\mathbf{x}^{f} \sim \mathcal{D}_{f}} \left[\left\|\bm\lambda^{f}_{\bm\theta}- (\bm\lambda^{f}_{\bm\theta^{\text{ref}}} -c\, \langle\bm\lambda^f_{\bm\theta^{\text{ref}}}, \bm\lambda_W\rangle{\bm\lambda_W})\right\|_2^2\right] \nonumber\\&+\alpha_r \mathbb{E}_{\mathbf{x}^{r} \sim \mathcal{D}_{r}} \left[\left\|\bm\lambda^{r}_{\bm\theta}-\bm\lambda^{r}_{\bm\theta^{\text{ref}}}\right\|_2^2\right],\label{rab}
\end{align}
We note that RAd and RAb are instruments designed to validate Hypothesis~\ref{hypothesis1}, not to improve unlearning performance or robustness. 
Ablation studies on robustness evaluation are provided in Appx.~\ref{appendix:robustness}.
\subsection{Alignment between random vector and concept vector}
\label{sec:alinment}
% It's nice that we know that the target representation causes (undesired) side effects; we can avoid them by choosing a specific direction for unlearning.

% LLM unlearning methods that use a random vector as a target vector have recently been widely adopted. 
A special case of RAd and RAb is to replace the concept vector $\bm\lambda_{W}$ with a random vector. One might be concerned: 

\emph{``How can it be ensured that sampling a random vector does not align with a concept vector in the model?''}

Suppose $\mathbf{u}$ is a random unit vector in $\mathbb{R}^{d}$. We show that in a high-dimensional representation space, \textit{e.g.,} in modern LLMs' representation space, $\bm\lambda_W$ and $\mathbf{u}$ are nearly orthogonal. That is, for a small, positive $\epsilon$, the following inequality
\begin{align}
    |\langle \mathbf{u}, \bm\lambda_W\rangle| \leq \epsilon
\end{align}
holds with high probability.
% RQ2: ``Whould it help to choose orthogonal to typical activations to minimize unintended interference with other representations?
\begin{restatable}[]{proposition}{goldba}
% \begin{proposition}
    Suppose $\bm\lambda_W \in \mathbb{R}^{d}$ is a unit concept vector and $\mathbf{u}$ is a random vector, uniformly sampled on the unit hypersphere $\mathbb{S}^{d-1}$. Then with probability at least $1-2\exp\left(-\frac{(d-1)\epsilon^2}{2}\right)$, we have that for any $\epsilon > \sqrt{\frac{2\ln2}{d-1}}$, $|\langle \mathbf{u}, \bm\lambda_W\rangle| \leq \epsilon$. 
    % \begin{align}
    %     \mathbb{P} \left[|\langle\mathbf{u},\bar\lambda_W\rangle| \leq \epsilon\right] \geq 1-2\exp\left(-\frac{(d-1)\epsilon^2}{2}\right)
    % \end{align}
    \label{proposition1}
% \end{proposition}
\end{restatable}
The proof is deferred to Appx.~\ref{appendix:proposition}. We discuss the implications of Proposition~\ref{proposition1}: \ding{192} \textbf{A random vector is nearly orthogonal to a concept vector}. Proposition~\ref{proposition1} establishes a guarantee that, in high-dimensional representation spaces (\textit{i.e.,} $d$ is large), a random vector is nearly orthogonal to a concept vector with high probability. Therefore, using a random vector in RAd is unlikely to inadvertently align with concept vectors. \ding{193} \textbf{When using RAd with the random vector, it should be sampled \emph{once and fixed} before unlearning}. If the random vector is resampled at each gradient update, the optimization pushes forget-samples' representations toward orthogonal directions. As a result, the models' representations are not misdirected and remain aligned with those of the base model. We conduct an empirical experiment to validate the claim in Sec.~\ref{sec:random_vector_analysis}. \ding{194} \textbf{RAb with a random vector is unlikely to achieve effective unlearning}. RAb performs unlearning by ablating the components of forget-samples' representations that align with the target vector. When the target vector is a random vector, it is likely to be nearly orthogonal to concept vectors. Consequently, RAb with a random vector ablates components that are unrelated to concepts.
\begin{table*}
\caption{\label{tab:truthfulqa}
Performance of RAd and RAb models on WMDP, MMLU, and TruthfulQA.
 Metrics include BLEU, ROUGE-1/2/L for open-ended generation, and accuracy for MC1/MC2, MMLU, and WMDP. \textcolor{blue}{Increases} and \textcolor{purple}{drops} are marked (compared to the base model). %Zephyr RAd 14, RAb 50, Mistral RAd 19, RAb 
}
\centering
\vspace{-0.5em}
\setlength{\tabcolsep}{4pt}
\resizebox{0.95\textwidth}{!}{
\begin{tabular}{llcccccccc}
\toprule
\multirow{2}{*}{\textbf{}} & \multirow{2}{*}{\textbf{Models}} 
& \multicolumn{4}{c}{\textbf{TruthfulQA open-ended}} 
& \multicolumn{2}{c}{\textbf{TruthfulQA multiple-choice}} 
& \multicolumn{2}{c}{\textbf{Unlearning tasks}} \\
\cmidrule(lr){3-6} \cmidrule(lr){7-8} \cmidrule(lr){9-10}
& & BLEU & ROUGE-1  & ROUGE-2  & ROUGE-L  & MC1 & MC2 & MMLU ($\uparrow$) & WMDP ($\downarrow$)  \\
\midrule
\multirow{5}{*}{\textbf{Zephyr-7B}} 
& Base model & $47.0$ & $45.5$ & $37.9$ & $42.6$ & $39.0$ & $55.0$ & $58.4$ & $54.4$\\
\cmidrule{2-10}
& RAd w/ random  & $\textbf{49.5}${\footnotesize\textcolor{blue}{$+2.5$}} & $47.7${\footnotesize\textcolor{blue}{$+2.2$}} & $39.5${\footnotesize\textcolor{blue}{$+1.6$}} & $44.3${\footnotesize\textcolor{blue}{$+1.7$}} & $38.4${\footnotesize\textcolor{purple}{$-0.6$}} & $55.9${\footnotesize\textcolor{blue}{$+0.9$}} & $55.9$ & $25.6$\\
& $\cellcolor{gray!15}$RAd w/ truth  & $\cellcolor{gray!15}$$47.7${\footnotesize\textcolor{blue}{$+0.7$}} & $\cellcolor{gray!15}$$\textbf{53.9}${\footnotesize\textcolor{blue}{$+8.4$}} & $\cellcolor{gray!15}$$\textbf{40.9}${\footnotesize\textcolor{blue}{$+3.0$}} & $\cellcolor{gray!15}$$\textbf{51.9}${\footnotesize\textcolor{blue}{$+9.3$}} & $\cellcolor{gray!15}$$\textbf{44.9}${\footnotesize\textcolor{blue}{$+5.9$}} & $\cellcolor{gray!15}$$\textbf{62.3}${\footnotesize\textcolor{blue}{$+7.3$}} & $\cellcolor{gray!15}$$54.9$ & $\cellcolor{gray!15}$$28.2$\\ \cmidrule{2-10}
& RAb w/ random  & $51.2${\footnotesize\textcolor{blue}{$+4.2$}} & $49.7${\footnotesize\textcolor{blue}{$+4.2$}} & $41.6${\footnotesize\textcolor{blue}{$+3.7$}} & $46.8${\footnotesize\textcolor{blue}{$+4.2$}} & $38.6${\footnotesize\textcolor{purple}{$-0.4$}} & $55.6${\footnotesize\textcolor{blue}{$+0.6$}} & $57.7$ & $50.2$\\
&$\cellcolor{gray!15}$RAb w/ truth  & $\cellcolor{gray!15}$$\textbf{41.1}${\footnotesize\textcolor{purple}{$-5.9$}} & $\cellcolor{gray!15}$$\textbf{41.9}${\footnotesize\textcolor{purple}{$-3.6$}} & $\cellcolor{gray!15}$$\textbf{31.6}${\footnotesize\textcolor{purple}{$-6.3$}} & $\cellcolor{gray!15}$$\textbf{40.9}${\footnotesize\textcolor{purple}{$-1.7$}} & $\cellcolor{gray!15}$$\textbf{26.1}${\footnotesize\textcolor{purple}{$-12.9$}} & $\cellcolor{gray!15}$$\textbf{40.0}${\footnotesize\textcolor{purple}{$-15.0$}} & $\cellcolor{gray!15}$$52.0$ & $\cellcolor{gray!15}$$32.9$\\
\midrule
% \midrule
\multirow{5}{*}{\textbf{Mistral-7B}} 
& Base model & $40.6$ & $38.7$ & $35.5$ & $40.6$ & $28.2$ & $42.6$ & $59.6$ & $55.7$\\
\cmidrule{2-10}
& RAd w/ random  & $40.4${\footnotesize\textcolor{purple}{$-0.2$}} & $39.9${\footnotesize\textcolor{blue}{$+1.2$}} & $38.2${\footnotesize\textcolor{blue}{$+2.7$}} & $40.4${\footnotesize\textcolor{purple}{$-0.2$}} & $28.6${\footnotesize\textcolor{blue}{$+0.4$}} & $42.9${\footnotesize\textcolor{blue}{$+0.3$}} & $53.6$ & $25.5$\\
&$\cellcolor{gray!15}$RAd w/ truth  & $\cellcolor{gray!15}$$\textbf{50.9}${\footnotesize\textcolor{blue}{$+10.3$}} & $\cellcolor{gray!15}$$\textbf{54.1}${\footnotesize\textcolor{blue}{$+15.4$}} & $\cellcolor{gray!15}$$\textbf{46.8}${\footnotesize\textcolor{blue}{$+11.3$}} & $\cellcolor{gray!15}$$\textbf{54.6}${\footnotesize\textcolor{blue}{$+14.0$}} & $\cellcolor{gray!15}$$\textbf{34.1}${\footnotesize\textcolor{blue}{$+5.9$}} & $\cellcolor{gray!15}$$\textbf{49.9}${\footnotesize\textcolor{blue}{$+7.3$}} & $\cellcolor{gray!15}$$53.0$ & $\cellcolor{gray!15}$$25.0$\\
\cmidrule{2-10}
& RAb w/ random  & $42.8${\footnotesize\textcolor{blue}{$+2.2$}} & $41.4${\footnotesize\textcolor{blue}{$+2.7$}} & $37.9${\footnotesize\textcolor{blue}{$+2.2$}} & $42.0${\footnotesize\textcolor{blue}{$+1.4$}} & $28.4${\footnotesize\textcolor{blue}{$+0.2$}} & $43.2${\footnotesize\textcolor{blue}{$+0.6$}} & $58.7$ & $51.1$\\
&$\cellcolor{gray!15}$RAb w/ truth & $\cellcolor{gray!15}$$\textbf{36.2}${\footnotesize\textcolor{purple}{$-4.4$}}&$\cellcolor{gray!15}$$\textbf{33.8}${\footnotesize\textcolor{purple}{$-4.9$}}&$\cellcolor{gray!15}$$\textbf{27.9}${\footnotesize\textcolor{purple}{$-7.6$}}&$\cellcolor{gray!15}$$\textbf{35.0}${\footnotesize\textcolor{purple}{$-5.6$}}&$\cellcolor{gray!15}$$\textbf{24.1}${\footnotesize\textcolor{purple}{$-4.1$}}&$\cellcolor{gray!15}$$\textbf{37.4}${\footnotesize\textcolor{purple}{$-5.2$}}&$\cellcolor{gray!15}$$50.2$ & $\cellcolor{gray!15}$$29.7$\\
\bottomrule
\end{tabular}}
% \vspace{-1.0em}
\end{table*}
\section{Empirical Validation}
\label{sec:experiment}
% \textbf{Testable predictions.} Based on the two conceptual models of RM and the analysis in Proposition~\ref{proposition1}, we derive three testable predictions from Hypothesis~\ref{hypothesis1}, which we evaluate in this section:
% \begin{itemize}
%     \setlength\itemsep{1pt}
%     \item[\textbf{P1}] \textbf{Specificity.} RAd and RAb with a concept vector shift the unlearned model's behavior along the corresponding concept direction. In contrast, RAd and RAb with a random vector do not induce corresponding concept-specific behavioral changes, as a random vector is nearly orthogonal to a concept vector (Proposition~\ref{proposition1}).

%     \item[\textbf{P2}] \textbf{Selectivity.} Unlearning with a target concept vector affects the unlearned model's behavior corresponding to the target concept while preserving the unlearned model's behaviors corresponding to off-target concepts.

%     \item[\textbf{P3}] \textbf{Decoupling (of unlearning and side effects).} If unlearning effectiveness arises from increasing the norm of forget-sample representations~\citep{li2024wmdp}, then RAd achieves comparable unlearning effectiveness with random and concept vectors at the same \emph{magnitude} (\emph{i.e.,} same coefficient $c$). Meanwhile, the side effects only arise from the \emph{direction} of the vector. 
% \end{itemize}

\textbf{Unlearning tasks.} We utilize WMDP-Bio and WMDP-Cyber~\citep{li2024wmdp} to study unlearning hazardous knowledge in the biology and cyber domains. Each task dataset consists of a forget-set $\mathcal{D}_f$ and a QA evaluation set. Following~\cite{li2024wmdp}, we use Wikitext~\citep{merity2017pointer} as the retain-set $\mathcal{D}_r$. For evaluation, we report the accuracy of WMDP-Bio and WMDP-Cyber QA sets and MMLU~\citep{hendrycks2021measuring}. Beyond hazardous knowledge in biology and cyber, we further conduct ablation studies using the MUSE benchmark~\citep{shi2025muse}, which includes two domains: Books and News. MUSE experiments are deferred to the Appx.~\ref{appendix:muse}. 
% An effective unlearned model is expected to exhibit low performance on forget-tasks while preserving high performance on retain-tasks.

\textbf{Side tasks.} We evaluate the unlearned model's behaviors and capabilities on: truthfulness with TruthfulQA open-ended generation task and TruthfulQA multiple-choice tasks~\citep{lin2022truthfulqa}, sentiment with GLUE-SST2~\citep{wang2018glue}, stereotypical biases with StereoSet~\citep{nadeem-etal-2021-stereoset}, refusal behaviors with Alpaca~\citep{alpaca} and AdvBench~\citep{zou2023universal}, controlling language with HellaSwag~\citep{zellers2019hellaswag} and ICL on linguistic and knowledge tasks~\citep{hendel2023incontext}, math reasoning task on GSM8K~\citep{cobbe2021training} and GSM+~\citep{li2024gsm}. Details are deferred to Appx.~\ref{appendix:benchmarks}.

\textbf{Models.} We primarily conduct experiments on WMDP using two widely used LLMs: Zephyr-7B-$\beta$~\cite{tunstall2024zephyr}, Mistral-7B-v0.1~\citep{jiang2023mistral7b}. For specific evaluation purposes, we use Llama-3-8B-Instruct~\citep{llama3modelcard} for refusal experiments in Sec.~\ref{sec:refusal} and Qwen2.5-7B~\citep{yang2024qwen2} for reasoning tasks in Sec.~\ref{sec:reasoning_tasks}. For MUSE experiments, we employ two models from~\citet{shi2025muse}, \textit{i.e.,} MUSE-books-target and MUSE-news-target.

\textbf{Finding the concept vector.} To construct a concept vector $\bm\lambda_W$, we organize two sets of counterfactual prompts that differ in their labels with respect to the target concept $W$. Specifically, let $\mathbb{P} = \{p_W^{+}\}^{|\mathbb{P}|}$ denote the set of prompts whose desired labels are associated with concept $W$ (labeled $1$), and let $\mathbb{C} = \{p_W^{-}\}^{|\mathbb{C}|}$ denote the corresponding counterfactual prompts (labeled $0$). We extract the representations of all prompts in $\mathbb{P} \cup \mathbb{C}$ at the unlearning layer of the base model. The concept vector is the normalized weight vector of a logistic regression probe trained to distinguish between the two sets. 

\textbf{Experimental setup.} Setups are described in their respective subsections, as different tasks require distinct evaluation protocols and metrics. Details of hyperparameter selection, implementation details, and prompt templates are deferred to Appx.~\ref{appendix:experimental_setup}.

\subsection{Truthfulness}
We employ the TruthfulQA open-ended generation task. Following~\citet{li2023inference}, we reorganize this dataset, where each QA pair has a truth label (labeled $1$) or untruth (labeled $0$).
We use half of the QAs in TruthfulQA open-ended as the development set $\mathcal{D}_{\text{dev}}$ \textit{i.e.,} to construct a dataset for training the probe, and use the other half as the test set.
% For each QA in $\mathcal{D}_{\text{dev}}$, the sample's activations are extracted and hooked at a layer to form a ``latent'' dataset.
% $\mathcal{D}_{\text{dev}}$ is split in a $4:1$ ratio to get the training and validation set for training the logistic regression probe. The normalized weights of the probe represent the truth direction.
 % The validation accuracy in this is $71.3$. 
For evaluation, we use the TruthfulQA open-ended test set and TruthfulQA MC1 (multiple-choice, single answer), TruthfulQA MC2 (multiple-choice, multiple answers) for testing the truthfulness performance. 
% These test sets are disjoint from $\mathcal{D}_{\text{dev}}$ used to construct the truth vector. 
For TruthfulQA open-ended generation tasks, we report BLEU, ROUGE-1/2/L; for TruthfulQA MC tasks, we report the accuracy.  

Tab.~\ref{tab:truthfulqa} shows that RAd with truth vector consistently improves TruthfulQA performance compared to the base model. For Zephyr-7B, the average improvements are $+5.3$ on the open-ended generation task and $+6.6$ on MC tasks, while Mistral-7B exhibits larger improvements of $+12.7$ and $+6.6$, respectively. In contrast, RAd with a random vector yields only marginal improvements on TruthfulQA: Zephyr-7B achieves average improvements of $+2.0$ and $+0.1$, while Mistral-7B shows improvements of $+0.8$ and $+0.4$ on open-ended and MC tasks, respectively. Furthermore, RAd with truth lowers WMDP accuracy while maintaining general performance on MMLU. RAb with truth consistently degrades TruthfulQA performance compared to the base model. For Zephyr-7B, the average decrease is $-4.4$ on open-ended generation tasks and $-14.0$ on MC tasks, while for Mistral-7B, the average decrease is $-5.6$ and $-4.7$, respectively. RAb with random vector fails to unlearn for both models.
\subsection{Sentiment}
\label{sec:sentiment}
We employ GLUE-SST2, a benchmark for sentiment analysis with positive (pos) and negative (neg) labels. GLUE-SST2 is partitioned into training, validation, and test sets. Since labels for the SST2 test set are not publicly available, we adopt the original validation set as the test set for evaluation. The training set is used to construct the sentiment vectors.
\begin{table}[t]
% \vspace{-1.2em}
\caption{RAd with \texttt{neg$\to$pos} or RAb with \texttt{pos$\to$neg} \textbf{increases positive sentiment}.
}\label{tab:sst2_negative}
 % RAd 16, RCo 1200, 16 Rab 20, 120, Mistral Rad positive 17  RAb negative 110
\centering
\vspace{-0.5em}
\resizebox{\linewidth}{!}{
\setlength{\tabcolsep}{6pt}
\begin{tabular}{llccccc}
\toprule
\multirow{2}{*}{\textbf{Model}} & \multirow{2}{*}{\textbf{Method}} 
& \multicolumn{3}{c}{\textbf{SST2 Negative}} 
& \multirow{2}{*}{\textbf{MMLU}} 
& \multirow{2}{*}{\textbf{WMDP}} \\
\cmidrule(lr){3-5}
& & \textbf{TN} 
& \textbf{FP}
& \textbf{IP}
&  &  \\
\midrule
\multirow{5}{*}{\textbf{Zephyr-7B}} 
& Base model & $82.5$ & $13.3$ & $4.2$ & $58.4$ & $54.4$ \\
\cmidrule{2-7}
& RAd w/ random  & $77.1$ & $16.8$ & $6.1$ & $55.8$ & $25.4$ \\
& \cellcolor{gray!15}RAd w/ neg$\to$pos  & \cellcolor{gray!15}$\textbf{43.9}${\footnotesize\textcolor{purple}{$-38.6$}} & \cellcolor{gray!15}$\textbf{44.9}${\footnotesize\textcolor{blue}{$+31.6$}} & \cellcolor{gray!15}$11.2$ & \cellcolor{gray!15}$54.8$ & \cellcolor{gray!15}$26.5$ \\
\cmidrule{2-7}
& RAb w/ random  & $78.7$ & $7.9$ & $1.6$ & $53.8$ & $37.7$\\
& \cellcolor{gray!15}RAb w/ pos$\to$neg  & \cellcolor{gray!15}$\textbf{44.2}${\footnotesize\textcolor{purple}{$-38.3$}} & \cellcolor{gray!15}$\textbf{53.2}${\footnotesize\textcolor{blue}{$+39.9$}} & \cellcolor{gray!15}$2.6$ & \cellcolor{gray!15}$49.5$ & \cellcolor{gray!15}$35.4$ \\
\midrule
% \midrule
\multirow{5}{*}{\textbf{Mistral-7B}} 
& Base model & $95.3$ & $3.7$ & $0.1$ & $59.6$ & $55.7$ \\
\cmidrule{2-7}
& RAd w/ random  & $93.9$ & $5.6$ & $0.5$ & $55.9$ & $25.5$ \\
& \cellcolor{gray!15}RAd w/ neg$\to$pos  & \cellcolor{gray!15} $\textbf{55.4}${\footnotesize\textcolor{purple}{$-39.9$}} & \cellcolor{gray!15} $\textbf{32.5}${\footnotesize\textcolor{blue}{$+28.8$}} & \cellcolor{gray!15}$12.1$ & \cellcolor{gray!15}$54.5$ & \cellcolor{gray!15}$25.8$ \\
\cmidrule{2-7}
& RAb w/ random & $91.1$ & $6.8$ & $2.1$ & $56.2$ & $44.2$ \\
& \cellcolor{gray!15}RAb w/ pos$\to$neg & \cellcolor{gray!15}$\textbf{72.9}${\footnotesize\textcolor{purple}{$-22.4$}} & \cellcolor{gray!15}$\textbf{26.9}${\footnotesize\textcolor{blue}{$+23.2$}} & \cellcolor{gray!15}$0.2$ & \cellcolor{gray!15}$45.5$ & \cellcolor{gray!15}$30.8$ \\
\bottomrule
\end{tabular}}
\vspace{-1.0em}
\end{table}
We define two concept vectors: \texttt{neg$\to$pos} and \texttt{pos$\to$neg}. The order of these concepts makes the sign of a representation meaningful, \textit{i.e.,} \texttt{neg$\to$pos} and \texttt{pos$\to$neg} are opposite. Once we identify the \texttt{neg$\to$pos} vector, we can take the opposite direction to represent \texttt{pos$\to$neg}. 
% To identify the \texttt{neg$\to$pos} vector, we train a logistic regression to distinguish between negative samples' representations (labeled $0$) and positive samples' representations (labeled $1$).
% The normalized weights of the probe represent 
Concept \texttt{neg$\to$pos} defines the direction associated with \emph{increasing positive sentiment}. 
In contrast, \texttt{pos$\to$neg} defines the direction associated with \emph{increasing negative sentiment}.

We partition the SST2 test set into two distinct subtasks: SST2 negative (containing negative samples), and SST2 positive (containing positive samples). For the SST2 negative task, we report \textit{true negative} (TN) and \textit{false positive} (FP) rates. For the SST2 positive task, we report \textit{true positive} (TP) and \textit{false negative} (FN). Beyond those metrics, we define \textit{invalid prediction} (IP = $\frac{\#(\hat{y} = -1)}{\#\text{samples}}$) rate, which measures the fraction of given samples for which the model generates an answer of neither positive nor negative. As shown in Tab.~\ref{tab:sst2_negative} and Tab.~\ref{tab:sst2_positive}, unlearning via RAd and RAb successfully steers model behavior toward the targeted sentiment. In the SST2 negative task, unlearning via RAd with \texttt{neg$\to$pos} or RAb with \texttt{pos$\to$neg} vector leads to a substantial drop in TN rates and a corresponding surge in FP.  For instance, Zephyr-7B's TN drops by $38.6$, while its FP increases by $31.6$. A similar trend is observed for the SST2 positive task (Tab.~\ref{tab:sst2_positive}). 
\begin{table}
\caption{RAd with \texttt{pos$\to$neg} or RAb with \texttt{neg$\to$pos} \textbf{increases negative sentiment}.} \label{tab:sst2_positive}
%  Zephyr RAd 23, RAb 20 120. Mistral RAd negative 20, Mistral RAb positive 110
\centering 
\vspace{-0.5em}
\setlength{\tabcolsep}{6pt}
\resizebox{\linewidth}{!}{
\begin{tabular}{llccccc}
\toprule
\multirow{2}{*}{\textbf{Model}} & \multirow{2}{*}{\textbf{Method}} 
& \multicolumn{3}{c}{\textbf{SST2 Positive}} 
& \multirow{2}{*}{\textbf{MMLU}} 
& \multirow{2}{*}{\textbf{WMDP}} \\
\cmidrule(lr){3-5}
& & \textbf{TP} 
& \textbf{FN}
& \textbf{IP}
&  &  \\
\midrule
\multirow{5}{*}{\textbf{Zephyr-7B}} 
& Base model & $91.6$ & $4.3$ & $4.1$ & $58.4$ & $54.4$ \\
\cmidrule{2-7}
& RAd w/ random & $93.5$ & $1.8$ & $4.7$ & $52.7$ & $25.1$ \\
& \cellcolor{gray!15}RAd w/ pos$\to$neg & $\cellcolor{gray!15}$$\textbf{69.4}${\footnotesize\textcolor{purple}{$-22.2$}} & $\cellcolor{gray!15}$$\textbf{26.5}${\footnotesize\textcolor{blue}{$+22.2$}} & $\cellcolor{gray!15}$$4.1$ & $\cellcolor{gray!15}$ $52.0$ & $\cellcolor{gray!15}$$24.6$ \\
\cmidrule{2-7}
& RAb w/ random & $91.9$ & $4.5$ & $3.6$ & $53.8$ & $37.7$ \\
& \cellcolor{gray!15}RAb w/ neg$\to$pos & $\cellcolor{gray!15}$$\textbf{66.6}${\footnotesize\textcolor{purple}{$-25.0$}} & $\cellcolor{gray!15}$$\textbf{28.2}${\footnotesize\textcolor{blue}{$+23.9$}} & $\cellcolor{gray!15}$$5.2$ & $\cellcolor{gray!15}$$49.5$ & $\cellcolor{gray!15}$$35.4$ \\
\midrule
% \midrule
\multirow{5}{*}{\textbf{Mistral-7B}} 
& Base model & $89.8$ & $10.2$ & $0.0$ & $59.6$ & $55.7$ \\
\cmidrule{2-7}
& RAd w/ random & $6.1$ & $0.7$ & $93.2$ & $51.3$ & $25.3$ \\
& \cellcolor{gray!15}RAd w/ pos$\to$neg & $\cellcolor{gray!15}$$\textbf{36.0}${\footnotesize\textcolor{purple}{$-53.8$}} & $\cellcolor{gray!15}$$\textbf{62.8}${\footnotesize\textcolor{blue}{$+52.6$}} & $\cellcolor{gray!15}$$1.2$ & $\cellcolor{gray!15}$$51.2$ & $\cellcolor{gray!15}$$26.7$ \\
\cmidrule{2-7}
& RAb w/ random  & $93.7$ & $6.3$ & $0.0$ & $56.2$ & $44.2$ \\
& \cellcolor{gray!15}RAb w/ neg$\to$pos  & $\cellcolor{gray!15}$$\textbf{39.8}${\footnotesize\textcolor{purple}{$-50.0$}} & $\cellcolor{gray!15}$ $\textbf{60.0}${\footnotesize\textcolor{blue}{$+48.8$}} & $\cellcolor{gray!15}$$0.2$ & $\cellcolor{gray!15}$$45.6$ & $\cellcolor{gray!15}$$31.0$ \\
\bottomrule
\end{tabular}}
\vspace{-0.7em}
\end{table}

\begin{table*}
% \vspace{-0.5em}
\caption{RAd with task-specific vectors improves ICL across four linguistic and knowledge tasks while effectively preserving unlearning performance. \textcolor{blue}{Increases} are marked compared to corresponding base models with zero-shot template.} 
\label{tab:icl}
% Exact match results for antonym (18), pres $\to$ past 16, country $\to$ capital (18), and person $\to$ language (19). Mistral: Antonym 19, present2past: 19, country2capital: 18, person2language: 20
\centering
\vspace{-0.5em}
\resizebox{0.8\linewidth}{!}{
\setlength{\tabcolsep}{4pt}
\begin{tabular}{lllcccccc}
\toprule
\multirow{2}{*}{\textbf{Model}} & \multirow{2}{*}{\textbf{Method}} 
& \multirow{2}{*}{\textbf{Template}} 
& \multicolumn{2}{c}{\textbf{Linguistic}} 
& \multicolumn{2}{c}{\textbf{Knowledge}} 
& \multirow{2}{*}{\textbf{MMLU}}
& \multirow{2}{*}{\textbf{WMDP}} \\
\cmidrule(lr){4-5} \cmidrule(lr){6-7}
& & & \textbf{antonyms} & \textbf{pres$\to$past} & \textbf{ctry$\to$cap} & \textbf{pers$\to$lang} & & \\
\midrule
\multirow{8}{*}{\textbf{Zephyr-7B}} 
& \multirow{2}{*}{Base model} & zero-shot & $6.1$  & $1.8$  & $24.6$ & $12.7$ & \multirow{2}{*}{$58.4$} & \multirow{2}{*}{$54.4$} \\
& & instruction & $74.4$ & $83.4$ & $91.5$ & $83.7$ &  &   \\
\cmidrule{2-9}
& RAd w/ random & zero-shot & $14.6$ & $1.6$ & $11.2$ & $9.7$ & $54.9$ & $25.9$ \\
\cmidrule{2-9}
& RAd w/ antonyms & zero-shot & \cellcolor{gray!15}$\textbf{39.0}${\footnotesize\textcolor{blue}{$+32.9$}} & $1.2$ & $8.4$ & $10.6$ & $53.3$ & $25.0$ \\
& RAd w/ pres$\to$past & zero-shot & $1.2$ & \cellcolor{gray!15}$\textbf{27.2}${\footnotesize\textcolor{blue}{$+25.4$}} & $2.1$ & $11.2$ & $54.4$ & $26.7$ \\
& RAd w/ ctry$\to$cap & zero-shot & $0.0$ & $0.4$ & \cellcolor{gray!15}$\textbf{69.0}${\footnotesize\textcolor{blue}{$+44.4$}} & $9.5$ & $54.8$ & $27.4$ \\
& RAd w/ pers$\to$lang & zero-shot & $3.6$ & $0.0$ & $0.7$ & \cellcolor{gray!15}$\textbf{43.5}${\footnotesize\textcolor{blue}{$+30.8$}} & $50.6$ & $25.5$ \\
\midrule 
\multirow{7}{*}{\textbf{Mistral-7B}} 
& \multirow{2}{*}{Base model} & zero-shot & $1.2$ & $0.0$ & $11.9$ & $0.0$ & \multirow{2}{*}{$59.6$} & \multirow{2}{*}{$55.7$} \\
& & instruction & $59.7$ & $72.1$ & $91.5$ & $80.3$ &  &   \\
\cmidrule{2-9}
& RAd w/ random & zero-shot & $14.6$ & $0.4$ & $7.7$ & $0.0$ & $53.7$ & $25.6$ \\
\cmidrule{2-9}
& RAd w/ antonyms & zero-shot & \cellcolor{gray!15}$\textbf{30.5}$ {\footnotesize\textcolor{blue}{$+29.3$}} & $0.2$ & $4.9$ & $0.0$ & $54.6$ & $24.9$ \\
& RAd w/ pres$\to$past & zero-shot & $1.2$ & \cellcolor{gray!15}$\textbf{28.4}${\footnotesize\textcolor{blue}{$+28.4$}} & $5.6$ & $0.0$ & $55.1$ & $24.5$ \\
& RAd w/ ctry$\to$cap & zero-shot & $1.2$ & $0.4$ & \cellcolor{gray!15}$\textbf{70.4}${\footnotesize\textcolor{blue}{$+58.5$}} & $0.0$ & $55.9$ & $26.6$ \\
& RAd w/ pers$\to$lang & zero-shot & $1.2$ & $0.4$ & $0.0$ & \cellcolor{gray!15}$\textbf{7.8}${\footnotesize\textcolor{blue}{$+7.8$}} & $50.1$ & $25.4$ \\
\bottomrule
\end{tabular}}
\vspace{-0.7em}
\end{table*}

\subsection{Refusal}
\label{sec:refusal}
We construct two datasets: $\mathcal{D}_{\text{harmful}}$, which contains harmful instructions drawn from AdvBench; and $\mathcal{D}_{\text{harmless}}$, which contains harmless instructions drawn from Alpaca. Each dataset consists of two disjoint sets: a train set and a test set. The train set is used to construct the refusal vector, while the test set is used to evaluate unlearned models. 
We define the refusal concept as \texttt{harmless$\to$harmful}, a unit vector representing the direction in activation space that induces harmful behavior. 
% This refusal vector is defined as the normalized weights vector of the logistic regression probe trained to distinguish between the harmful instructions' representations (labeled $1$) and harmless instructions' representations (labeled $0$).

Following prior work~\citep{liu2024autodan,xu-etal-2024-cognitive, arditi2024refusal, robey2025smoothllm}, we report the \textit{refusal score}. The refusal score measures the refusal of an answer by string matching. A refusal contains a refusal substring, such as ``As an AI language model.'' If the generated answer includes at least one such refusal substring, it is classified as a refusal; otherwise, as non-refusal. Since Mistral-7B-v0.1 is not an instruction-tuned model, we employ Llama3-8B-Instruct to use the chat template for consistent evaluation. The set of refusal substrings and chat template for evaluation is provided in Appx.~\ref{appendix:refusal_substrings}.

\begin{table}[H]
% \vspace{-1.2em}
\caption{
RAd w/ refusal \textbf{induces refusal to harmless instructions}. 
} \label{tab:alpaca}
% Zephyr 18 RAd 
\centering
\vspace{-0.5em}
\resizebox{\linewidth}{!}{
\setlength{\tabcolsep}{6pt}
\begin{tabular}{llccc}
\toprule
\multirow{2}{*}{\textbf{Model}} & \multirow{2}{*}{\textbf{Method}} 
& \multicolumn{1}{c}{\textbf{Alpaca}} 
& \multirow{2}{*}{\textbf{MMLU}} 
& \multirow{2}{*}{\textbf{WMDP}} \\
\cmidrule(lr){3-3}
& & \textbf{Refusal score} 
&  &  \\
\midrule
\multirow{3}{*}{\textbf{Zephyr-7B}} 
& Base model & $8.6$ & $58.4$ & $54.4$ \\
\cmidrule{2-5}
& RAd w/ random & $9.6${\footnotesize\textcolor{blue}{$+1.0$}} & $54.9$ & $26.0$ \\
& \cellcolor{gray!15}RAd w/ refusal & $\cellcolor{gray!15}$$\textbf{37.5}${\footnotesize\textcolor{blue}{$+28.9$}} & $\cellcolor{gray!15}$$51.7$ & $\cellcolor{gray!15}$$26.7$ \\
% \midrule 
\midrule
\multirow{3}{*}{\textbf{Llama-3-8B}} 
& Base model & $3.8$ & $63.8$ & $58.7$ \\
\cmidrule{2-5}
& RAd w/ random & $4.8${\footnotesize\textcolor{blue}{$+1.0$}} & $62.7$ & $34.0$ \\
& \cellcolor{gray!15}RAd w/ refusal & $\cellcolor{gray!15}$$\textbf{100.0}${\footnotesize\textcolor{blue}{$+96.2$}} & $\cellcolor{gray!15}$$62.5$ & $\cellcolor{gray!15}$$31.8$ \\
\bottomrule
\end{tabular}}
% \vspace{-1.2em}
\end{table}

\begin{table}[H]
% \vspace{-1.2em}
\caption{
RAb w/ refusal \textbf{ablates refusal to harmful instructions}. 
} \label{tab:advbench}
% Zephyr RAb 20, 40, Llama-3 RAb 20, 60
\centering
\vspace{-0.5em}
\setlength{\tabcolsep}{6pt}
\resizebox{\linewidth}{!}{
\begin{tabular}{llccc}
\toprule
\multirow{2}{*}{\textbf{Model}} & \multirow{2}{*}{\textbf{Method}} 
& \multicolumn{1}{c}{\textbf{AdvBench}} 
& \multirow{2}{*}{\textbf{MMLU}} 
& \multirow{2}{*}{\textbf{WMDP}} \\
\cmidrule(lr){3-3}
& & \textbf{Refusal score} 
&  &  \\
\midrule
\multirow{3}{*}{\textbf{Zephyr-7B}} 
    & Base model 
        & $90.3$ 
        & $58.4$ 
        & $54.4$ \\
    \cmidrule{2-5}
    & RAb w/ random
        & $82.7${\footnotesize\textcolor{purple}{$-7.6$}} 
        & $57.6$ 
        & $52.1$ \\
    & $\cellcolor{gray!15}$RAb w/ refusal
        & $\cellcolor{gray!15}$$\textbf{49.0}${\footnotesize\textcolor{purple}{$-41.3$}}  
        & $\cellcolor{gray!15}$$54.2$ 
        & $\cellcolor{gray!15}$$36.8$ \\
    \midrule 
    % \midrule
    \multirow{3}{*}{\textbf{Llama-3-8B}} 
    & Base model & $98.1$ & $63.8$ & $58.7$ \\
    \cmidrule{2-5}
    & RAb w/ random & $98.1${\footnotesize\textcolor{gray}{$-0.0$}} & $63.4$ & $57.5$ \\
    & $\cellcolor{gray!15}$RAb w/ refusal & $\cellcolor{gray!15}$$\textbf{1.9}${\footnotesize\textcolor{purple}{$-96.2$}} & $\cellcolor{gray!15}$$55.1$ &$\cellcolor{gray!15}$ $38.4$ \\
        \bottomrule
\end{tabular}}
% \vspace{-1.6em}
\end{table}

Tab.~\ref{tab:alpaca} shows that RAd with refusal vector makes the unlearned model \textit{refuse even harmless instructions} while Tab.~\ref{tab:advbench} shows that unlearning via RAb with refusal removes the model’s refusal behavior, preventing it from refusing harmful instructions. In contrast, RAd or RAb with a random vector has only a marginal effect on refusal behavior.

\subsection{Instruction-following behaviors}
In-context learning (ICL;~\citet{radford2019language, brown2020language, dong2024survey}) is the ability of a model to leverage its internal knowledge to adapt and reason given the \textit{instruction}. Consider a knowledge task, where the model is asked to generate the capital of a given country name. With a zero-shot prompt template, such as ``\texttt{Text:} \texttt{Japan}\texttt{\textbackslash nLabel:}'' which provides no specific task knowledge, the model often fails to answer and achieves near-zero performance. However, when the task knowledge is provided, \textit{e.g.,} replacing ``\texttt{Label:}'' with ``\texttt{Capital:}'' the model's performance increases significantly. This phenomenon has been attributed to the model implicitly learning a task vector from the instruction~\citep{hendel2023incontext}. We argue that if an \textit{instruction} vector can be effectively represented linearly as a one-dimensional vector in the model’s latent space, using the instruction vector as the target vector makes the unlearned model \emph{elicit controlled task-specific knowledge corresponding to the instruction}. 

\subsubsection{Linguistic and knowledge tasks}
\label{sec:icl}

We conduct experiments with four ICL tasks across two categories: linguistic and factual knowledge~\citep{hendel-etal-2023-context}, including (1) \emph{antonyms} (2) \emph{country-to-capital} (ctry$\to$cap), (3) \emph{person-to-language} (pers$\to$lang), and (4) \emph{present-to-past} (pres$\to$past). For evaluation, we randomly split each original dataset into training, validation, and test sets in a $4:1:5$ ratio. The training and validation sets are used to construct the \textit{instruction} vector. To construct the instruction vector, each sample is formatted in two templates: \emph{zero-shot template} (without specifying task knowledge) and (2) \emph{instruction template} (explicitly specifies the task knowledge). 
The instruction vector is the normalized weights vector of a logistic regression probe that was trained to distinguish between zero-shot samples' representations (labeled $0$) and instruction samples' representations (labeled $1$). Prompt templates for tasks are deferred to Tab.~\ref{tab:chat_templates}.

For evaluation, we report the exact-match accuracy on the $4$ tasks under the zero-shot template. 
As shown in Tab.~\ref{tab:icl}, base models exhibit low or near-zero accuracy in the zero-shot setting, while providing task-specific instructions significantly improves performance, confirming that these tasks rely on contextual task vectors. 
Unlearning via RAd with instruction vector consistently improves zero-shot ICL performance on the corresponding task for both Zephyr-7B and Mistral-7B. 
For example, RAd with ctry$\to$cap vector boosts zero-shot accuracy from $24.6$ to $69.0$ on Zephyr-7B and from $11.9$ to $70.4$ on Mistral-7B, while leaving unrelated tasks unaffected. 
Similar improvements are observed for antonyms, pres$\to$past, and pers$\to$lang tasks.
In contrast, RAd with random vector shows no significant changes compared to the base model, indicating that the improvements arise from instruction vectors.
\subsubsection{Reasoning tasks}
\label{sec:reasoning_tasks}

Beyond short task-specific instructions, we test two more demanding forms of instruction variation for complex reasoning:  a \emph{CoT template} and a \emph{few-shot template.}

We employ $2$ math reasoning benchmarks, including GSM8K and GSM+ for evaluation. Each sample in the GSM8K training split is formatted under three templates: a zero-shot template, a CoT template, and a 5-shot template (See Tab.~\ref{tab:chat_templates} for details).
% (``\texttt{Question: \{question\}\textbackslash n}\texttt{Answer:}''), a CoT template (``\texttt{Question: \{question\}\textbackslash n}\texttt{Answer:} \texttt{Let's think step by step.}''), and a 5-shot template (``\texttt{[5 samples]}\texttt{Question: \{question\}\textbackslash n}\texttt{Answer:}''). 
For each variation, we train a logistic regression probe to distinguish zero-shot samples' representations (labeled $0$) from the corresponding variant's representations (labeled $1$). The normalized probe weights define the \textit{CoT vector} and the \textit{5-shot vector}, respectively. For the evaluations on GSM+, we assess how well the unlearning generalizes by using the checkpoint whose instruction vector was constructed from GSM8K.
\begin{table}
% \vspace{-1.5em}
\caption{Performance of unlearned models w/ CoT and 5-shot vector on reasoning tasks. \textcolor{blue}{Increases} and \textcolor{purple}{drops} are marked compared to the base model with the \textbf{corresponding templates}.}
\centering
\vspace{-0.5em}
\resizebox{\linewidth}{!}{
\setlength{\tabcolsep}{4pt}
\begin{tabular}{lllcccc}
\toprule
\multirow{2}{*}{\textbf{Model}}
& \multirow{2}{*}{\textbf{Method}}
& \multirow{2}{*}{\textbf{Template}}
& \multicolumn{2}{c}{\textbf{Reasoning tasks}}
& \multicolumn{2}{c}{\textbf{Unlearning tasks}} \\
\cmidrule(lr){4-5} \cmidrule(lr){6-7}
& &
& \textbf{GSM8K}
& \textbf{GSM}+
& \textbf{MMLU}
& \textbf{WMDP} \\
\midrule
\multirow{10}{*}{\textbf{Zephyr-7B}}
& \multirow{3}{*}{Base model}
& zero-shot & $15.8$ & $10.5$
& \multirow{3}{*}{$58.4$} & \multirow{3}{*}{$54.4$} \\
& & CoT & $18.8$ & $12.1$
&  &   \\
& & 5-shot & $34.2$ & $19.8$ \\
\cmidrule{2-7}
& RAd w/ random & zero-shot & $15.1${\footnotesize\textcolor{purple}{$-0.7$}} & $10.3${\footnotesize\textcolor{purple}{$-0.2$}}
& $54.2$ & $25.3$ \\
& $\cellcolor{gray!15}$RAd w/ CoT & $\cellcolor{gray!15}$zero-shot & $\cellcolor{gray!15}$$17.4${\footnotesize\textcolor{blue}{$+1.6$}} & $\cellcolor{gray!15}$$10.4${\footnotesize\textcolor{purple}{$-0.1$}}
& $\cellcolor{gray!15}$$55.0$ & $\cellcolor{gray!15}$$25.9$ \\
& $\cellcolor{gray!15}$RAd w/ 5-shot & $\cellcolor{gray!15}$zero-shot
& $\cellcolor{gray!15}$$17.0${\footnotesize\textcolor{blue}{$+1.2$}}
& $\cellcolor{gray!15}$$10.5${\footnotesize\textcolor{gray}{$+0.0$}}
& $\cellcolor{gray!15}$$54.7$ & $\cellcolor{gray!15}$$26.7$ \\
% & \revision{RAb w/ random} & \revision{CoT} & \revision{$20.7$} & \revision{$12.6$}
% &  &  \\
\cmidrule{2-7}
& RAb w/ random & zero-shot & $15.1${\footnotesize\textcolor{purple}{$-0.7$}} & $9.9${\footnotesize\textcolor{purple}{$-0.6$}}
& \multirow{1}{*}{$58.1$} & \multirow{1}{*}{$49.5$} \\
& $\cellcolor{gray!15}$RAb w/ CoT & $\cellcolor{gray!15}$CoT & $\cellcolor{gray!15}$$3.3${\footnotesize\textcolor{purple}{$-15.5$}} & $\cellcolor{gray!15}$\revision{$1.2${\footnotesize\textcolor{purple}{$-10.9$}}}
& $\cellcolor{gray!15}$$53.1$ & $\cellcolor{gray!15}$$32.0$ \\
& $\cellcolor{gray!15}$RAb w/ 5-shot & $\cellcolor{gray!15}$5-shot
& $\cellcolor{gray!15}$$23.9${\footnotesize\textcolor{purple}{$-10.3$}}
& $\cellcolor{gray!15}$$12.8${\footnotesize\textcolor{purple}{$-7.0$}}
& $\cellcolor{gray!15}$$50.5$ & $\cellcolor{gray!15}$$34.3$ \\
\midrule
\multirow{10}{*}{\textbf{Qwen2.5-7B}}
& \multirow{3}{*}{Base model}
& zero-shot & $10.1$ & $10.8$
& \multirow{3}{*}{$71.8$} & \multirow{3}{*}{$64.4$} \\
& & CoT & $36.4$ & $11.4$
&  &   \\
& & 5-shot
& $81.8$ & $59.4$
&  &   \\
% & & \revision{5-shot} & \revision{$81.8$} & \revision{$59.4$}
% &  &   \\
\cmidrule{2-7}
& RAd w/ random & zero-shot & $6.8${\footnotesize\textcolor{purple}{$-3.3$}} & $6.9${\footnotesize\textcolor{purple}{$-3.9$}} & $69.0$ & $58.8$ \\
& $\cellcolor{gray!15}$RAd w/ CoT & $\cellcolor{gray!15}$zero-shot & $\cellcolor{gray!15}$$21.5${\footnotesize\textcolor{blue}{$+11.4$}} & $\cellcolor{gray!15}$$20.7${\footnotesize\textcolor{blue}{$+9.9$}} & $\cellcolor{gray!15}$$62.8$ & $\cellcolor{gray!15}$$42.6$ \\
& $\cellcolor{gray!15}$RAd w/ 5-shot & $\cellcolor{gray!15}$zero-shot
& $\cellcolor{gray!15}$$47.2${\footnotesize\textcolor{blue}{$+37.1$}}
& $\cellcolor{gray!15}$$34.6${\footnotesize\textcolor{blue}{$+23.8$}}
& $\cellcolor{gray!15}$$58.7$ & $\cellcolor{gray!15}$$41.6$ \\
\cmidrule{2-7}
& RAb w/ random & zero-shot & $9.7${\footnotesize\textcolor{purple}{$-0.4$}} & $11.5${\footnotesize\textcolor{blue}{$+0.7$}} & $71.2$ & $61.4$ \\
& $\cellcolor{gray!15}$RAb w/ CoT & $\cellcolor{gray!15}$CoT &$\cellcolor{gray!15}$$1.3${\footnotesize\textcolor{purple}{$-35.1$}} & $\cellcolor{gray!15}$$0.6${\footnotesize\textcolor{purple}{$-10.8$}} & $\cellcolor{gray!15}$$54.0$ & $\cellcolor{gray!15}$$45.9$ \\
& $\cellcolor{gray!15}$RAb w/ 5-shot & $\cellcolor{gray!15}$5-shot
& $\cellcolor{gray!15}$$1.4${\footnotesize\textcolor{purple}{$-80.4$}}
& $\cellcolor{gray!15}$$4.5${\footnotesize\textcolor{purple}{$-54.9$}}
& $\cellcolor{gray!15}$$57.1$ & $\cellcolor{gray!15}$$47.7$ \\
\bottomrule \label{tab:reasoning_tasks}
\end{tabular}}
\vspace{-2.7em}
\end{table}

We observe the following from Tab.~\ref{tab:reasoning_tasks}. First, the effectiveness of RAd and RAb depends on the underlying reasoning capability of the base model. 
Specifically, RAd and RAb show limited effectiveness on Zephyr-7B, whereas they produce substantially larger effects on a strong reasoning model, Qwen2.5-7B. 
Second, RAd and RAb remain effective with the 5-shot instruction. Despite the significantly increased complexity and length of the 5-shot samples compared with the CoT prompt, the learned 5-shot direction still effectively modulates the model's reasoning capabilities. Finally, eliciting stronger reasoning capabilities comes with a trade-off in preserving the model's general capabilities. Across both CoT and 5-shot settings, larger improvements on reasoning tasks are generally accompanied by larger degradations on unlearning tasks. 
% An ablation study on the effects of the coefficient on performance is deferred to Appx.~\ref{appendix:reasoning_coeff}.

\section{Ablation study}
\subsection{Random vector sampling in RAd}
\label{sec:random_vector_analysis}
Proposition~\ref{proposition1} implies that the random vector used in ``RAd w/ random'' is \textbf{sampled once and kept the same throughout the unlearning process}. This design choice is motivated by the reason that, from an optimization standpoint, resampling the random vector across steps would direct representations toward inconsistent and conflicting targets at each gradient update. This causes gradient cancellation, \textit{i.e.,} updates push representations in orthogonal directions and undo each other, leaving the forget-samples' representations insufficiently misdirected and closely aligned with those of the base model. To empirically support this claim, we experiment by comparing \emph{RAd with a fixed random vector} with \emph{RAd with multiple random vectors} (\textit{i.e.,} resampled at each step) in terms of unlearning performance and representation alignment (via cosine similarity). Tab.~\ref{table:random_vector} shows that RAd with a fixed random vector achieves better unlearning performance, reducing WMDP from 
$54.4$ to $26.5$, while RAd (resampled at each step) fails to unlearn. The cosine similarity further confirms that a fixed random vector in RAd is essential for effectively misdirecting latent representations.
\begin{table}[H]
\vspace{-0.1em}
\caption{Unlearning performance of Zephyr-7B on MMLU and WMDP, and representation alignment between base and RAd models on WMDP.}
\centering
% \vspace{-0.5em}
\resizebox{0.9\linewidth}{!}{
\setlength{\tabcolsep}{4pt}
\begin{tabular}{lccc}
\toprule
\textbf{Model} & \textbf{MMLU} ($\uparrow$) & \textbf{WMDP} ($\downarrow$) & \textbf{Cosine} \\
\midrule
Base model         & $58.4$ & $54.4$ & $-$ \\
\midrule
RAd (fixed) & $56.4${\footnotesize$\textcolor{purple}{-2.0}$}  & $26.5${\footnotesize$\textcolor{purple}{-27.9}$} & $0.11$ \\
RAd (resampling)   & $58.2${\footnotesize$\textcolor{purple}{-0.2}$}  & $53.1${\footnotesize$\textcolor{purple}{-1.3}$}  & $0.98$ \\
\bottomrule
\end{tabular}} \label{table:random_vector}
% \vspace{-0.5em}
\end{table}
% \vspace{-3.0em}

% \begin{figure}
%     \centering
%     \includegraphics[width=0.4\linewidth]{assets/concept_direction_alignment.pdf}
%     \caption{\revision{abc}}
%     \label{fig:concept_direction_alignment}
% \end{figure}
% \begin{wrapfigure}{r}{0.4\textwidth}
% \centering
% \includegraphics[width=\linewidth]{assets/qwen2.5_rad_cot.pdf}
% \caption{\revision{Qwen2.5-7B RAd w/ cot}}
% \label{fig:qwen2.5_rad_cot}
% \end{wrapfigure}
% \begin{wrapfigure}{r}{0.4\textwidth}
% \centering
% \includegraphics[width=\linewidth]{assets/qwen2.5_rab_cot.pdf}
% \caption{\revision{Qwen2.5-7B RAb w/ cot}}
% \label{fig:qwen2.5_rab_cot}
% \end{wrapfigure}
% \begin{wrapfigure}{r}{0.4\textwidth}
% \centering
% \includegraphics[width=\linewidth]{assets/qwen2.5_rad_5shot.pdf}
% \caption{\revision{Qwen2.5-7B RAd w/ 5-shot}}
% \label{fig:qwen2.5_rad_5shot}
% \end{wrapfigure}
% \begin{wrapfigure}{r}{0.4\textwidth}
% \centering
% \includegraphics[width=\linewidth]{assets/qwen2.5_rab_5shot.pdf}
% \caption{\revision{Qwen2.5-7B RAb w/ 5-shot}}
% \label{fig:qwen2.5_rab_5shot}
% \end{wrapfigure}
\subsection{RAd and RAb vs. concept-guided inference-time steering}
We investigate whether the RAd and RAb uniquely offer unlearning with side behaviors or whether these behaviors are inherited from the concept vector in inference-time steering. We conduct two inference-time settings: (1) \textbf{concept-guided additive steering}, which performs inference-time steering by adding the concept vector without parameter updates; (2) \textbf{concept-guided ablative steering}, which performs inference-time ablative steering using the concept vector. We evaluate both settings on Zephyr-7B and Mistral-7B while varying the steering coefficient $c$, using TruthfulQA (open-ended and multiple-choice) to measure side behaviors and MMLU and WMDP to assess unlearning performance.
\begin{figure*}[h]
    \centering
    \begin{subfigure}[t]{0.43\linewidth}
        \centering
        \includegraphics[width=\linewidth]{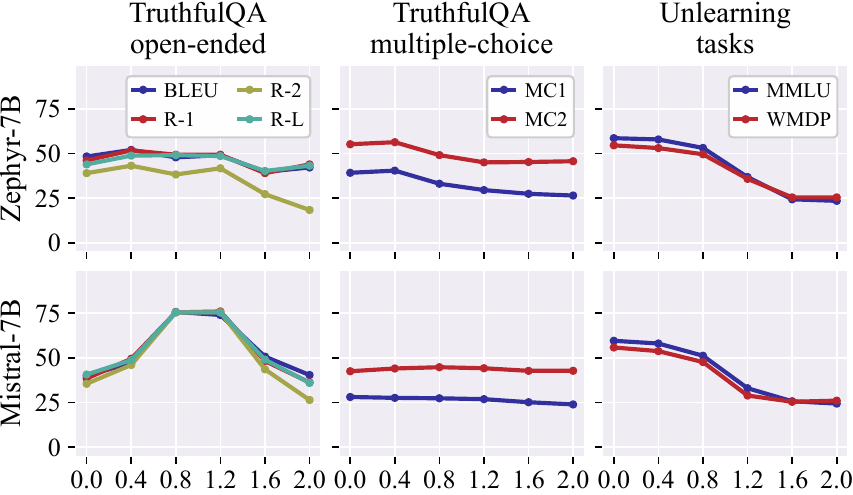}
        \caption{Truth-guided additive steering}
        \label{fig:inference_time_steering.rad.truth}
    \end{subfigure}
    % \hspace{0.2cm}
    \begin{subfigure}[t]{0.43\linewidth}
        \centering
        \includegraphics[width=\linewidth]{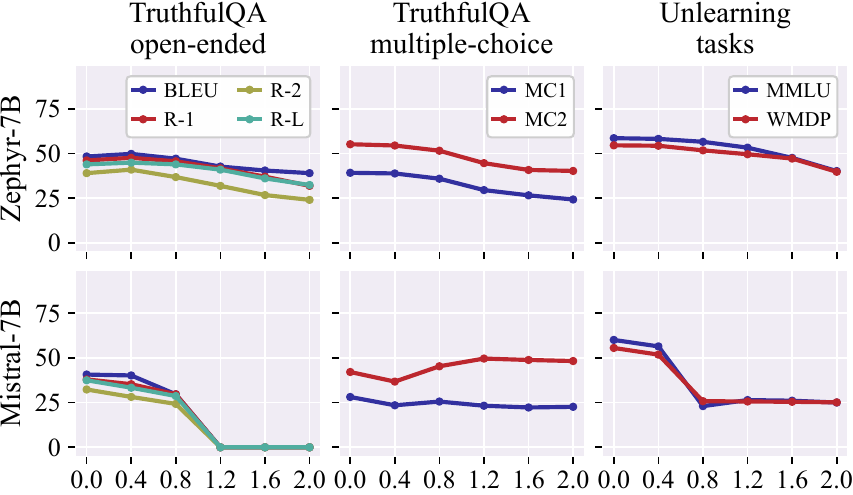}
        \caption{Random-guided additive steering}
        \label{fig:inference_time_steering.rad.rand}
    \end{subfigure}
    
    \begin{subfigure}[t]{0.43\linewidth}
        \centering
        \includegraphics[width=\linewidth]{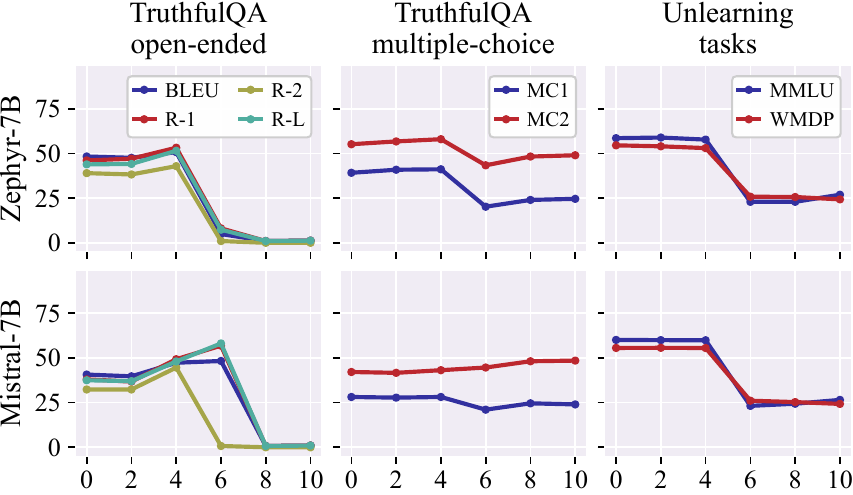}
        \caption{Truth-guided ablative steering}
        \label{fig:inference_time_steering.rab.truth}
    \end{subfigure}
    % \hspace{0.2cm}
    \begin{subfigure}[t]{0.43\linewidth}
        \centering
        \includegraphics[width=\linewidth]{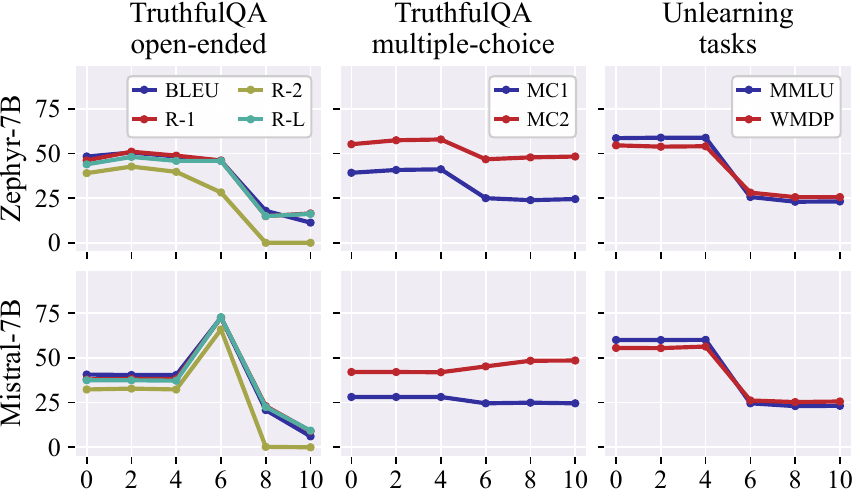}
        \caption{Random-guided ablative steering}
        \label{fig:inference_time_steering.rab.rand}
    \end{subfigure}
    \caption{Accuracy of concept-guided inference-time steering (y-axis) with varying steering coefficients (x-axis).}
    \label{fig:inference_time_steering}
    \vspace{-1.em}
\end{figure*}

Fig.~\ref{fig:inference_time_steering} reveals two observations. First, increasing the steering coefficient consistently improves TruthfulQA performance (Fig.~\ref{fig:inference_time_steering.rad.truth}) or degrades TruthfulQA performance (Fig.~\ref{fig:inference_time_steering.rab.truth}), demonstrating that inference-time steering effectively induces or suppresses the desired semantic behavior. Second, inference-time steering fails to achieve unlearning. As the steering coefficient increases, MMLU and WMDP either remain unchanged or degrade simultaneously, indicating that inference-time steering is insufficient to selectively suppress forget knowledge while preserving general knowledge. These results distinguish inference-time steering from RAd and RAb. While inference-time steering can elicit the behaviors associated with the concept, it fails to produce effective unlearning. 
% In contrast, RAd and RAb enable effective unlearning while eliciting controllable side behaviors and capabilities.
% \revision{Fig.~\ref{fig:inference_time_steering} reveals three observations. First, increasing the steering coefficient consistently improves TruthfulQA performance with the truth concept (Fig.~\ref{fig:inference_time_steering.rad.truth}) or degrades TruthfulQA performance (Fig.~\ref{fig:inference_time_steering.rab.truth}), demonstrating that inference-time steering effectively induces or suppresses the desired semantic behavior (Fig.~\ref{fig:inference_time_steering.rad.truth}). However, inference-time steering fails to unlearn. MMLU and WMDP either remain unchanged or deteriorate together as the steering coefficient increases, indicating that the intervention does not selectively suppress forget knowledge while preserving general knowledge. These results distinguish inference-time steering from RAd and RAb. While concept-guided inference-time steering elicits the side behaviors associated with the high-level concept, it cannot achieve effective unlearning. In contrast, RAd and RAb uniquely enable unlearning with side behaviors.}

\subsection{Comparison to current LLM unlearning methods}
\label{sec:comparision}
\begin{figure}[h]
    \centering
    \includegraphics[width=0.9\linewidth]{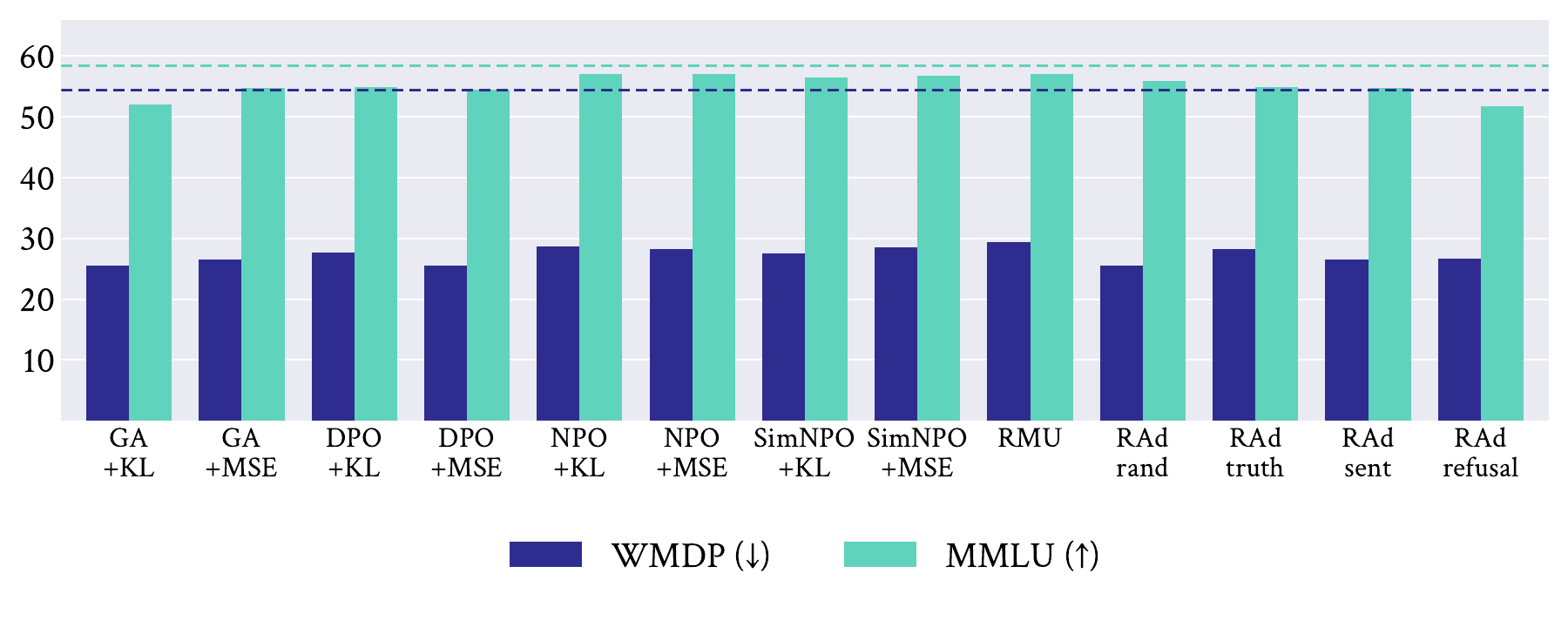}
    \caption{Unlearning performance across methods of Zephyr-7B on MMLU and WMDP (avg. of WMDP-bio and WMDP-cyber).}
    \label{fig:other_methods}
    \vspace{-1.em}
\end{figure}
We consider the following methods: RMU~\citep{li2024wmdp},  Gradient Ascent~\citep{thudi2022unrolling, liu2022continual, yaolarge}), NPO~\citep{zhang2024negative}, 
SimNPO~\citep{fan2025simplicity}, DPO~\citep{maini2024tofu,yuan2025a}. 
For preference optimization methods, we employ MSE or Kullback-Leibler (KL) divergence as the retain loss. Combining these, we evaluate eight preference optimization methods, including GA+MSE, GA+KL, NPO+MSE, NPO+KL, DPO+MSE, DPO+KL, SimNPO+MSE, and SimNPO+KL. We defer the formulation of these methods and their hyperparameters to Appx.~\ref{appendix:unlearning_baselines}. 
Fig.~\ref{fig:other_methods} shows that RAd achieves competitive performance across current advanced LLM unlearning methods. RAd demonstrates balanced unlearning, \textit{i.e.,} consistently reducing WMDP scores while maintaining MMLU performance. Specifically, MMLU decreases of approximately 2.5-3.5\%, which is comparable to DPO (3.5-4\%), GA (4-6\%), and only slightly larger (by 1-1.5\%) than that of SimNPO and NPO, and uniquely offers controllable side effects through the target concept vectors.
\subsection{Mechanism of RAd and RAb: Superficial masking or true erasure?}
\label{sec:mechanism}
Despite years of research, the mechanism of LLM unlearning remains a controversial subject~\citep{liu2025rethinking, cooper2025machine, hu2025unlearning, yu2025impossibility, triantafillou2026your, mavrothalassitis2026ascent, huu2025improving, chen2026unlearning}. 
In our study, we acknowledge that RAd and RAb may fall into the category of \textit{superficial masking}, \textit{i.e.,} redirecting the forget-samples' representations to suppress the model's ability to access target knowledge, rather than achieving \textit{true erasure}, \textit{i.e.,} surgically erasing target knowledge from the model's weights. 
Nevertheless, we emphasize that our central objective is not to claim true erasure. Rather, our work aims to reveal a previously unexplored phenomenon: that the choice of target vector systematically elicits controllable side behaviors and capabilities.

\section{Conclusion}
We revisit RM through the lens of the Linear Representation Hypothesis. We show that by redirecting the forget-samples' representations relative to a one-dimensional concept vector, via linear operations such as addition or ablation, the unlearned model not only unlearns but also induces controllable side behaviors and capabilities. 
% We hope this reveal provides a useful lens for future work on principled LLM unlearning methods.

% \section*{Acknowledgements}
\newpage
\section*{AI Use Statement}
% AI tools were used for grammar checking and formatting the tables and figures. 
The code for this project was initially implemented by the authors. AI tools (Claude Sonnet, ChatGPT) were later used to support refactoring and proving the code. 

We have discussed with Claude about the paper's narrative and presentation. We received some suggestions from Claude. Unfortunately, we found them insufficiently aligned with our intended message of the paper.

We hereby declare that, to our best knowledge and belief, the paper's technical content was handwritten by the authors.

\section*{Impact Statement}

% This work proposes RM methods for LLM unlearning by redirecting forget-samples' representation toward high-level semantic concepts in the latent space. Our primary goal is to enable selective knowledge suppression while preserving the model's general knowledge and capabilities, which can benefit applications requiring compliance with privacy regulations and copyright removal.

\textbf{Dual-use risk.} Our proposed methods are dual-use. They could be used to suppress safety capabilities, degrade model truthfulness, or manipulate other behaviors. These risks highlight the importance of responsible deployment of RM based unlearning methods. 
% Our analysis contributes to a better understanding of these risks and motivates future research on robust and secure machine unlearning.}

\textbf{Responsible-use guidelines.} We recommend that RM be deployed only within a controlled and well-documented framework. The forget set, retain set, and target concept should be explicitly documented and reviewed before deployment. 
% As discussed in Section~\ref{sec:alinment}, when avoiding unintended side effects is a priority, we recommend using a random vector rather than a concept vector as the target vector. 
Our experiments demonstrate that RM can induce side effects in the evaluated models and tasks; however, they do not establish that arbitrary concepts or capabilities can be reliably controlled across models.

\bibliography{example_paper}

@inproceedings{mikolov-etal-2013-linguistic,
    title = "Linguistic Regularities in Continuous Space Word Representations",
    author = "Mikolov, Tomas  and
      Yih, Wen-tau  and
      Zweig, Geoffrey",
    editor = "Vanderwende, Lucy  and
      Daum{\'e} III, Hal  and
      Kirchhoff, Katrin",
    booktitle = "Proceedings of the 2013 Conference of the North {A}merican Chapter of the Association for Computational Linguistics: Human Language Technologies",
    month = jun,
    year = "2013",
    address = "Atlanta, Georgia",
    publisher = "Association for Computational Linguistics",
    url = "https://aclanthology.org/N13-1090/",
    pages = "746--751"
}

@inproceedings{gu2026unlearners,
  title={Unlearners Can Lie: Evaluating and Improving Honesty in LLM Unlearning},
  author={Gu, Renjie and Du, Jiazhen and Zhang, Yihua and Liu, Sijia},
  booktitle={Proceedings of the 64th Annual Meeting of the Association for Computational Linguistics (Volume 1: Long Papers)},
  pages={11930--11952},
  year={2026}
}

@article{yu2025impossibility,
  title={On the Impossibility of Retrain Equivalence in Machine Unlearning},
  author={Yu, Jiatong and He, Yinghui and Goyal, Anirudh and Arora, Sanjeev},
  journal={arXiv preprint arXiv:2510.16629},
  year={2025}
}

@article{gao2021framework,
  title={A framework for few-shot language model evaluation},
  author={Gao, Leo and Tow, Jonathan and Biderman, Stella and Black, Sid and DiPofi, Anthony and Foster, Charles and Golding, Laurence and Hsu, Jeffrey and McDonell, Kyle and Muennighoff, Niklas and others},
  journal={Zenodo},
  year={2021}
}

@book{vershynin2018high,
  title={High-dimensional probability: An introduction with applications in data science},
  author={Vershynin, Roman},
  volume={47},
  year={2018},
  publisher={Cambridge university press}
}

@book{ledoux2001concentration,
  title={The concentration of measure phenomenon},
  author={Ledoux, Michel},
  number={89},
  year={2001},
  publisher={American Mathematical Soc.}
}

@book{milman1986asymptotic,
  title={Asymptotic theory of finite dimensional normed spaces},
  author={Milman, Vitali D and Schechtman, Gideon},
  year={1986},
  publisher={Springer}
}

@inproceedings{wang-etal-2025-model-unlearning,
    title = "Model Unlearning via Sparse Autoencoder Subspace Guided Projections",
    author = "Wang, Xu  and
      Li, Zihao  and
      Wang, Benyou  and
      Hu, Yan  and
      Zou, Difan",
    editor = "Christodoulopoulos, Christos  and
      Chakraborty, Tanmoy  and
      Rose, Carolyn  and
      Peng, Violet",
    booktitle = "Proceedings of the 2025 Conference on Empirical Methods in Natural Language Processing",
    month = nov,
    year = "2025",
    address = "Suzhou, China",
    publisher = "Association for Computational Linguistics",
    url = "https://aclanthology.org/2025.emnlp-main.1348/",
    doi = "10.18653/v1/2025.emnlp-main.1348",
    pages = "26530--26546",
    ISBN = "979-8-89176-332-6"
}

@article{farrell2024applying,
  title={Applying sparse autoencoders to unlearn knowledge in language models},
  author={Farrell, Eoin and Lau, Yeu-Tong and Conmy, Arthur},
  journal={arXiv preprint arXiv:2410.19278},
  year={2024}
}

@article{yamashita2025sparse,
  title={Sparse-Autoencoder-Guided Internal Representation Unlearning for Large Language Models},
  author={Yamashita, Tomoya and Ito, Akira and Yamanaka, Yuuki and Yamada, Masanori and Miura, Takayuki and Shibahara, Toshiki},
  journal={arXiv preprint arXiv:2509.15631},
  year={2025}
}

@inproceedings{
muhamed2025saes,
title={{SAE}s Can Improve Unlearning: Dynamic Sparse Autoencoder Guardrails for Precision Unlearning in {LLM}s},
author={Aashiq Muhamed and Jacopo Bonato and Mona T. Diab and Virginia Smith},
booktitle={Second Conference on Language Modeling},
year={2025},
url={https://openreview.net/forum?id=kaPAalWAp3}
}

@article{triantafillou2026your,
  title={Is your algorithm unlearning or untraining?},
  author={Triantafillou, Eleni and Humayun, Ahmed Imtiaz and Ribero, Monica and Turner, Alexander Matt and Mozer, Michael C and Kaissis, Georgios},
  journal={arXiv preprint arXiv:2604.07962},
  year={2026}
}

@inproceedings{
cooper2025machine,
title={Machine Unlearning Doesn't Do What You Think:  Lessons for Generative {AI} Policy and Research},
author={A. Feder Cooper and Christopher A. Choquette-Choo and Miranda Bogen and Kevin Klyman and Matthew Jagielski and Katja Filippova and Ken Liu and Alexandra Chouldechova and Jamie Hayes and Yangsibo Huang and Eleni Triantafillou and Peter Kairouz and Nicole Elyse Mitchell and Niloofar Mireshghallah and Abigail Z. Jacobs and James Grimmelmann and Vitaly Shmatikov and Christopher De Sa and Ilia Shumailov and Andreas Terzis and Solon Barocas and Jennifer Wortman Vaughan and danah boyd and Yejin Choi and Sanmi Koyejo and Fernando Delgado and Percy Liang and Daniel E. Ho and Pamela Samuelson and Miles Brundage and David Bau and Seth Neel and Hanna Wallach and Amy B. Cyphert and Mark Lemley and Nicolas Papernot and Katherine Lee},
booktitle={The Thirty-Ninth Annual Conference on Neural Information Processing Systems Position Paper Track},
year={2025},
url={https://openreview.net/forum?id=mfd6GRW4Az}
}

@inproceedings{dong2024survey,
  title={A survey on in-context learning},
  author={Dong, Qingxiu and Li, Lei and Dai, Damai and Zheng, Ce and Ma, Jingyuan and Li, Rui and Xia, Heming and Xu, Jingjing and Wu, Zhiyong and Chang, Baobao and others},
  booktitle={Proceedings of the 2024 conference on empirical methods in natural language processing},
  pages={1107--1128},
  year={2024}
}

@article{radford2019language,
  title={Language Models are Unsupervised Multitask Learners},
  author={Radford, Alec and Wu, Jeff and Child, Rewon and Luan, David and Amodei, Dario and Sutskever, Ilya},
  year={2019}
}

@inproceedings{
park2025the,
title={The Geometry of Categorical and Hierarchical Concepts in Large Language Models},
author={Kiho Park and Yo Joong Choe and Yibo Jiang and Victor Veitch},
booktitle={The Thirteenth International Conference on Learning Representations},
year={2025},
url={https://openreview.net/forum?id=bVTM2QKYuA}
}

@inproceedings{li2024wmdp,
  title={The WMDP Benchmark: Measuring and Reducing Malicious Use with Unlearning},
  author={Li, Nathaniel and Pan, Alexander and Gopal, Anjali and Yue, Summer and Berrios, Daniel and Gatti, Alice and Li, Justin D and Dombrowski, Ann-Kathrin and Goel, Shashwat and Mukobi, Gabriel and others},
  booktitle={International Conference on Machine Learning},
  pages={28525--28550},
  year={2024},
  organization={PMLR}
}

@article{elhage2022toy,
  title={Toy models of superposition},
  author={Elhage, Nelson and Hume, Tristan and Olsson, Catherine and Schiefer, Nicholas and Henighan, Tom and Kravec, Shauna and Hatfield-Dodds, Zac and Lasenby, Robert and Drain, Dawn and Chen, Carol and others},
  journal={arXiv preprint arXiv:2209.10652},
  year={2022}
}

@inproceedings{park2024linear,
  title={The Linear Representation Hypothesis and the Geometry of Large Language Models},
  author={Park, Kiho and Choe, Yo Joong and Veitch, Victor},
  booktitle={International Conference on Machine Learning},
  pages={39643--39666},
  year={2024},
  organization={PMLR}
}

@article{zou2023universal,
  title={Universal and transferable adversarial attacks on aligned language models},
  author={Zou, Andy and Wang, Zifan and Carlini, Nicholas and Nasr, Milad and Kolter, J Zico and Fredrikson, Matt},
  journal={arXiv preprint arXiv:2307.15043},
  year={2023}
}

@article{arora-etal-2016-latent,
    title = "A Latent Variable Model Approach to {PMI}-based Word Embeddings",
    author = "Arora, Sanjeev  and
      Li, Yuanzhi  and
      Liang, Yingyu  and
      Ma, Tengyu  and
      Risteski, Andrej",
    editor = "Lee, Lillian  and
      Johnson, Mark  and
      Toutanova, Kristina",
    journal = "Transactions of the Association for Computational Linguistics",
    volume = "4",
    year = "2016",
    address = "Cambridge, MA",
    publisher = "MIT Press",
    url = "https://aclanthology.org/Q16-1028/",
    doi = "10.1162/tacl_a_00106",
    pages = "385--399"
}

@article{thaker2024guardrail,
  title={Guardrail baselines for unlearning in llms},
  author={Thaker, Pratiksha and Maurya, Yash and Hu, Shengyuan and Wu, Zhiwei Steven and Smith, Virginia},
  journal={arXiv preprint arXiv:2403.03329},
  year={2024}
}

@inproceedings{dang2025effects,
  title={On Effects of Steering Latent Representation for Large Language Model Unlearning},
  author={Dang, Huu-Tien and Pham, Tin and Thanh-Tung, Hoang and Inoue, Naoya},
  booktitle={Proceedings of the AAAI Conference on Artificial Intelligence},
  volume={39},
  number={22},
  pages={23733--23742},
  year={2025}
}

@misc{alpaca,
  author = {Rohan Taori and Ishaan Gulrajani and Tianyi Zhang and Yann Dubois and Xuechen Li and Carlos Guestrin and Percy Liang and Tatsunori B. Hashimoto },
  title = {Stanford Alpaca: An Instruction-following LLaMA model},
  year = {2023},
  publisher = {GitHub},
  journal = {GitHub repository},
  howpublished = {\url{https://github.com/tatsu-lab/stanford_alpaca}},
}

@article{liu2024large,
  title={Large language model unlearning via embedding-corrupted prompts},
  author={Liu, Chris and Wang, Yaxuan and Flanigan, Jeffrey and Liu, Yang},
  journal={Advances in Neural Information Processing Systems},
  volume={37},
  pages={118198--118266},
  year={2024}
}

@inproceedings{ren-etal-2025-general,
    title = "A General Framework to Enhance Fine-tuning-based {LLM} Unlearning",
    author = "Ren, Jie  and
      Dai, Zhenwei  and
      Tang, Xianfeng  and
      Liu, Hui  and
      Zeng, Jingying  and
      Li, Zhen  and
      Goutam, Rahul  and
      Wang, Suhang  and
      Xing, Yue  and
      He, Qi  and
      Liu, Hui",
    editor = "Che, Wanxiang  and
      Nabende, Joyce  and
      Shutova, Ekaterina  and
      Pilehvar, Mohammad Taher",
    booktitle = "Findings of the Association for Computational Linguistics: ACL 2025",
    month = jul,
    year = "2025",
    address = "Vienna, Austria",
    publisher = "Association for Computational Linguistics",
    url = "https://aclanthology.org/2025.findings-acl.949/",
    doi = "10.18653/v1/2025.findings-acl.949",
    pages = "18464--18476",
    ISBN = "979-8-89176-256-5"
}

@article{kuo2025exact,
  title={Exact unlearning of finetuning data via model merging at scale},
  author={Kuo, Kevin and Setlur, Amrith and Srinivas, Kartik and Raghunathan, Aditi and Smith, Virginia},
  journal={arXiv preprint arXiv:2504.04626},
  year={2025}
}

@inproceedings{
shen2025llm,
title={{LLM} Unlearning via Neural Activation Redirection},
author={William F. Shen and Xinchi Qiu and Meghdad Kurmanji and Alex Iacob and Lorenzo Sani and Yihong Chen and Nicola Cancedda and Nicholas D. Lane},
booktitle={The Thirty-ninth Annual Conference on Neural Information Processing Systems},
year={2025},
url={https://openreview.net/forum?id=teB4aqJsNP}
}

@inproceedings{nadeem-etal-2021-stereoset,
    title = "{S}tereo{S}et: Measuring stereotypical bias in pretrained language models",
    author = "Nadeem, Moin  and
      Bethke, Anna  and
      Reddy, Siva",
    editor = "Zong, Chengqing  and
      Xia, Fei  and
      Li, Wenjie  and
      Navigli, Roberto",
    booktitle = "Proceedings of the 59th Annual Meeting of the Association for Computational Linguistics and the 11th International Joint Conference on Natural Language Processing (Volume 1: Long Papers)",
    month = aug,
    year = "2021",
    address = "Online",
    publisher = "Association for Computational Linguistics",
    url = "https://aclanthology.org/2021.acl-long.416",
    doi = "10.18653/v1/2021.acl-long.416",
    pages = "5356--5371",
}

@inproceedings{
fan2025simplicity,
title={Simplicity Prevails: Rethinking Negative Preference Optimization for {LLM} Unlearning},
author={Chongyu Fan and Jiancheng Liu and Licong Lin and Jinghan Jia and Ruiqi Zhang and Song Mei and Sijia Liu},
booktitle={The Thirty-ninth Annual Conference on Neural Information Processing Systems},
year={2025},
url={https://openreview.net/forum?id=JbvSQm5h1l}
}

@article{zou2023representation,
  title={Representation engineering: A top-down approach to ai transparency},
  author={Zou, Andy and Phan, Long and Chen, Sarah and Campbell, James and Guo, Phillip and Ren, Richard and Pan, Alexander and Yin, Xuwang and Mazeika, Mantas and Dombrowski, Ann-Kathrin and others},
  journal={arXiv preprint arXiv:2310.01405},
  year={2023}
}

@inproceedings{wang2025rethinking,
  title={Rethinking llm unlearning objectives: A gradient perspective and go beyond},
  author={Wang, Qizhou and Zhou, Jin and Shin, Saebyeol and Han, Bo and Weinberger, Kilian and others},
  booktitle={International Conference on Learning Representations},
  volume={2025},
  pages={61897--61931},
  year={2025}
}

@inproceedings{
yuan2025a,
title={A Closer Look at Machine Unlearning for Large Language Models},
author={Xiaojian Yuan and Tianyu Pang and Chao Du and Kejiang Chen and Weiming Zhang and Min Lin},
booktitle={The Thirteenth International Conference on Learning Representations},
year={2025},
url={https://openreview.net/forum?id=Q1MHvGmhyT}
}

@article{liu2025rethinking,
  title={Rethinking machine unlearning for large language models},
  author={Liu, Sijia and Yao, Yuanshun and Jia, Jinghan and Casper, Stephen and Baracaldo, Nathalie and Hase, Peter and Yao, Yuguang and Liu, Chris Yuhao and Xu, Xiaojun and Li, Hang and others},
  journal={Nature Machine Intelligence},
  pages={1--14},
  year={2025},
  publisher={Nature Publishing Group UK London}
}

@inproceedings{
shi2025muse,
title={{MUSE}: Machine Unlearning Six-Way Evaluation for Language Models},
author={Weijia Shi and Jaechan Lee and Yangsibo Huang and Sadhika Malladi and Jieyu Zhao and Ari Holtzman and Daogao Liu and Luke Zettlemoyer and Noah A. Smith and Chiyuan Zhang},
booktitle={The Thirteenth International Conference on Learning Representations},
year={2025},
url={https://openreview.net/forum?id=TArmA033BU}
}

@article{chen2025feature,
  title={Feature-Selective Representation Misdirection for Machine Unlearning},
  author={Chen, Taozhao and Huang, Linghan and Choo, Kim-Kwang Raymond and Chen, Huaming},
  journal={arXiv preprint arXiv:2512.16297},
  year={2025}
}

@article{eldan2023s,
  title={Who's Harry Potter? Approximate Unlearning in LLMs},
  author={Eldan, Ronen and Russinovich, Mark},
  journal={arXiv preprint arXiv:2310.02238},
  year={2023}
}

@inproceedings{wu-etal-2023-depn,
    title = "{DEPN}: Detecting and Editing Privacy Neurons in Pretrained Language Models",
    author = "Wu, Xinwei  and
      Li, Junzhuo  and
      Xu, Minghui  and
      Dong, Weilong  and
      Wu, Shuangzhi  and
      Bian, Chao  and
      Xiong, Deyi",
    editor = "Bouamor, Houda  and
      Pino, Juan  and
      Bali, Kalika",
    booktitle = "Proceedings of the 2023 Conference on Empirical Methods in Natural Language Processing",
    month = dec,
    year = "2023",
    address = "Singapore",
    publisher = "Association for Computational Linguistics",
    url = "https://aclanthology.org/2023.emnlp-main.174/",
    doi = "10.18653/v1/2023.emnlp-main.174",
    pages = "2875--2886"
}

@article{zhang2023composing,
  title={Composing parameter-efficient modules with arithmetic operation},
  author={Zhang, Jinghan and Liu, Junteng and He, Junxian and others},
  journal={Advances in Neural Information Processing Systems},
  volume={36},
  pages={12589--12610},
  year={2023}
}

@inproceedings{
wang2018glue,
title={{GLUE}: A Multi-Task Benchmark and Analysis Platform for Natural Language Understanding},
author={Alex Wang and Amanpreet Singh and Julian Michael and Felix Hill and Omer Levy and Samuel R. Bowman},
booktitle={International Conference on Learning Representations},
year={2019},
url={https://openreview.net/forum?id=rJ4km2R5t7},
}

@inproceedings{
hendel2023incontext,
title={In-Context Learning Creates Task Vectors},
author={Roee Hendel and Mor Geva and Amir Globerson},
booktitle={The 2023 Conference on Empirical Methods in Natural Language Processing},
year={2023},
url={https://openreview.net/forum?id=QYvFUlF19n}
}

@article{llama3modelcard,

title={Llama 3 Model Card},

author={AI@Meta},

year={2024},

url = {https://github.com/meta-llama/llama3/blob/main/MODEL_CARD.md}

}

@inproceedings{xu-etal-2024-cognitive,
    title = "Cognitive Overload: Jailbreaking Large Language Models with Overloaded Logical Thinking",
    author = "Xu, Nan  and
      Wang, Fei  and
      Zhou, Ben  and
      Li, Bangzheng  and
      Xiao, Chaowei  and
      Chen, Muhao",
    editor = "Duh, Kevin  and
      Gomez, Helena  and
      Bethard, Steven",
    booktitle = "Findings of the Association for Computational Linguistics: NAACL 2024",
    month = jun,
    year = "2024",
    address = "Mexico City, Mexico",
    publisher = "Association for Computational Linguistics",
    url = "https://aclanthology.org/2024.findings-naacl.224/",
    doi = "10.18653/v1/2024.findings-naacl.224",
    pages = "3526--3548"
}

@article{
robey2025smoothllm,
title={Smooth{LLM}: Defending Large Language Models Against Jailbreaking Attacks},
author={Alexander Robey and Eric Wong and Hamed Hassani and George J. Pappas},
journal={Transactions on Machine Learning Research},
issn={2835-8856},
year={2025},
url={https://openreview.net/forum?id=laPAh2hRFC},
note={}
}

@inproceedings{
liu2024autodan,
title={Auto{DAN}: Generating Stealthy Jailbreak Prompts on Aligned Large Language Models},
author={Xiaogeng Liu and Nan Xu and Muhao Chen and Chaowei Xiao},
booktitle={The Twelfth International Conference on Learning Representations},
year={2024},
url={https://openreview.net/forum?id=7Jwpw4qKkb}
}

@inproceedings{
hendrycks2021measuring,
title={Measuring Massive Multitask Language Understanding},
author={Dan Hendrycks and Collin Burns and Steven Basart and Andy Zou and Mantas Mazeika and Dawn Song and Jacob Steinhardt},
booktitle={International Conference on Learning Representations},
year={2021},
url={https://openreview.net/forum?id=d7KBjmI3GmQ}
}

@inproceedings{
merity2017pointer,
title={Pointer Sentinel Mixture Models},
author={Stephen Merity and Caiming Xiong and James Bradbury and Richard Socher},
booktitle={International Conference on Learning Representations},
year={2017},
url={https://openreview.net/forum?id=Byj72udxe}
}

@misc{jiang2023mistral7b,
      title={Mistral 7B}, 
      author={Albert Q. Jiang and Alexandre Sablayrolles and Arthur Mensch and Chris Bamford and Devendra Singh Chaplot and Diego de las Casas and Florian Bressand and Gianna Lengyel and Guillaume Lample and Lucile Saulnier and Lélio Renard Lavaud and Marie-Anne Lachaux and Pierre Stock and Teven Le Scao and Thibaut Lavril and Thomas Wang and Timothée Lacroix and William El Sayed},
      year={2023},
      eprint={2310.06825},
      archivePrefix={arXiv},
      primaryClass={cs.CL},
      url={https://arxiv.org/abs/2310.06825}, 
}

@inproceedings{
tunstall2024zephyr,
title={Zephyr: Direct Distillation of {LM} Alignment},
author={Lewis Tunstall and Edward Emanuel Beeching and Nathan Lambert and Nazneen Rajani and Kashif Rasul and Younes Belkada and Shengyi Huang and Leandro Von Werra and Cl{\'e}mentine Fourrier and Nathan Habib and Nathan Sarrazin and Omar Sanseviero and Alexander M Rush and Thomas Wolf},
booktitle={First Conference on Language Modeling},
year={2024},
url={https://openreview.net/forum?id=aKkAwZB6JV}
}

@inproceedings{jang2023knowledge,
  title={Knowledge unlearning for mitigating privacy risks in language models},
  author={Jang, Joel and Yoon, Dongkeun and Yang, Sohee and Cha, Sungmin and Lee, Moontae and Logeswaran, Lajanugen and Seo, Minjoon},
  booktitle={Proceedings of the 61st Annual Meeting of the Association for Computational Linguistics (Volume 1: Long Papers)},
  pages={14389--14408},
  year={2023}
}

@article{gu2024second,
  title={Second-order information matters: Revisiting machine unlearning for large language models},
  author={Gu, Kang and Rashid, Md Rafi Ur and Sultana, Najrin and Mehnaz, Shagufta},
  journal={arXiv preprint arXiv:2403.10557},
  year={2024}
}

@inproceedings{
ren2025keeping,
title={Keeping an Eye on {LLM} Unlearning: The Hidden Risk and Remedy},
author={Jie Ren and Zhenwei DAI and Xianfeng Tang and Yue Xing and Shenglai Zeng and Hui Liu and Jingying Zeng and Qiankun Peng and Samarth Varshney and Suhang Wang and Qi He and Charu C. Aggarwal and Hui Liu},
booktitle={The Thirty-ninth Annual Conference on Neural Information Processing Systems},
year={2025},
url={https://openreview.net/forum?id=MgN8Px0NA5}
}

@misc{
turner2025steering,
title={Steering Language Models with Activation Engineering},
author={Alexander Matt Turner and Lisa Thiergart and Gavin Leech and David Udell and Juan J Vazquez and Ulisse Mini and Monte MacDiarmid},
year={2025},
url={https://openreview.net/forum?id=2XBPdPIcFK}
}

@inproceedings{
gurnee2024language,
title={Language Models Represent Space and Time},
author={Wes Gurnee and Max Tegmark},
booktitle={The Twelfth International Conference on Learning Representations},
year={2024},
url={https://openreview.net/forum?id=jE8xbmvFin}
}

@inproceedings{nanda-etal-2023-emergent,
    title = "Emergent Linear Representations in World Models of Self-Supervised Sequence Models",
    author = "Nanda, Neel  and
      Lee, Andrew  and
      Wattenberg, Martin",
    editor = "Belinkov, Yonatan  and
      Hao, Sophie  and
      Jumelet, Jaap  and
      Kim, Najoung  and
      McCarthy, Arya  and
      Mohebbi, Hosein",
    booktitle = "Proceedings of the 6th BlackboxNLP Workshop: Analyzing and Interpreting Neural Networks for NLP",
    month = dec,
    year = "2023",
    address = "Singapore",
    publisher = "Association for Computational Linguistics",
    url = "https://aclanthology.org/2023.blackboxnlp-1.2/",
    doi = "10.18653/v1/2023.blackboxnlp-1.2",
    pages = "16--30"
}

@inproceedings{pennington-etal-2014-glove,
    title = "{G}lo{V}e: Global Vectors for Word Representation",
    author = "Pennington, Jeffrey  and
      Socher, Richard  and
      Manning, Christopher",
    editor = "Moschitti, Alessandro  and
      Pang, Bo  and
      Daelemans, Walter",
    booktitle = "Proceedings of the 2014 Conference on Empirical Methods in Natural Language Processing ({EMNLP})",
    month = oct,
    year = "2014",
    address = "Doha, Qatar",
    publisher = "Association for Computational Linguistics",
    url = "https://aclanthology.org/D14-1162/",
    doi = "10.3115/v1/D14-1162",
    pages = "1532--1543"
}

@inproceedings{
tamirisa2025tamperresistant,
title={Tamper-Resistant Safeguards for Open-Weight {LLM}s},
author={Rishub Tamirisa and Bhrugu Bharathi and Long Phan and Andy Zhou and Alice Gatti and Tarun Suresh and Maxwell Lin and Justin Wang and Rowan Wang and Ron Arel and Andy Zou and Dawn Song and Bo Li and Dan Hendrycks and Mantas Mazeika},
booktitle={The Thirteenth International Conference on Learning Representations},
year={2025},
url={https://openreview.net/forum?id=4FIjRodbW6}
}

@inproceedings{
wang2025invariance,
title={Invariance Makes {LLM} Unlearning Resilient Even to Unanticipated Downstream Fine-Tuning},
author={Changsheng Wang and Yihua Zhang and Jinghan Jia and Parikshit Ram and Dennis Wei and Yuguang Yao and Soumyadeep Pal and Nathalie Baracaldo and Sijia Liu},
booktitle={Forty-second International Conference on Machine Learning},
year={2025},
url={https://openreview.net/forum?id=x2lm33kdrZ}
}

@article{yan2025dual,
  title={Dual-Space Smoothness for Robust and Balanced LLM Unlearning},
  author={Yan, Han and Liu, Zheyuan and Jiang, Meng},
  journal={arXiv preprint arXiv:2509.23362},
  year={2025}
}

@inproceedings{
fan2025towards,
title={Towards {LLM} Unlearning Resilient to Relearning Attacks: A Sharpness-Aware Minimization Perspective and Beyond},
author={Chongyu Fan and Jinghan Jia and Yihua Zhang and Anil Ramakrishna and Mingyi Hong and Sijia Liu},
booktitle={Forty-second International Conference on Machine Learning},
year={2025},
url={https://openreview.net/forum?id=zZjLv6F0Ks}
}

@article{
sheshadri2025latent,
title={Latent Adversarial Training Improves Robustness to Persistent Harmful Behaviors in {LLM}s},
author={Abhay Sheshadri and Aidan Ewart and Phillip Huang Guo and Aengus Lynch and Cindy Wu and Vivek Hebbar and Henry Sleight and Asa Cooper Stickland and Ethan Perez and Dylan Hadfield-Menell and Stephen Casper},
journal={Transactions on Machine Learning Research},
issn={2835-8856},
year={2025},
url={https://openreview.net/forum?id=6LxMeRlkWl},
note={}
}

@inproceedings{
huang2025unlearn,
title={Unlearn and Burn: Adversarial Machine Unlearning Requests Destroy Model Accuracy},
author={Yangsibo Huang and Daogao Liu and Lynn Chua and Badih Ghazi and Pritish Kamath and Ravi Kumar and Pasin Manurangsi and Milad Nasr and Amer Sinha and Chiyuan Zhang},
booktitle={The Thirteenth International Conference on Learning Representations},
year={2025},
url={https://openreview.net/forum?id=5xxGP9x5dZ}
}

@inproceedings{
wu2025unlearned,
title={Unlearned but Not Forgotten: Data Extraction after Exact Unlearning in {LLM}},
author={Xiaoyu Wu and Yifei Pang and Terrance Liu and Steven Wu},
booktitle={The Thirty-ninth Annual Conference on Neural Information Processing Systems},
year={2025},
url={https://openreview.net/forum?id=BpAx3OuNOr}
}

@article{doshi2024does,
  title={Does unlearning truly unlearn? a black box evaluation of llm unlearning methods},
  author={Doshi, Jai and Stickland, Asa Cooper},
  journal={arXiv preprint arXiv:2411.12103},
  year={2024}
}

@article{xu2025unlearning,
  title={Unlearning Isn't Deletion: Investigating Reversibility of Machine Unlearning in LLMs},
  author={Xu, Xiaoyu and Yue, Xiang and Liu, Yang and Ye, Qingqing and Zheng, Huadi and Hu, Peizhao and Du, Minxin and Hu, Haibo},
  journal={arXiv preprint arXiv:2505.16831},
  year={2025}
}

@article{hu2025blur,
  title={BLUR: A Benchmark for LLM Unlearning Robust to Forget-Retain Overlap},
  author={Hu, Shengyuan and Kale, Neil and Thaker, Pratiksha and Fu, Yiwei and Wu, Steven and Smith, Virginia},
  journal={arXiv preprint arXiv:2506.15699},
  year={2025}
}

@inproceedings{
hu2025unlearning,
title={Unlearning or Obfuscating? Jogging the Memory of Unlearned {LLM}s via Benign Relearning},
author={Shengyuan Hu and Yiwei Fu and Steven Wu and Virginia Smith},
booktitle={The Thirteenth International Conference on Learning Representations},
year={2025},
url={https://openreview.net/forum?id=fMNRYBvcQN}
}

@inproceedings{lo-etal-2024-large,
    title = "Large Language Models Relearn Removed Concepts",
    author = "Lo, Michelle  and
      Barez, Fazl  and
      Cohen, Shay",
    editor = "Ku, Lun-Wei  and
      Martins, Andre  and
      Srikumar, Vivek",
    booktitle = "Findings of the Association for Computational Linguistics: ACL 2024",
    month = aug,
    year = "2024",
    address = "Bangkok, Thailand",
    publisher = "Association for Computational Linguistics",
    url = "https://aclanthology.org/2024.findings-acl.492/",
    doi = "10.18653/v1/2024.findings-acl.492",
    pages = "8306--8323"
}

@article{barez2025open,
  title={Open problems in machine unlearning for ai safety},
  author={Barez, Fazl and Fu, Tingchen and Prabhu, Ameya and Casper, Stephen and Sanyal, Amartya and Bibi, Adel and O'Gara, Aidan and Kirk, Robert and Bucknall, Ben and Fist, Tim and others},
  journal={arXiv preprint arXiv:2501.04952},
  year={2025}
}

@article{deeb2024unlearning,
  title={Do unlearning methods remove information from language model weights?},
  author={Deeb, Aghyad and Roger, Fabien},
  journal={arXiv preprint arXiv:2410.08827},
  year={2024}
}

@article{chen2023fast,
  title={Fast model debias with machine unlearning},
  author={Chen, Ruizhe and Yang, Jianfei and Xiong, Huimin and Bai, Jianhong and Hu, Tianxiang and Hao, Jin and Feng, Yang and Zhou, Joey Tianyi and Wu, Jian and Liu, Zuozhu},
  journal={Advances in Neural Information Processing Systems},
  volume={36},
  pages={14516--14539},
  year={2023}
}

@article{mavrothalassitis2026ascent,
  title={Ascent fails to forget},
  author={Mavrothalassitis, Ioannis and Puigdemont, Pol and Levi, Noam and Cevher, Volkan},
  journal={Advances in Neural Information Processing Systems},
  volume={38},
  pages={45330--45365},
  year={2026}
}

@inproceedings{
ding2025unified,
title={Unified Parameter-Efficient Unlearning for {LLM}s},
author={Chenlu Ding and Jiancan Wu and Yancheng Yuan and Jinda Lu and Kai Zhang and Alex Su and Xiang Wang and Xiangnan He},
booktitle={The Thirteenth International Conference on Learning Representations},
year={2025},
url={https://openreview.net/forum?id=zONMuIVCAT}
}

@article{li2024delta,
  title={Delta-influence: Unlearning poisons via influence functions},
  author={Li, Wenjie and Li, Jiawei and Zeng, Pengcheng and de Witt, Christian Schroeder and Prabhu, Ameya and Sanyal, Amartya},
  journal={arXiv preprint arXiv:2411.13731},
  year={2024}
}

@article{yosinski2014transferable,
  title={How transferable are features in deep neural networks?},
  author={Yosinski, Jason and Clune, Jeff and Bengio, Yoshua and Lipson, Hod},
  journal={Advances in neural information processing systems},
  volume={27},
  year={2014}
}

@article{grosse2023studying,
  title={Studying large language model generalization with influence functions},
  author={Grosse, Roger and Bae, Juhan and Anil, Cem and Elhage, Nelson and Tamkin, Alex and Tajdini, Amirhossein and Steiner, Benoit and Li, Dustin and Durmus, Esin and Perez, Ethan and others},
  journal={arXiv preprint arXiv:2308.03296},
  year={2023}
}

@inproceedings{koh2017understanding,
  title={Understanding black-box predictions via influence functions},
  author={Koh, Pang Wei and Liang, Percy},
  booktitle={International conference on machine learning},
  pages={1885--1894},
  year={2017},
  organization={PMLR}
}

@inproceedings{jia-etal-2024-soul,
    title = "{SOUL}: Unlocking the Power of Second-Order Optimization for {LLM} Unlearning",
    author = "Jia, Jinghan  and
      Zhang, Yihua  and
      Zhang, Yimeng  and
      Liu, Jiancheng  and
      Runwal, Bharat  and
      Diffenderfer, James  and
      Kailkhura, Bhavya  and
      Liu, Sijia",
    editor = "Al-Onaizan, Yaser  and
      Bansal, Mohit  and
      Chen, Yun-Nung",
    booktitle = "Proceedings of the 2024 Conference on Empirical Methods in Natural Language Processing",
    month = nov,
    year = "2024",
    address = "Miami, Florida, USA",
    publisher = "Association for Computational Linguistics",
    url = "https://aclanthology.org/2024.emnlp-main.245/",
    doi = "10.18653/v1/2024.emnlp-main.245",
    pages = "4276--4292"
}

@inproceedings{sanyal2025agents,
  title={Agents are all you need for LLM unlearning},
  author={Sanyal, Debdeep and Mandal, Murari},
  booktitle={Second Conference on Language Modeling},
  year={2025}
}

@inproceedings{zhang2025llm,
  title={LLM-Eraser: Optimizing Large Language Model Unlearning through Selective Pruning},
  author={Zhang, Shengming and Zhang, Le and Zhou, Jingbo and Zheng, Zhi and Xiong, Hui},
  booktitle={Proceedings of the 31st ACM SIGKDD Conference on Knowledge Discovery and Data Mining V. 1},
  pages={1960--1971},
  year={2025}
}

@article{
lucki2025an,
title={An Adversarial Perspective on Machine Unlearning for {AI} Safety},
author={Jakub {\L}ucki and Boyi Wei and Yangsibo Huang and Peter Henderson and Florian Tram{\`e}r and Javier Rando},
journal={Transactions on Machine Learning Research},
issn={2835-8856},
year={2025},
url={https://openreview.net/forum?id=J5IRyTKZ9s},
note={}
}

@article{xiao2025right,
  title={The Right to be Forgotten in Pruning: Unveil Machine Unlearning on Sparse Models},
  author={Xiao, Yang and Li, Gen and Ji, Jie and Ye, Ruimeng and Ma, Xiaolong and Hui, Bo},
  journal={arXiv preprint arXiv:2507.18725},
  year={2025}
}

@article{pochinkov2024dissecting,
  title={Dissecting language models: Machine unlearning via selective pruning},
  author={Pochinkov, Nicholas and Schoots, Nandi},
  journal={arXiv preprint arXiv:2403.01267},
  year={2024}
}

@article{jia2023model,
  title={Model sparsity can simplify machine unlearning},
  author={Jia, Jinghan and Liu, Jiancheng and Ram, Parikshit and Yao, Yuguang and Liu, Gaowen and Liu, Yang and Sharma, Pranay and Liu, Sijia},
  journal={Advances in Neural Information Processing Systems},
  volume={36},
  pages={51584--51605},
  year={2023}
}

@inproceedings{wu2023depn,
  title={DEPN: Detecting and Editing Privacy Neurons in Pretrained Language Models},
  author={Wu, Xinwei and Li, Junzhuo and Xu, Minghui and Dong, Weilong and Wu, Shuangzhi and Bian, Chao and Xiong, Deyi},
  booktitle={Proceedings of the 2023 Conference on Empirical Methods in Natural Language Processing},
  pages={2875--2886},
  year={2023}
}

@inproceedings{foster2024fast,
  title={Fast machine unlearning without retraining through selective synaptic dampening},
  author={Foster, Jack and Schoepf, Stefan and Brintrup, Alexandra},
  booktitle={Proceedings of the AAAI conference on artificial intelligence},
  volume={38},
  number={11},
  pages={12043--12051},
  year={2024}
}

@article{mahmood2026representation,
  title={Representation-Aware Unlearning via Activation Signatures: From Suppression to Knowledge-Signature Erasure},
  author={Mahmood, Syed Naveed and Bhuiyan, Md Rezaur Rahman and Zaman, Tasfia and Khondaker, Jareen Tasneem and Sakib, Md Sameer and Tasnim, Nazia and Sadeque, Farig},
  journal={arXiv preprint arXiv:2601.10566},
  year={2026}
}

@inproceedings{pawelczyk2024context,
  title={In-Context Unlearning: Language Models as Few-Shot Unlearners},
  author={Pawelczyk, Martin and Neel, Seth and Lakkaraju, Himabindu},
  booktitle={International Conference on Machine Learning},
  pages={40034--40050},
  year={2024},
  organization={PMLR}
}

@misc{
deng2025inferencetime,
title={Inference-time Unlearning via Adaptive Output Regulation},
author={Zhijie Deng and Chris Yuhao Liu and Zirui Pang and Xinlei He and Lei Feng and Qi Xuan and Zhaowei Zhu and Jiaheng Wei},
year={2025},
url={https://openreview.net/forum?id=cuN6DSCS8i}
}

@article{wei2025hubble,
  title={Hubble: a Model Suite to Advance the Study of LLM Memorization},
  author={Wei, Johnny Tian-Zheng and Godbole, Ameya and Khan, Mohammad Aflah and Wang, Ryan and Zhu, Xiaoyuan and Flemings, James and Kashyap, Nitya and Gummadi, Krishna P and Neiswanger, Willie and Jia, Robin},
  journal={arXiv preprint arXiv:2510.19811},
  year={2025}
}

@article{wang2025machine,
  title={When machine unlearning meets retrieval-augmented generation (rag): Keep secret or forget knowledge?},
  author={Wang, Shang and Zhu, Tianqing and Ye, Dayong and Zhou, Wanlei},
  journal={IEEE Transactions on Dependable and Secure Computing},
  year={2025},
  publisher={IEEE}
}

@inproceedings{
loshchilov2018decoupled,
title={Decoupled Weight Decay Regularization},
author={Ilya Loshchilov and Frank Hutter},
booktitle={International Conference on Learning Representations},
year={2019},
url={https://openreview.net/forum?id=Bkg6RiCqY7},
}

@inproceedings{wang-etal-2025-reasoning,
    title = "Reasoning Model Unlearning: Forgetting Traces, Not Just Answers, While Preserving Reasoning Skills",
    author = "Wang, Changsheng  and
      Fan, Chongyu  and
      Zhang, Yihua  and
      Jia, Jinghan  and
      Wei, Dennis  and
      Ram, Parikshit  and
      Baracaldo, Nathalie  and
      Liu, Sijia",
    editor = "Christodoulopoulos, Christos  and
      Chakraborty, Tanmoy  and
      Rose, Carolyn  and
      Peng, Violet",
    booktitle = "Proceedings of the 2025 Conference on Empirical Methods in Natural Language Processing",
    month = nov,
    year = "2025",
    address = "Suzhou, China",
    publisher = "Association for Computational Linguistics",
    url = "https://aclanthology.org/2025.emnlp-main.220/",
    doi = "10.18653/v1/2025.emnlp-main.220",
    pages = "4427--4443",
    ISBN = "979-8-89176-332-6"
}

@inproceedings{li2024gsm,
  title={Gsm-plus: A comprehensive benchmark for evaluating the robustness of llms as mathematical problem solvers},
  author={Li, Qintong and Cui, Leyang and Zhao, Xueliang and Kong, Lingpeng and Bi, Wei},
  booktitle={Proceedings of the 62nd Annual Meeting of the Association for Computational Linguistics (Volume 1: Long Papers)},
  pages={2961--2984},
  year={2024}
}

@inproceedings{
yaolarge,
title={Large Language Model Unlearning},
author={Yuanshun Yao and Xiaojun Xu and Yang Liu},
booktitle={The Thirty-eighth Annual Conference on Neural Information Processing Systems},
year={2024},
url={https://openreview.net/forum?id=8Dy42ThoNe}
}

@inproceedings{thudi2022unrolling,
  title={Unrolling sgd: Understanding factors influencing machine unlearning},
  author={Thudi, Anvith and Deza, Gabriel and Chandrasekaran, Varun and Papernot, Nicolas},
  booktitle={2022 IEEE 7th European Symposium on Security and Privacy (EuroS\&P)},
  pages={303--319},
  year={2022},
  organization={IEEE}
}

@inproceedings{liu2022continual,
  title={Continual learning and private unlearning},
  author={Liu, Bo and Liu, Qiang and Stone, Peter},
  booktitle={Conference on Lifelong Learning Agents},
  pages={243--254},
  year={2022},
  organization={PMLR}
}

@article{li2025editing,
  title={Editing as Unlearning: Are Knowledge Editing Methods Strong Baselines for Large Language Model Unlearning?},
  author={Li, Zexi and Wang, Xiangzhu and Shen, William F and Kurmanji, Meghdad and Qiu, Xinchi and Cai, Dongqi and Wu, Chao and Lane, Nicholas D},
  journal={arXiv preprint arXiv:2505.19855},
  year={2025}
}

@article{wu2025beyond,
  title={Beyond Sharp Minima: Robust LLM Unlearning via Feedback-Guided Multi-Point Optimization},
  author={Wu, Wenhan and Liu, Zheyuan and Gao, Chongyang and Wang, Ren and Ding, Kaize},
  journal={arXiv preprint arXiv:2509.20230},
  year={2025}
}

@article{hossain2025investigating,
  title={Investigating Model Editing for Unlearning in Large Language Models},
  author={Hossain, Shariqah and Kagal, Lalana},
  journal={arXiv preprint arXiv:2512.20794},
  year={2025}
}

@article{
huu2025improving,
title={Improving {LLM} Unlearning Robustness via Random Perturbations},
author={Dang Huu-Tien and Hoang Thanh-Tung and Anh Tuan Bui and Phuong Minh Nguyen and Le-Minh Nguyen and Naoya Inoue},
journal={Transactions on Machine Learning Research},
issn={2835-8856},
year={2026},
url={https://openreview.net/forum?id=QYw192hTdH},
note={}
}

@inproceedings{
lu2022quark,
title={{QUARK}: Controllable Text Generation with Reinforced Unlearning},
author={Ximing Lu and Sean Welleck and Jack Hessel and Liwei Jiang and Lianhui Qin and Peter West and Prithviraj Ammanabrolu and Yejin Choi},
booktitle={Advances in Neural Information Processing Systems},
editor={Alice H. Oh and Alekh Agarwal and Danielle Belgrave and Kyunghyun Cho},
year={2022},
url={https://openreview.net/forum?id=5HaIds3ux5O}
}

@inproceedings{
wang2025large,
title={Large Scale Knowledge Washing},
author={Yu Wang and Ruihan Wu and Zexue He and Xiusi Chen and Julian McAuley},
booktitle={The Thirteenth International Conference on Learning Representations},
year={2025},
url={https://openreview.net/forum?id=dXCpPgjTtd}
}

@article{wang2023concept,
  title={Concept algebra for (score-based) text-controlled generative models},
  author={Wang, Zihao and Gui, Lin and Negrea, Jeffrey and Veitch, Victor},
  journal={Advances in Neural Information Processing Systems},
  volume={36},
  pages={35331--35349},
  year={2023}
}

@article{turner2023steering,
  title={Steering language models with activation engineering},
  author={Turner, Alexander Matt and Thiergart, Lisa and Leech, Gavin and Udell, David and Vazquez, Juan J and Mini, Ulisse and MacDiarmid, Monte},
  journal={arXiv preprint arXiv:2308.10248},
  year={2023}
}

@inproceedings{
zheng2024on,
title={On Prompt-Driven Safeguarding for Large Language Models},
author={Chujie Zheng and Fan Yin and Hao Zhou and Fandong Meng and Jie Zhou and Kai-Wei Chang and Minlie Huang and Nanyun Peng},
booktitle={Forty-first International Conference on Machine Learning},
year={2024},
url={https://openreview.net/forum?id=ugxGpOEkox}
}

@article{wolf2024tradeoffs,
  title={Tradeoffs Between Alignment and Helpfulness in Language Models with Representation Engineering},
  author={Wolf, Yotam and Wies, Noam and Shteyman, Dorin and Rothberg, Binyamin and Levine, Yoav and Shashua, Amnon},
  journal={arXiv preprint arXiv:2401.16332},
  year={2024}
}

@article{rosati2024representation,
  title={Representation noising: A defence mechanism against harmful finetuning},
  author={Rosati, Domenic and Wehner, Jan and Williams, Kai and Bartoszcze, Lukasz and Gonzales, Robie and Majumdar, Subhabrata and Sajjad, Hassan and Rudzicz, Frank and others},
  journal={Advances in Neural Information Processing Systems},
  volume={37},
  pages={12636--12676},
  year={2024}
}

@inproceedings{
rafailov2023direct,
title={Direct Preference Optimization: Your Language Model is Secretly a Reward Model},
author={Rafael Rafailov and Archit Sharma and Eric Mitchell and Christopher D Manning and Stefano Ermon and Chelsea Finn},
booktitle={Thirty-seventh Conference on Neural Information Processing Systems},
year={2023},
url={https://openreview.net/forum?id=HPuSIXJaa9}
}

@inproceedings{
zhang2024negative,
title={Negative Preference Optimization: From Catastrophic Collapse to Effective Unlearning},
author={Ruiqi Zhang and Licong Lin and Yu Bai and Song Mei},
booktitle={First Conference on Language Modeling},
year={2024},
url={https://openreview.net/forum?id=MXLBXjQkmb}
}

@inproceedings{hendel-etal-2023-context,
    title = "In-Context Learning Creates Task Vectors",
    author = "Hendel, Roee  and
      Geva, Mor  and
      Globerson, Amir",
    editor = "Bouamor, Houda  and
      Pino, Juan  and
      Bali, Kalika",
    booktitle = "Findings of the Association for Computational Linguistics: EMNLP 2023",
    month = dec,
    year = "2023",
    address = "Singapore",
    publisher = "Association for Computational Linguistics",
    url = "https://aclanthology.org/2023.findings-emnlp.624/",
    doi = "10.18653/v1/2023.findings-emnlp.624",
    pages = "9318--9333"
}

@article{brown2020language,
  title={Language models are few-shot learners},
  author={Brown, Tom and Mann, Benjamin and Ryder, Nick and Subbiah, Melanie and Kaplan, Jared D and Dhariwal, Prafulla and Neelakantan, Arvind and Shyam, Pranav and Sastry, Girish and Askell, Amanda and others},
  journal={Advances in neural information processing systems},
  volume={33},
  pages={1877--1901},
  year={2020}
}

@inproceedings{
marks2024the,
title={The Geometry of Truth: Emergent Linear Structure in Large Language Model Representations of True/False Datasets},
author={Samuel Marks and Max Tegmark},
booktitle={First Conference on Language Modeling},
year={2024},
url={https://openreview.net/forum?id=aajyHYjjsk}
}

@article{arditi2024refusal,
  title={Refusal in language models is mediated by a single direction},
  author={Arditi, Andy and Obeso, Oscar and Syed, Aaquib and Paleka, Daniel and Panickssery, Nina and Gurnee, Wes and Nanda, Neel},
  journal={Advances in Neural Information Processing Systems},
  volume={37},
  pages={136037--136083},
  year={2024}
}

@article{li2023inference,
  title={Inference-time intervention: Eliciting truthful answers from a language model},
  author={Li, Kenneth and Patel, Oam and Vi{\'e}gas, Fernanda and Pfister, Hanspeter and Wattenberg, Martin},
  journal={Advances in Neural Information Processing Systems},
  volume={36},
  pages={41451--41530},
  year={2023}
}

@article{tigges2023linear,
  title={Linear representations of sentiment in large language models},
  author={Tigges, Curt and Hollinsworth, Oskar John and Geiger, Atticus and Nanda, Neel},
  journal={arXiv preprint arXiv:2310.15154},
  year={2023}
}

@inproceedings{lin2022truthfulqa,
  title={Truthfulqa: Measuring how models mimic human falsehoods},
  author={Lin, Stephanie and Hilton, Jacob and Evans, Owain},
  booktitle={Proceedings of the 60th annual meeting of the association for computational linguistics (volume 1: long papers)},
  pages={3214--3252},
  year={2022}
}

@inproceedings{
maini2024tofu,
title={{TOFU}: A Task of Fictitious Unlearning for {LLM}s},
author={Pratyush Maini and Zhili Feng and Avi Schwarzschild and Zachary Chase Lipton and J Zico Kolter},
booktitle={First Conference on Language Modeling},
year={2024},
url={https://openreview.net/forum?id=B41hNBoWLo}
}

@article{yao2024large,
  title={Large language model unlearning},
  author={Yao, Yuanshun and Xu, Xiaojun and Liu, Yang},
  journal={Advances in Neural Information Processing Systems},
  volume={37},
  pages={105425--105475},
  year={2024}
}

@article{ren2025sok,
  title={SoK: Machine Unlearning for Large Language Models},
  author={Ren, Jie and Xing, Yue and Cui, Yingqian and Aggarwal, Charu C and Liu, Hui},
  journal={arXiv preprint arXiv:2506.09227},
  year={2025}
}

@article{10.1145/3603620,
author = {Xu, Heng and Zhu, Tianqing and Zhang, Lefeng and Zhou, Wanlei and Yu, Philip S.},
title = {Machine Unlearning: A Survey},
year = {2023},
issue_date = {January 2024},
publisher = {Association for Computing Machinery},
address = {New York, NY, USA},
volume = {56},
number = {1},
issn = {0360-0300},
url = {https://doi.org/10.1145/3603620},
doi = {10.1145/3603620},
journal = {ACM Comput. Surv.},
month = aug,
articleno = {9},
numpages = {36}
}

@article{nguyen2025survey,
  title={A survey of machine unlearning},
  author={Nguyen, Thanh Tam and Huynh, Thanh Trung and Ren, Zhao and Nguyen, Phi Le and Liew, Alan Wee-Chung and Yin, Hongzhi and Nguyen, Quoc Viet Hung},
  journal={ACM Transactions on Intelligent Systems and Technology},
  volume={16},
  number={5},
  pages={1--46},
  year={2025},
  publisher={ACM New York, NY}
}

@inproceedings{bourtoule2021machine,
  title={Machine unlearning},
  author={Bourtoule, Lucas and Chandrasekaran, Varun and Choquette-Choo, Christopher A and Jia, Hengrui and Travers, Adelin and Zhang, Baiwu and Lie, David and Papernot, Nicolas},
  booktitle={2021 IEEE symposium on security and privacy (SP)},
  pages={141--159},
  year={2021},
  organization={IEEE}
}

@INPROCEEDINGS{7163042,
  author={Cao, Yinzhi and Yang, Junfeng},
  booktitle={2015 IEEE Symposium on Security and Privacy}, 
  title={Towards Making Systems Forget with Machine Unlearning}, 
  year={2015},
  volume={},
  number={},
  pages={463-480},
  doi={10.1109/SP.2015.35}}

@inproceedings{thaker2025position,
  title={Position: Llm unlearning benchmarks are weak measures of progress},
  author={Thaker, Pratiksha and Hu, Shengyuan and Kale, Neil and Maurya, Yash and Wu, Zhiwei Steven and Smith, Virginia},
  booktitle={2025 IEEE Conference on Secure and Trustworthy Machine Learning (SaTML)},
  pages={520--533},
  year={2025},
  organization={IEEE}
}

@misc{logitlensblog,
  title={interpreting {GPT}: the logit lens},
  url={https://www.lesswrong.com/posts/AcKRB8wDpdaN6v6ru/interpreting-gpt-the-logit-lens},
  note={Accessed: 2026-01-13},
  author={nostalgebraist},
  year={2020},
}

@inproceedings{weiassessing,
  title={Assessing the Brittleness of Safety Alignment via Pruning and Low-Rank Modifications},
  author={Wei, Boyi and Huang, Kaixuan and Huang, Yangsibo and Xie, Tinghao and Qi, Xiangyu and Xia, Mengzhou and Mittal, Prateek and Wang, Mengdi and Henderson, Peter},
  booktitle={Forty-first International Conference on Machine Learning},
  year={2024}
}

@inproceedings{huunlearning,
  title={Unlearning or Obfuscating? Jogging the Memory of Unlearned LLMs via Benign Relearning},
  author={Hu, Shengyuan and Fu, Yiwei and Wu, Steven and Smith, Virginia},
  booktitle={The Thirteenth International Conference on Learning Representations},
  url={https://openreview.net/forum?id=fMNRYBvcQN},
  year={2025}
}

@article{hu2022lora,
  title={Lora: Low-rank adaptation of large language models.},
  author={Hu, Edward J and Shen, Yelong and Wallis, Phillip and Allen-Zhu, Zeyuan and Li, Yuanzhi and Wang, Shean and Wang, Lu and Chen, Weizhu and others},
  journal={ICLR},
  volume={1},
  number={2},
  pages={3},
  year={2022},
  url={https://openreview.net/forum?id=nZeVKeeFYf9}
}

@misc{diffinmeans,
  title={Diff-in-Means Concept Editing is Worst-Case Optimal},
  url={https://blog.eleuther.ai/diff-in-means/},
  note={Accessed: 2026-01-13},
  author={Nora Belrose},
  year={2023},
}

@inproceedings{leesnip,
  title={SNIP: SINGLE-SHOT NETWORK PRUNING BASED ON CONNECTION SENSITIVITY},
  author={Lee, Namhoon and Ajanthan, Thalaiyasingam and Torr, Philip},
  booktitle={International Conference on Learning Representations},
  year={2019}
}

@article{cobbe2021training,
  title={Training verifiers to solve math word problems},
  author={Cobbe, Karl and Kosaraju, Vineet and Bavarian, Mohammad and Chen, Mark and Jun, Heewoo and Kaiser, Lukasz and Plappert, Matthias and Tworek, Jerry and Hilton, Jacob and Nakano, Reiichiro and others},
  journal={arXiv preprint arXiv:2110.14168},
  year={2021}
}

@article{fan2026simplicity,
  title={Simplicity prevails: Rethinking negative preference optimization for llm unlearning},
  author={Fan, Chongyu and Liu, Jiancheng and Lin, Licong and Jia, Jinghan and Zhang, Ruiqi and Mei, Song and Liu, Sijia},
  journal={Advances in Neural Information Processing Systems},
  volume={38},
  pages={1540--1567},
  year={2026}
}

@article{yang2024qwen2,
  title={Qwen2. 5 Technical Report},
  author={Yang, An and Yang, Baosong and Zhang, Beichen and Hui, Binyuan and Zheng, Bo and Yu, Bowen and Li, Chengyuan and Liu, Dayiheng and Huang, Fei and Wei, Haoran and others},
  journal={arXiv e-prints},
  pages={arXiv--2412},
  year={2024}
}

@article{jia2024wagle,
  title={Wagle: Strategic weight attribution for effective and modular unlearning in large language models},
  author={Jia, Jinghan and Liu, Jiancheng and Zhang, Yihua and Ram, Parikshit and Baracaldo, Nathalie and Liu, Sijia},
  journal={Advances in Neural Information Processing Systems},
  volume={37},
  pages={55620--55646},
  year={2024}
}

@inproceedings{chen2026unlearning,
  title={Unlearning Isn't Invisible: Detecting Unlearning Traces in LLMs from Model Outputs},
  author={Chen, Yiwei and Pal, Soumyadeep and Zhang, Yimeng and Qu, Qing and Liu, Sijia},
  booktitle={International Conference on Learning Representations},
  volume={2026},
  pages={111444--111470},
  year={2026}
}

@inproceedings{
chen2026persona,
title={Persona Vectors: Monitoring and Controlling Character Traits in Language Models},
author={Runjin Chen and Andy Arditi and Henry Sleight and Owain Evans and Jack Lindsey},
booktitle={The Fortieth Annual Conference on Neural Information Processing Systems},
year={2026},
url={https://openreview.net/forum?id=QvUJ3iQrY6}
}

@article{cook1982residuals,
  title={Residuals and influence in regression},
  author={Cook, R Dennis and Weisberg, Sanford},
  year={1982},
  publisher={New York: Chapman and Hall}
}

@inproceedings{
frantar2023optq,
title={{OPTQ}: Accurate Quantization for Generative Pre-trained Transformers},
author={Elias Frantar and Saleh Ashkboos and Torsten Hoefler and Dan Alistarh},
booktitle={The Eleventh International Conference on Learning Representations },
year={2023},
url={https://openreview.net/forum?id=tcbBPnfwxS}
}

@inproceedings{zellers2019hellaswag,
  title={Hellaswag: Can a machine really finish your sentence?},
  author={Zellers, Rowan and Holtzman, Ari and Bisk, Yonatan and Farhadi, Ali and Choi, Yejin},
  booktitle={Proceedings of the 57th annual meeting of the association for computational linguistics},
  pages={4791--4800},
  year={2019}
}
\bibliographystyle{icml2027}

%%%%%%%%%%%%%%%%%%%%%%%%%%%%%%%%%%%%%%%%%%%%%%%%%%%%%%%%%%%%%%%%%%%%%%%%%%%%%%%
%%%%%%%%%%%%%%%%%%%%%%%%%%%%%%%%%%%%%%%%%%%%%%%%%%%%%%%%%%%%%%%%%%%%%%%%%%%%%%%
% APPENDIX
%%%%%%%%%%%%%%%%%%%%%%%%%%%%%%%%%%%%%%%%%%%%%%%%%%%%%%%%%%%%%%%%%%%%%%%%%%%%%%%
%%%%%%%%%%%%%%%%%%%%%%%%%%%%%%%%%%%%%%%%%%%%%%%%%%%%%%%%%%%%%%%%%%%%%%%%%%%%%%%
% \newpage
\onecolumn
\appendix
\appendixpage
\label{appendix}

\begin{itemize}
   \item[] \hyperref[appendix:related_works]{A. Extended related work} \dotfill \pageref{appendix:related_works}
   
    \item[] \hyperref[appendix:unlearning_baselines]{B. Unlearning baselines} \dotfill \pageref{appendix:unlearning_baselines}

    \item[] \hyperref[appendix:additional_concepts]{C. Experiments on stereotypical bias and language} \dotfill \pageref{appendix:additional_concepts}
    \begin{itemize}
        \item[] \hyperref[appendix:stereotype_bias]{C.1 Stereotypical bias} \dotfill \pageref{appendix:stereotype_bias}
        
        \item[] \hyperref[appendix:language]{C.2 Language} \dotfill \pageref{appendix:language}
    \end{itemize}
  
     \item[] \hyperref[appendix:experimental_setup]{D. Full experimental setup} \dotfill \pageref{appendix:experimental_setup}

  % \begin{itemize}
  \begin{itemize}
    \item[] \hyperref[appendix:benchmarks]{D.1 Benchmarks} \dotfill \pageref{appendix:benchmarks}
    
    \item[] \hyperref[appendix:implementation_details]{D.2 Implementation details} \dotfill \pageref{appendix:implementation_details}
  \end{itemize}
  \item[] \hyperref[appendix:prompt_template]{E. Prompt templates} \dotfill \pageref{appendix:prompt_template}
    \begin{itemize}
    \item[] \hyperref[appendix:refusal_substrings]{E.1 Refusal substrings} \dotfill \pageref{appendix:refusal_substrings}
    \item[] \hyperref[appendix:template_for_tasks]{E.2 Prompt template for tasks} \dotfill \pageref{appendix:template_for_tasks}

  \end{itemize}

    \item[] \hyperref[appendix:effect_of_probe]{F. Effects of probes: Probe selection and sample efficiency} \dotfill \pageref{appendix:effect_of_probe}

    \item[] \hyperref[appendix:rmu_concept]{G. Is the phenomenon specific to RAd and RAb?} \dotfill \pageref{appendix:rmu_concept}
    
    \item[] \hyperref[appendix:deeper_layer]{H. Performance of RAd and RAb at deeper layers} \dotfill \pageref{appendix:deeper_layer}
    
    \item[] \hyperref[appendix:robustness]{I. Robustness of RAd and RAb models against knowledge recovery} \dotfill \pageref{appendix:robustness}
    
    \item[] \hyperref[appendix:muse]{H. Experiments on MUSE benchmark} \dotfill \pageref{appendix:muse}

    \item[] \hyperref[appendix:theoretical_analysis]{J. Proof of Proposition~\ref{proposition1}} \dotfill \pageref{appendix:theoretical_analysis}
    
    \item[] \hyperref[appendix:additional_results]{K. Additional results} \dotfill \pageref{appendix:additional_results}

     \item[] \hyperref[appendix:limitation]{L. Limitation and future work} \dotfill \pageref{appendix:limitation}
\end{itemize}

\section{Extended related work}
\label{appendix:related_works}
\textbf{Machine unlearning for LLMs.} MU has emerged as a popular tool for removing undesirable knowledge from LLMs, including sensitive, toxic, private information~\citep{lu2022quark, jang2023knowledge, zhang2023composing, wu-etal-2023-depn, wang2025large, wei2025hubble}, copyrighted materials~\citep{eldan2023s, yao2024large, thaker2024guardrail, shi2025muse}, and hazardous knowledge in domains such as biology and cybersecurity in LLMs~\citep{li2024wmdp, liu2024large, huu2025improving, fan2025simplicity}.

\textbf{Training-based unlearning.} Training-based MU methods can be broadly categorized into two paradigms. First, representation misdirection aims to redirect internal representations to suppress target knowledge~\citep{rosati2024representation,li2024wmdp, dang2025effects, ren-etal-2025-general, shen2025llm, chen2025feature, mahmood2026representation, ren-etal-2025-general}. Second, preference optimization based methods, which minimize the model's likelihood on forget-samples~\citep{maini2024tofu, yuan2025a, fan2025simplicity, zhang2024negative, wang2025rethinking}. 

\textbf{Training-free unlearning.} Beyond training, training-free approaches have been proposed, including inference-time unlearning~\citep{deng2025inferencetime, sanyal2025agents, liu2024large, wang2025machine}, in-context unlearning~\citep{pawelczyk2024context}, and guardrail-based unlearning~\citep{thaker2024guardrail}.

\textbf{Other perspectives.} Other lines of work explore structural MU, such as pruning-based methods, which prune neurons or parameters associated with undesired knowledge~\citep{wu2023depn, jia2023model, foster2024fast, pochinkov2024dissecting, xiao2025right, zhang2025llm, jia2024wagle}. Influence functions~\citep{koh2017understanding, grosse2023studying} approximate the influence of training data points on model predictions~\citep{ chen2023fast, li2024delta, gu2024second, jia-etal-2024-soul, ding2025unified}. Unlearning via model merging~\cite{kuo2025exact}, editing~\citep{hossain2025investigating, li2025editing}. Unlearning with specific models, such as reasoning models~\citep{wang-etal-2025-reasoning}. Unlearning using SAEs~\citep{yamashita2025sparse, muhamed2025saes, farrell2024applying, wang-etal-2025-model-unlearning}.

\textbf{Linear Representation Hypothesis.} The idea of the linear representation hypothesis can be broadly formulated in three notions. First, a concept is represented as a one-dimensional language model's subspace~\citep{mikolov-etal-2013-linguistic, pennington-etal-2014-glove, arora-etal-2016-latent, elhage2022toy}. Second, as a measurement (\textit{e.g.,}~\citet{nanda-etal-2023-emergent, gurnee2024language}), \textit{i.e.,} concept output probabilities are logit-linear of representations. Third, as an intervention (\textit{e.g.,} ~\citet{wang2023concept, turner2025steering}): adding suitable steering vectors shifts a concept without changing other concepts. Recently, ~\citet{park2024linear, park2025the} introduced the notion of a causal inner product that aligns the latent and unembedding representations to unify these three notions. 

\textbf{Unlearning robustness.} Recent studies revealed that unlearned models are brittle to knowledge recovery, \textit{i.e.,} unlearned knowledge can be recovered through relearning~\citep{li2024wmdp, deeb2024unlearning, lo-etal-2024-large, xu2025unlearning}, knowledge recovery attacks~\citep{hu2025unlearning, lucki2025an, wu2025unlearned, huang2025unlearn}, or even benign perturbations~\citep{thaker2025position, hu2025blur, huu2025improving, ren2025keeping}, finetuning on forget-unrelated tasks~\citep{lucki2025an, doshi2024does}. Researchers developed robust methods for LLM unlearning, such as sharpness-aware minimization based~\citep{fan2025towards, yan2025dual}, random noise augmentation~\citep{huu2025improving}, invariant risk
minimization~\citep{wang2025invariance}, latent adversarial training~\citep{sheshadri2025latent}, tamper-resistant safeguards~\citep{tamirisa2025tamperresistant}, and feedback-guided multi-point optimization~\citep{wu2025beyond}.

\section{Unlearning baselines}
\label{appendix:unlearning_baselines}
\textbf{Representation Misdirection for Unlearning (RMU;\textnormal{~\citet{li2024wmdp}})} pushes the forget-samples' representations to a fixed random vector $c\mathbf{u}$, where $\mathbf{u} \in \mathbb{R}^{d_l}$ is a unit vector with each element uniformly sampled from $[0,1)$, and $c\in \mathbb{R}^{+}$ is the scaling coefficient. RMU optimizes the following two-part loss:
\begin{align}
\mathcal{L}^{\text{RMU}}= \mathbb{E}_{\mathbf{x}^{f} \sim \mathcal{D}_{f}} \left[\left\|\lambda_{\bm\theta}^{f}-c\mathbf{u}\right\|_2^2\right] + \alpha_r \mathbb{E}_{\mathbf{x}^{r} \sim \mathcal{D}_{r}} \left[\left\|\lambda_{\bm\theta}^r-\lambda_{\bm\theta^{\text{ref}}}^r\right\|_2^2\right],
\end{align}
where $\bm\theta$ and $\bm\theta^{\text{ref}}$ are the parameters of the updated and reference (frozen weight) models, respectively.

\textbf{Gradient Ascent\footnote{GA is one of the most widely adopted unlearning methods, having been employed extensively across the unlearning literature. Here, we cite representative works that use GA for LLM unlearning.} (GA;\textnormal{~\citet{yaolarge, maini2024tofu, shi2025muse, mavrothalassitis2026ascent})}} minimizes the unconditional likelihood of forget-samples at the output logits. GA loss is defined as
\begin{align}
\mathcal{L}^{\text{GA}}
= -\alpha_f\mathbb{E}_{(\mathbf{x}^{f}, \mathbf{y}^{f})\sim \mathcal{D}_{f}}\left[-\log \pi_{\bm\theta}(\mathbf{y}^{f}|\mathbf{x}^{f})\right]
= \alpha_f\mathbb{E}_{(\mathbf{x}^{f}, \mathbf{y}^{f})\sim \mathcal{D}_{f}}\left[\log \pi_{\bm\theta}(\mathbf{y}^{f}|\mathbf{x}^{f})\right],
\end{align}
where $\pi_{\bm\theta}(\mathbf{y}^{f}|\mathbf{x}^{f})$ denotes the model’s predicted probability of forget-sample $(\mathbf{x}^{f},\mathbf{y}^{f})$.

\textbf{Direct Preference Optimization (DPO).} \citet{zhang2024negative, maini2024tofu, yuan2025a}) adopt standard DPO~\citep{rafailov2023direct}, that use refusal answers $\mathbf{y}^{\text{ref}} \in \mathcal{D}_{\text{ref}}$ such as ``I Don't Know'' as the positive samples and  forget-samples as negative samples:
\begin{align}
    \mathcal{L}^{\text{DPO}} = \alpha_f\mathbb{E}_{(\mathbf{x}^f,\mathbf{y}^f)\sim \mathcal{D}_f}\left[-\frac{2}{\beta} \log \sigma \left(\beta \left[\log \frac{\pi_{\bm\theta}(\mathbf{y}^{\text{ref}}|\mathbf{x}^{f})}{\pi_{\bm\theta}(\mathbf{y}^f|\mathbf{x}^{f})} - \log \frac{\pi_{\bm\theta^{\text{ref}}}(\mathbf{y}^{\text{ref}}|\mathbf{x}^{f})}{\pi_{\bm\theta^{\text{ref}}}(\mathbf{y}^f|\mathbf{x}^{f})}\right]\right)\right]\label{eq:dpo}
\end{align}
where $\beta \in \mathbb{R}^{+}$ is a hyperparameter, $\sigma$ is the sigmoid function, and $\pi_{\bm\theta^{\text{ref}}}(\mathbf{y}^f|\mathbf{x}^f)$ denotes the predicted probability of $\mathbf{y}^f$ given $\mathbf{x}^f$ in the reference model $f_{\bm\theta^{\text{ref}}}$.

\textbf{Negative Preference Optimization (NPO;\textnormal{~\citet{zhang2024negative}})} extends GA by incorporating adaptive gradient weights to enable more controlled and stable optimization, thereby mitigating the catastrophic collapse observed in GA:
\begin{align}
    \mathcal{L}^{\text{NPO}} = \alpha_f\mathbb{E}_{(\mathbf{x}^f,\mathbf{y}^f)\sim \mathcal{D}_f} \left[-\frac{2}{\beta}\log\sigma\left(-\beta\log\left(\frac{\pi_{\bm\theta}(\mathbf{y}^f|\mathbf{x}^f)}{\pi_{\bm\theta^{\text{ref}}}(\mathbf{y}^f|\mathbf{x}^f)}\right)\right)\right],
\end{align}

\textbf{Simple NPO (SimNPO;\textnormal{~\citet{fan2026simplicity}})} simplifies NPO by using a normalized sequence log-probability that is divided by the output length, and it adds a margin term with a hyperparameter $\gamma \geq 0$:
\begin{align}
    \mathcal{L}^{\text{SimNPO}} = \alpha_f\mathbb{E}_{(\mathbf{x}^f,\mathbf{y}^f)\sim \mathcal{D}_f} \left[-\frac{2}{\beta}\log\sigma\left(-\frac{\beta}{|\mathbf{y}^{f}|}\log\pi_{\theta}(\mathbf{y}^f|\mathbf{x}^f) -\gamma\right)\right],
\end{align}
where $|\mathbf{y}^f|$ is the length of output sequence $\mathbf{y}^{f}$. 

\textbf{Retain-losses.} We employ Mean Squared Error (MSE): $\mathcal{L}^{\text{MSE}} =  \alpha_r\mathbb{E}_{(\mathbf{x}^{r},\mathbf{y}^r) \sim \mathcal{D}_{r}} ||\log\pi_{\bm\theta}(\mathbf{x}^{r})-\log\pi_{\bm\theta^{\text{ref}}}(\mathbf{x}^{r})||^2$ or Kullback--Leibler divergence (KL): $\mathcal{L}^{\text{KL}} = \alpha_r \mathbb{E}_{(\mathbf{x}^{r},\mathbf{y}^r) \sim \mathcal{D}_{r}} \text{KL}\left(\log\pi_{\bm\theta}(\mathbf{x}^{r}),\log\pi_{\bm\theta^{\text{ref}}}(\mathbf{x}^{r})\right)$ as the retain-loss.

\section{Experiments on stereotypical bias and language}
\label{appendix:additional_concepts}
In this section, we investigate whether RAd and RAb help to reduce stereotypical biases (Sec.~\ref{appendix:stereotype_bias}) and control the generated language (Sec.~\ref{appendix:language}).
\subsection{Stereotypical bias}
\label{appendix:stereotype_bias}
We employ the StereoSet intrasentence task~\citep{nadeem-etal-2021-stereoset}, a benchmark for measuring stereotypical biases in four domains: gender, profession, race, and religion. The intrasentence task evaluates bias at the sentence level using a fill-in-the-blank format. We reorganize StereoSet such that each QA pair is associated with a stereotype label, where stereotype and anti-stereotype are assigned labels $1$ and $0$, respectively. The original StereoSet dataset includes the development and test sets. For evaluation, we randomly split the original development set in a $8:2$ ratio to obtain training and validation sets for training the logistic regression probe, and we use the original test set for evaluating the unlearned models. An example from StereoSet is shown in Fig.~\ref{fig:stereotype_example}.

\begin{wrapfigure}{r}{0.45\textwidth}
    \vspace{-1.5em}
    \begin{tcolorbox}[boxrule=0pt, colframe=white, fontupper=\normalfont\small]
        Choose the appropriate word.
        
        ``Afghanistan is a\;\textunderscore\textunderscore\textunderscore\;  country.''

        A. important \hfill (anti-stereotype)

        B. frame \hfill (unrelated)

        C. dangerous \hfill (stereotype)

        Answer:
    \end{tcolorbox}
    % \vspace{-0.5em}
    \caption{An example from the StereoSet (race domain).}
    \label{fig:stereotype_example}
    \vspace{-1.0em}
\end{wrapfigure}

We define two concepts: \texttt{ster$\to$anti-ster} and \texttt{anti-ster$\to$ster}. Concept \texttt{ster$\to$anti-ster} defines the direction associated with \textit{reducing biases}. In contrast, \texttt{anti-ster$\to$ster} defines the direction associated with \textit{increasing biases}.

We measure how often the unlearned model selects the stereotype, anti-stereotype, or unrelated based on the model's assigned probabilities for the answer tokens. We compute \textit{stereotype, anti-stereotype, and unrelated rates} as the percentage of examples where the unlearned model assigns the highest probability to the stereotype, anti-stereotype, or unrelated, respectively. Unlearned models are expected to have a lower stereotype rate and a higher anti-stereotype rate (less bias). 

\begin{table}[h]
\caption{Performance of RAd and RAb unlearned models on WMDP, MMLU, and StereoSet.
\textcolor{blue}{Increases} and \textcolor{purple}{drops} are marked (compared to the corresponding base models).}
\label{tab:stereotype_intrasentence_mc}
\centering
\resizebox{0.75\linewidth}{!}{
\setlength{\tabcolsep}{6pt}
\begin{tabular}{llccccc}
\toprule
\multirow{2}{*}{\textbf{Model}} & \multirow{2}{*}{\textbf{Method}} 
& \multicolumn{3}{c}{\textbf{StereoSet Intrasentence}} 
& \multirow{2}{*}{\textbf{MMLU}} 
& \multirow{2}{*}{\textbf{WMDP}} \\
\cmidrule(lr){3-5}
& & \textbf{Stereotype} 
& \textbf{Anti-stereotype}
& \textbf{Unrelated}
&  &  \\
\midrule
\multirow{5}{*}{\textbf{Zephyr-7B}} 
& Base model & $52.7$ & $35.8$ & $11.5$ & $58.4$ & $54.4$ \\

\cmidrule{2-7}
& RAd w/ random & $53.0${\footnotesize\textcolor{blue}{$+0.3$}} & $36.7${\footnotesize\textcolor{blue}{$+0.9$}} & $10.3$ & $55.7$ & $26.8$ \\
& \cellcolor{gray!15}RAd w/ ster$\to$anti-ster & \cellcolor{gray!15}$47.7${\footnotesize\textcolor{purple}{$-5.0$}} & \cellcolor{gray!15}$40.2${\footnotesize\textcolor{blue}{$+4.4$}} & \cellcolor{gray!15}$12.1$ & \cellcolor{gray!15}$53.1$ & \cellcolor{gray!15}$25.6$ \\
% & \cellcolor{gray!15}RAd w/ anti-ster$\to$ster & \cellcolor{gray!15}$61.9${\footnotesize\textcolor{blue}{$+9.2$}} & \cellcolor{gray!15}$22.3${\footnotesize\textcolor{red}{$-13.5$}} & \cellcolor{gray!15}$15.8$ & \cellcolor{gray!15}$53.8$ & \cellcolor{gray!15}$26.9$ \\
\cmidrule{2-7}
& RAb w/ random & $53.5${\footnotesize\textcolor{blue}{$+0.8$}} & $37.3${\footnotesize\textcolor{blue}{$+1.5$}}
& $9.2$ & $56.8$ & $43.8$ \\
& \cellcolor{gray!15}RAb w/ anti-ster$\to$ster & \cellcolor{gray!15}$45.6${\footnotesize\textcolor{purple}{$-7.1$}} & \cellcolor{gray!15}$37.6${\footnotesize\textcolor{blue}{$+1.8$}} & \cellcolor{gray!15}$16.8$ & \cellcolor{gray!15}$50.2$ & \cellcolor{gray!15}$29.8$ \\
% & \cellcolor{gray!15}RAb w/ ster$\to$anti-ster & \cellcolor{gray!15}$49.4${\footnotesize\textcolor{red}{$-3.3$}} & \cellcolor{gray!15}$37.5${\footnotesize\textcolor{blue}{$+1.7$}} & \cellcolor{gray!15}$13.1$ & \cellcolor{gray!15}$52.7$ & \cellcolor{gray!15}$35.3$ \\
\midrule 
% \midrule
\multirow{5}{*}{\textbf{Mistral-7B}} 
& Base model & $60.8$ & $31.4$ & $7.8$ & $59.6$ & $55.7$ \\

\cmidrule{2-7}
& RAd w/ random & $46.8${\footnotesize\textcolor{purple}{$-14.0$}} & $46.3${\footnotesize\textcolor{blue}{$+14.9$}}
& $6.9$ & $56.8$ & $27.3$ \\
& \cellcolor{gray!15}RAd w/ ster$\to$anti-ster & \cellcolor{gray!15}$36.6${\footnotesize\textcolor{purple}{$-24.2$}} & \cellcolor{gray!15}$53.5${\footnotesize\textcolor{blue}{$+22.1$}} & \cellcolor{gray!15}$9.9$ & \cellcolor{gray!15}$56.1$ & \cellcolor{gray!15}$27.5$ \\
% & \cellcolor{gray!15}RAd w/ anti-ster$\to$ster & \cellcolor{gray!15}$59.9${\footnotesize\textcolor{red}{$-0.9$}} & \cellcolor{gray!15}$31.7${\footnotesize\textcolor{blue}{$+0.3$}} & \cellcolor{gray!15}$8.4$ & \cellcolor{gray!15}$57.9$ & \cellcolor{gray!15}$27.5$ \\
\cmidrule{2-7}
& RAb w/ random & $54.5${\footnotesize\textcolor{purple}{$-6.3$}} & $38.0${\footnotesize\textcolor{blue}{$+6.6$}} & $7.5$ & $57.5$ & $47.6$ \\
& \cellcolor{gray!15}RAb w/ anti-ster$\to$ster & \cellcolor{gray!15}$42.5${\footnotesize\textcolor{purple}{$-18.3$}} & \cellcolor{gray!15}$36.5${\footnotesize\textcolor{purple}{$+5.1$}} & \cellcolor{gray!15}$21.0$ & \cellcolor{gray!15}$53.2$ & \cellcolor{gray!15}$34.8$ \\
% & \cellcolor{gray!15}RAb w/ ster$\to$anti-ster & \cellcolor{gray!15}$48.5${\footnotesize\textcolor{red}{$-12.3$}} & \cellcolor{gray!15}$37.4${\footnotesize\textcolor{blue}{$+6.0$}} & \cellcolor{gray!15}$14.1$ & \cellcolor{gray!15}$54.5$ & \cellcolor{gray!15}$38.1$ \\
\bottomrule
\end{tabular}}
\end{table}

Tab.~\ref{tab:stereotype_intrasentence_mc} shows that RAd with \texttt{ster$\to$anti-ster} or RAb with \texttt{anti-ster$\to$ster} consistently reduce stereotypical biases for both Zephyr-7B and Mistral-7B. For RAd with \texttt{ster$\to$anti-ster}, both Zephyr-7B and Mistral-7B exhibit a lower stereotype rate and a higher anti-stereotype rate compared with their corresponding base models. For Zephyr-7B, the stereotype rate decreases from $52.7$ to $47.7$, while the anti-stereotype rate increases from $35.8$ to $40.2$. For Mistral-7B, the stereotype rate decreases from $60.8$ to $36.6$, and the anti-stereotype rate increases from $31.4$ to $53.5$.
RAb with \texttt{anti-ster$\to$ster} produces a similar reduction. For Zephyr-7B, the stereotype rate decreases from $52.7$ to $45.6$, while the anti-stereotype rate increases from $35.8$ to $37.6$ and the unrelated rate increases from $11.5$ to $16.8$. For Mistral-7B, the stereotype rate decreases from $60.8$ to $42.5$, accompanied by an increase in the unrelated rate from $7.8$ to $21.0$. Overall, these results demonstrate that both RAd and RAb can reduce the unlearned model's stereotypical biases.

\subsection{Language}
\label{appendix:language}
We employ HellaSwag~\citep{zellers2019hellaswag}, a dataset for natural language completion. Each sample contains English sentences, and the model is asked to generate a continuation. We aim to control the language of its generation. Each sample in the training split is formatted using two templates: a zero-shot template (``\texttt{Finish this sentence: \{context\}} \texttt{Answer:}'') and a language-specific template (``\texttt{Finish this sentence: \{context\}} \texttt{Answer: in \{language\}:}''). The language-specific vector (\textit{e.g.,} en$\to$fr) is the normalized weights of the logistic regression probe that was trained to distinguish between zero-shot samples' representations (labeled $0$) and language-specific samples' representations (labeled $1$). We conduct experiments across four languages: English to French (en$\to$fr), English to Spanish (en$\to$es), English to Japanese (en$\to$ja), and English to Vietnamese (en$\to$vi).

For evaluation, we define the \emph{language presence rate} (LPR) as the fraction of samples in which a target language appears. A higher LPR implies the model tends to generate text in the target language. More formally, given a dataset $\mathcal{D} = \{\mathbf{x}_i\}_{i=1}^{N}$, LPR is defined as the fraction of samples in which a target language $l$ (\textit{e.g.,} Japanese) appears in the text: $\textnormal{LPR}(l, \mathcal{D}) = \frac{\sum_{i}^{N}\mathbb{I} (l \in L(\mathbf{x}_i))}{N} \in [0,1]$, where $\mathbb{I}(\cdot)$ is the indicator function. $L(\mathbf{x}_i)$ is the set of detected languages in sample $\mathbf{x}_i$. 
A rate of $1$ indicates the target language appears in all samples, while a rate of $0$ indicates the target language does not appear in any. We employ Lingua, a language detection framework that supports multiple languages, to detect language within a given text. Lingua is publicly available at \url{https://github.com/pemistahl/lingua-py}.

\textit{Example:} Consider $\mathcal{D}$ with $N=3$ samples:
\begin{CJK}{UTF8}{min}
    $\mathbf{x}_1$: ``Today's weather is so nice.'',
    $\mathbf{x}_2$: ``初めまして。'',
    $\mathbf{x}_3$: ``In Japanese, we use こんにちは to say hello.''
\end{CJK} We have: $L(\mathbf{x}_1)=\{\textnormal{en}\}$, $L(\mathbf{x}_2)=\{\textnormal{ja}\}$, $L(\mathbf{x}_3)=\{\textnormal{en}, \textnormal{ja}\} $. The LPR for ja is: $\textnormal{LPR}(\text{ja}, \mathcal{D}) = \frac{\sum_{i}^{3}\mathbb{I} (\textnormal{ja} \in L(\mathbf{x}_i))}{3} = \frac{0 + 1 + 1}{3} \approx 0.67$.

Tab.~\ref{tab:language_control} demonstrates that unlearning via RAd with language-specific vector elicits the target language in the model's responses. For example, RAd w/ en$\to$fr vector makes the unlearned model generate more text in French (from $0.22$ to $0.51$) and less text in English (from $1.00$ to $0.83$).
Conversely, unlearning via RAb w/ language-specific vector can suppress the model's ability to generate text in the target language.
\begin{table}[H]
\caption{
Unlearning via RAd with language-specific vectors encourages the model to generate texts in the corresponding target languages. LPR of unlearned models on HellaSwag with four language-specific vectors. \textcolor{blue}{Increases} and \textcolor{purple}{drops} are marked (compared to the base model with the corresponding template).  
% \textcolor{red}{abc} $c=90, \alpha=5$
} \label{tab:language_control}
\centering 
\vspace{-0.5em}
\setlength{\tabcolsep}{10pt}
\resizebox{0.95\linewidth}{!}{
\begin{tabular}{lcccccccc}
\toprule
\multirow{2}{*}{\textbf{Method}} 
& \multirow{2}{*}{\textbf{Template}}
& \multicolumn{5}{c}{\textbf{Language Presence Rate on HellaSwag}} 
& \multirow{2}{*}{\textbf{MMLU} ($\uparrow$)} 
& \multirow{2}{*}{\textbf{WMDP} ($\downarrow$)} \\
\cmidrule(lr){3-7}
& 
& \textbf{en} 
& \textbf{fr}
& \textbf{es}
& \textbf{ja}
& \textbf{vi}
&  &  \\
\midrule
\multirow{5}{*}{\makecell[l]{Base model\\(Zephyr-7B)}}
& zero-shot & $1.00$ & $0.22$ & $0.22$ & $0.00$ & $0.00$ & \multirow{5}{*}{$58.4$} & \multirow{5}{*}{$54.4$} \\
& fr        & $0.76$ & $0.99$ & $0.42$ & $0.00$ & $0.00$ & & \\
& es        & $0.78$ & $0.31$ & $1.00$ & $0.00$ & $0.00$ & & \\
& ja        & $0.60$ & $0.11$ & $0.11$ & $0.70$ & $0.01$ & & \\
& vi        & $0.60$ & $0.12$ & $0.09$ & $0.00$ & $0.89$ & & \\
\midrule

RAd w/ random & zero-shot & $1.00$ & $0.23$ & $0.23$ & $0.00$ & $0.00$ & $55.4$ & $25.8$ \\
\midrule
RAd w/ en$\to$fr & zero-shot & $0.83$ & \cellcolor{gray!15}$\textbf{0.51}${\footnotesize\textcolor{blue}{$+0.29$}}  & $0.20$ & $0.00$ & $0.00$ & $52.5$ & $26.1$ \\
RAd w/ en$\to$es & zero-shot & $0.68$ & $0.19$ & \cellcolor{gray!15}$\textbf{0.67}${\footnotesize\textcolor{blue}{$+0.45$}} & $0.00$ & $0.00$ & $51.9$ & $26.2$ \\
RAd w/ en$\to$ja & zero-shot & $0.58$ & $0.12$ & $0.11$ & \cellcolor{gray!15}$\textbf{0.50}${\footnotesize\textcolor{blue}{$+0.50$}} & $0.00$ & $55.1$ & $25.1$ \\
RAd w/ en$\to$vi & zero-shot & $0.53$ & $0.11$ & $0.11$ & $0.00$ & \cellcolor{gray!15}$\textbf{0.62}${\footnotesize\textcolor{blue}{$+0.62$}} & $51.4$ & $25.5$ \\

\midrule
RAb w/ random & zero-shot & $1.00$ & $0.22$ & $0.22$ & $0.00$ & $0.00$ & $58.1$ & $47.9$ \\
\midrule
RAb w/ en$\to$fr & fr & $0.97$ & \cellcolor{gray!15}$\textbf{0.26}${\footnotesize\textcolor{purple}{$-0.73$}} & $0.19$ & $0.00$ & $0.00$ & $53.0$ & $33.3$ \\
RAb w/ en$\to$es & es & $0.99$ & $0.18$ & \cellcolor{gray!15}$\textbf{0.25}${\footnotesize\textcolor{purple}{$-0.75$}} & $0.00$ & $0.00$ & $51.5$ & $31.3$ \\
RAb w/ en$\to$ja & ja & $0.96$ & $0.18$ & $0.18$ & \cellcolor{gray!15}$\textbf{0.01}${\footnotesize\textcolor{purple}{$-0.69$}} & $0.00$ & $52.3$ & $31.0$ \\
RAb w/ en$\to$vi & vi & $0.99$ & $0.19$ & $0.19$ & $0.00$ & \cellcolor{gray!15}$\textbf{0.03}${\footnotesize\textcolor{purple}{$-0.86$}} & $50.0$ & $30.1$ \\
\bottomrule
\end{tabular}}
\vspace{-0.5em}
\end{table}

\section{Full experimental setup}
\label{appendix:experimental_setup}

\subsection{Benchmarks}
\label{appendix:benchmarks}
\textbf{WMDP~\textnormal{\citep{li2024wmdp}}}, stands for the Weapons of Mass Destruction Proxy, is an unlearning benchmark designed to measure and mitigate the malicious use of LLMs across Biology and Cyber domains. Each dataset consists of a forget-set, a retain-set, and a QA set. Both the forget and retain sets are collected from PubMed papers (Biology) or Github repositories (Cyber). For WMDP-Bio, the forget-set includes papers used to generate the QA set, while the retain set is sampled from general biology papers, excluding both forget-set papers and topics related to the QA set via keyword filtering. The WMDP-Bio QA set contains $1,273$ multiple-choice QAs. For WMDP-Cyber, forget and retain sets distinguished by different keyword sets used during data collection. The WMDP-Cyber QA set contains $1,987$ multiple-choice QAs. The WMDP corpus is publicly available at~\url{https://huggingface.co/datasets/cais/wmdp}.

\textbf{MUSE~\textnormal{\citep{shi2025muse}}} is an unlearning benchmark designed to evaluate six desirable properties of unlearned models. Two evaluation datasets are considered: MUSE-News, comprising BBC news articles, and MUSE-Books, comprising Harry Potter books. This benchmark is available at \url{https://huggingface.co/datasets/muse-bench}.

\textbf{Wikitext~\textnormal{\citep{merity2017pointer}}} comprises over $100$ million tokens extracted from articles on Wikipedia. Following~\citet{li2024wmdp, lucki2025an}, we use the \texttt{wikitext-2-raw-v1} test and train splits for unlearning (used for retaining) and knowledge recovery attacks, respectively. The dataset is available at \url{https://huggingface.co/datasets/Salesforce/wikitext}.

\textbf{MMLU~\textnormal{\citep{hendrycks2021measuring}}} is a benchmark comprising $15,908$ multiple-choice QAs for assessing models' world knowledge and problem-solving ability. The benchmark covers $57$ tasks spanning mathematics, history, computer science, law, and more. The data is available at \url{https://huggingface.co/datasets/cais/mmlu}.
% \vspace{-1.0em}
% \subsection{Side Benchmarks}
% \label{appendix:side_tasks}

\textbf{TruthfulQA~\textnormal{\citep{lin2022truthfulqa}}} consists of three tasks: TruthfulQA open-ended generation (answer generation), TruthfulQA MC1 (multiple-choice, single correct answer), and TruthfulQA MC2 (multiple-choice, multiple correct answers). The benchmark is available at~\url{https://github.com/sylinrl/TruthfulQA}.

\textbf{GLUE-SST2~\textnormal{\citep{wang2018glue}}} is a binary sentiment classification benchmark derived from movie reviews. The task requires models to predict whether a given sentence expresses positive or negative sentiment. This benchmark is available at~\url{https://huggingface.co/datasets/nyu-mll/glue}.

\newrevision{\textbf{StereoSet}~\textnormal{\citep{nadeem-etal-2021-stereoset}} is a large-scale dataset for measuring stereotypical biases in language models across four domains: gender, profession, race, and religion. We employ the intrasentence task to control stereotypical biases in unlearned models. The dataset is available at~\url{https://github.com/McGill-NLP/bias-bench/tree/main/data/stereoset}.} 

\textbf{AdvBench~\textnormal{\citep{zou2023universal}}} is a benchmark of harmful instructions designed to evaluate the safety and robustness of LLMs. It consists of instructions covering a wide range of harmful behaviors and is commonly used to assess the model’s refusal. The dataset is publicly available at \url{https://raw.githubusercontent.com/llm-attacks/llm-attacks/main/data/advbench/harmful_behaviors.csv}

\textbf{Alpaca~\textnormal{\citep{alpaca}}} is an instruction-following dataset consisting of diverse, human-readable instructions. It covers a broad range of tasks, including reasoning, summarization, and question answering, and is commonly used to assess general instruction-following behavior. The dataset is available at~\url{https://huggingface.co/datasets/tatsu-lab/alpaca}.

\textbf{ICL tasks~\textnormal{\citep{hendel-etal-2023-context}}} are a collection of simple ICL benchmarks designed to evaluate a model’s ability to acquire and apply task structure. We employ four tasks spanning two categories: linguistic and factual knowledge, including antonyms, present-to-past, and person-to-language and country-to-capital. The dataset is available at \url{https://github.com/roeehendel/icl_task_vectors/tree/master}.

\textbf{Reasoning tasks.} We employ GSM8K~\citep{cobbe2021training} and GSM-Plus~\citep{li2024gsm}. GSM8K consists of diverse grade school math word problems to assess multi-step mathematical reasoning. Extended from GSM8K, GSM-Plus consists of adversarial problems, which can assess the robustness of models to various mathematical perturbations. These datasets are available at \url{https://huggingface.co/datasets/openai/gsm8k} and \url{https://huggingface.co/datasets/qintongli/GSM-Plus}, respectively.

\textbf{HellaSwag~\textnormal{\citep{zellers2019hellaswag}}} is a dataset for commonsense natural language completion and inference where models are asked to select the most relevant follow-up to a context from $4$ choices. We adopt this dataset to study language control in unlearned models. The dataset is available at \url{https://huggingface.co/datasets/Rowan/hellaswag}.

\subsection{Implementation details}
\label{appendix:implementation_details}
\begin{wraptable}{r}{0.42\textwidth}
\vspace{-1em}
\caption{Hyperparameters for side tasks.}
\label{tab:hparams_side_tasks}
\centering
\resizebox{\linewidth}{!}{
\setlength{\tabcolsep}{6pt}
\begin{tabular}{lllccc}
\toprule
\multirow{2}{*}{\textbf{Methods}} & \multirow{2}{*}{\textbf{Tasks}} & \multirow{2}{*}{\textbf{Models}} & \multicolumn{2}{c}{\textbf{Hypers.}} & \multirow{2}{*}{\textbf{References}} \\
\cmidrule(lr){4-5}
& & & $\alpha_r$ & $c$ & \\
\midrule
\multirow{26}{*}{RAd} 
& \multirow{2}{*}{Truthfulness} & Zephyr-7B & $1200.0$ & $14.0$ & Tab.~\ref{tab:truthfulqa} \\ 
& & Mistral-7B & $1200.0$ & $19.0$ & Tab.~\ref{tab:truthfulqa} \\
% & & \revision{Qwen2.5-7B} & \revision{$1.0$} & \revision{$90.0$} & \revision{Tab.~\ref{tab:truthfulqa_more_models}} \\
% & & \revision{Phi4-mini} & \revision{$1.0$} & \revision{$36.0$} & \revision{Tab.~\ref{tab:truthfulqa_more_models}} \\

\cmidrule(lr){2-6}
& Sentiment \\
& neg$\to$pos & Zephyr-7B & $1200.0$ & $23.0$ & Tab.~\ref{tab:sst2_negative} \\
& neg$\to$pos & Mistral-7B & $1200.0$ & $17.0$ & Tab.~\ref{tab:sst2_negative} \\
& pos$\to$neg & Zephyr-7B & $1200.0$ & $16.0$ & Tab.~\ref{tab:sst2_positive} \\ 
& pos$\to$neg & Mistral-7B & $1200.0$ & $16.0$ & Tab.~\ref{tab:sst2_positive} \\ 

\cmidrule(lr){2-6}
& \multirow{2}{*}{Refusal} & Zephyr-7B & $1200.0$ & $18.0$ & Tab.~\ref{tab:alpaca} \\ 
& & Llama-3-8B & $1200.0$ & $24.0$ & Tab.~\ref{tab:alpaca} \\

\cmidrule(lr){2-6}
& \newrevision{Bias} \\
& \newrevision{ster$\to$anti-ster} & \newrevision{Zephyr-7B} & \newrevision{$1200.0$} & \newrevision{$16.0$} & \newrevision{Tab.~\ref{tab:stereotype_intrasentence_mc}} \\
& \newrevision{ster$\to$anti-ster} & \newrevision{Mistral-7B} & \newrevision{$1200.0$} & \newrevision{$16.0$} & \newrevision{Tab.~\ref{tab:stereotype_intrasentence_mc}} \\
& \newrevision{anti-ster$\to$ster} & \newrevision{Zephyr-7B} & \newrevision{$1200.0$} & \newrevision{$18.0$} & \newrevision{Tab.~\ref{tab:stereotype_intrasentence_mc}} \\
& \newrevision{anti-ster$\to$ster} & \newrevision{Mistral-7B} & \newrevision{$1200.0$} & \newrevision{$17.0$} & \newrevision{Tab.~\ref{tab:stereotype_intrasentence_mc}} \\

\cmidrule(lr){2-6}
& Language \\
& en$\to$fr & Zephyr-7B & $1200.0$ & $17.0$ & Tab.~\ref{tab:language_control} \\
& en$\to$es & Zephyr-7B & $1200.0$ & $17.0$ & Tab.~\ref{tab:language_control} \\
& en$\to$ja & Zephyr-7B & $1200.0$ & $17.0$ & Tab.~\ref{tab:language_control} \\
& en$\to$vi & Zephyr-7B & $1200.0$ & $17.0$ & Tab.~\ref{tab:language_control} \\

\cmidrule(lr){2-6}
& Reasoning \\
& CoT & Zephyr-7B & $1200.0$ & $20.0$ & Tab.~\ref{tab:reasoning_tasks} \\ 
% & CoT & Mistral-7B & $1200.0$ & $20.0$ & Tab.~\ref{tab:reasoning_tasks} \\
& \revision{CoT} & \revision{Qwen2.5-7B}  & \revision{$1.0$} & \revision{$50.0$} & \revision{Tab.~\ref{tab:reasoning_tasks}} \\
& \revision{5-shot} & \revision{Zephyr-7B}  & \revision{$1200.0$} & \revision{$14.0$} & \revision{Tab.~\ref{tab:reasoning_tasks}} \\
& \revision{5-shot} & \revision{Qwen2.5-7B}  & \revision{$1.0$} & \revision{$60.0$} & \revision{Tab.~\ref{tab:reasoning_tasks}} \\

\cmidrule(lr){2-6}
& \multirow{2}{*}{Antonyms} & Zephyr-7B & $1200.0$ & $18.0$ & Tab.~\ref{tab:icl} \\ 
& & Mistral-7B & $1200.0$ & $19.0$ & Tab.~\ref{tab:icl} \\
\cmidrule(lr){2-6}
& \multirow{2}{*}{pres$\to$past} & Zephyr-7B & $1200.0$ & $18.0$ & Tab.~\ref{tab:icl} \\ 
& & Mistral-7B & $1200.0$ & $19.0$ & Tab.~\ref{tab:icl} \\
\cmidrule(lr){2-6}
& \multirow{2}{*}{ctry$\to$cap} & Zephyr-7B & $1200.0$ & $18.0$ & Tab.~\ref{tab:icl} \\ 
& & Mistral-7B & $1200.0$ & $18.0$ & Tab.~\ref{tab:icl} \\
\cmidrule(lr){2-6}
& \multirow{2}{*}{pers$\to$lang} & Zephyr-7B & $1200.0$ & $19.0$ & Tab.~\ref{tab:icl} \\
& & Mistral-7B & $1200.0$ & $20.0$ & Tab.~\ref{tab:icl} \\

\midrule
\multirow{15}{*}{RAb} 
& \multirow{2}{*}{Truthfulness} & Zephyr-7B & $20.0$ & $50.0$ & Tab.~\ref{tab:truthfulqa} \\ 
& & Mistral-7B & $20.0$ & $60.0$ & Tab.~\ref{tab:truthfulqa} \\

\cmidrule(lr){2-6}
& Sentiment \\
& pos$\to$neg & Zephyr-7B & $20.0$ & $120.0$ & Tab.~\ref{tab:sst2_negative} \\ 
& pos$\to$neg & Mistral-7B & $20.0$ & $110.0$ & Tab.~\ref{tab:sst2_negative} \\
& neg$\to$pos & Zephyr-7B & $20.0$ & $120.0$ & Tab.~\ref{tab:sst2_positive} \\
& neg$\to$pos & Mistral-7B & $20.0$ & $110.0$ & Tab.~\ref{tab:sst2_positive} \\

\cmidrule(lr){2-6}
& \multirow{2}{*}{Refusal} & Zephyr-7B & $20.0$ & $40.0$ & Tab.~\ref{tab:advbench} \\ 
& & Llama-3-8B & $20.0$ & $60.0$ & Tab.~\ref{tab:advbench} \\

\cmidrule(lr){2-6}
\cmidrule(lr){2-6}
& \newrevision{Bias} \\
& \newrevision{anti-ster$\to$ster} & \newrevision{Zephyr-7B} & \newrevision{$20.0$} & \newrevision{$90.0$} & \newrevision{Tab.~\ref{tab:stereotype_intrasentence_mc}} \\
& \newrevision{anti-ster$\to$ster} & \newrevision{Mistral-7B} & \newrevision{$20.0$} & \newrevision{$90.0$} & \newrevision{Tab.~\ref{tab:stereotype_intrasentence_mc}} \\
& \newrevision{ster$\to$anti-ster} & \newrevision{Zephyr-7B} & \newrevision{$20.0$} & \newrevision{$70.0$} & \newrevision{Tab.~\ref{tab:stereotype_intrasentence_mc}} \\
& \newrevision{ster$\to$anti-ster} & \newrevision{Mistral-7B} & \newrevision{$20.0$} & \newrevision{$80.0$} & \newrevision{Tab.~\ref{tab:stereotype_intrasentence_mc}} \\

\cmidrule(lr){2-6}
& Language \\
& en$\to$fr & Zephyr-7B & $20.0$ & $68.0$ & Tab.~\ref{tab:language_control} \\
& en$\to$es & Zephyr-7B & $20.0$ & $68.0$ & Tab.~\ref{tab:language_control} \\
& en$\to$ja & Zephyr-7B & $20.0$ & $68.0$ & Tab.~\ref{tab:language_control} \\
& en$\to$vi & Zephyr-7B & $20.0$ & $68.0$ & Tab.~\ref{tab:language_control} \\

\cmidrule(lr){2-6}
& Reasoning \\
& CoT & Zephyr-7B & $20.0$ & $60.0$ & Tab.~\ref{tab:reasoning_tasks} \\ 
% & CoT & Mistral-7B & $20.0$ & $60.0$ & Tab.~\ref{tab:reasoning_tasks} \\
& \revision{CoT} & \revision{Qwen2.5-7B}  & \revision{$1.0$} & \revision{$70.0$} & \revision{Tab.~\ref{tab:reasoning_tasks}} \\
& \revision{5-shot} & \revision{Zephyr-7B}  & \revision{$20.0$} & \revision{$14.0$} & \revision{Tab.~\ref{tab:reasoning_tasks}} \\
& \revision{5-shot} & \revision{Qwen2.5-7B}  & \revision{$10.0$} & \revision{$11.0$} & \revision{Tab.~\ref{tab:reasoning_tasks}} \\
\midrule
\multirow{2}{*}{\newrevision{RMU}} 
& \multirow{2}{*}{\newrevision{Truthfulness}} & \newrevision{Zephyr-7B} & \newrevision{$1200.0$} & \newrevision{$18.0$} & \newrevision{Tab.~\ref{tab:RMUvsRAd}} \\ 
& & \newrevision{Mistral-7B} & \newrevision{$1200.0$} & \newrevision{$16.0$} & \newrevision{Tab.~\ref{tab:RMUvsRAd}} \\
\bottomrule
\end{tabular}}
\vspace{-5.4em}
\end{wraptable}

For RAd and RAb unlearning, we employ AdamW~\citep{loshchilov2018decoupled} to fine-tune models for $T=500$ update steps with a learning rate of $5e-5$, batch size of $4$, and weight decay of $0.02$. WMDP-Bio and WMDP-Cyber are learned jointly.
The maximum sequence length is set to $500$ for both WMDP-Bio and WMDP-Cyber.  
We fix the forget weights at $\alpha_f^{\text{bio}}=\alpha_f^{\text{cyber}} = 1.0$ and perform a grid search over the retain weights: $\alpha_r^{\text{bio}}=\alpha_r^{\text{cyber}}$ and the coefficient $c$. Hyperparameters are summarized in Tab.~\ref{tab:hparams_side_tasks}. The unlearn layer is set to $l=7$ for all methods. Following prior work~\cite{li2024wmdp}, for memory efficiency, we update the MLP down-projection matrices in three layers $\{l,l-1,l-2\}$ of the model. In this paper, the representation is the mean of all tokens' activations in a sample, extracted from the MLP's output at layer $l=7$. Evaluation is primarily conducted using the \texttt{lm-eval-harness (version 0.4.5)} framework~\citep{gao2021framework}.

For other unlearning baselines, we adopt the default hyperparameters used in prior work~\citep{yuan2025a, fan2025simplicity}. Specifically, we set $\beta = 0.1$ for DPO, NPO, SimNPO, and $\gamma = 0$ for both SimNPO+KL and SimNPO+MSE. The forget weights are fixed at $\alpha_f^{\text{bio}} =  \alpha_f^{\text{cyber}} = 1.0$. For the retain weights, we perform a grid search over combinations of $(\alpha_r^{\text{bio}}, \alpha_r^{\text{cyber}})$ and select the best combinations measured on MMLU and WMDP. For Zephyr-7B  we select $(60, 60)$ for GA+KL, $(12, 12)$ for GA+MSE, $(30, 50)$ for DPO+KL, $(5, 20)$ for DPO+MSE, $(50, 50)$ for NPO+KL, $(5, 20)$ for NPO+MSE, $(20, 50)$ for SimNPO+KL, and $(10, 5)$ for SimNPO+MSE. For Mistral-7B, we select $(50, 40)$ for GA+KL, $(20, 30)$ for GA+MSE, $(10, 5)$ for DPO+KL, $(20, 20)$ for DPO+MSE, $(40, 30)$ for NPO+KL, $(7, 7)$ for NPO+MSE, $(30, 30)$ for SimNPO+KL, and $(10, 5)$ for SimNPO+MSE. For RMU, the coefficient $c$ is set to $6.5$.
% For TruthfulQA 
% \begin{itemize}
%     \item Zephyr RAd 1200, 14, RAb 20, 50,
%     \item Mistral RAd 1200, 19, RAb 20, 60
% \end{itemize}
% For Sentiment positive $to$ Negative:
% \begin{itemize}
%     \item Zephyr RAd 1200, 16, Rab 20, 120
%     \item Mistral Rad positive 17  RAb negative 110
% \end{itemize}
% For sentiment negative $to$ positive
% \begin{itemize}
%     \item Zephyr RAd 1200, 23, RAb 20 120.
%     \item Mistral Rad positive 17  RAb negative 110
% \end{itemize}

% For Refusal
% Zephyr RAd 1200 18, llama-3 RAd 1200, 24
% Zephyr RAb 20, 40, Llama-3 RAb 20, 60
% Zephyr: antonym (1200, 18), pres $\to$ past 1200, 16, country $\to$ capital (1200, 18), and person $\to$ language (1200, 19).
% Mistral: Antonym 1200, 19, present2past: 1200, 19, country2capital: 1200, 18, person2language: 1200, 20
\textbf{Reproducibility.} Our implementation is available at \url{https://github.com/huutiendang/llmu-analysis}. 
% \newpage

\section{Prompt templates}
\label{appendix:prompt_template}
\subsection{Refusal substrings}
\label{appendix:refusal_substrings}

We evaluate the refusal score by checking whether the generated answers contain any predefined refusal substrings. Matching is both order-agnostic and case-insensitive, \textit{i.e.,} a refusal substring may appear anywhere in the answer. The set of refusal substrings is:

\emph{``As an AI language model'', ``As a language model'', ``As an assistant'', ``I strongly advise against'', ``I'm sorry'', ``I am sorry'', ``I am not capable of'', ``I’m not able to'', ``I’m unable to'', ``I apologize'', ``I do not''}. 

% \subsection{LLM-as-a-judge prompt}
% \label{appendix:llm-as-a-judge-prompt}
\subsection{Prompt templates for tasks}
\label{appendix:template_for_tasks}
Prompt templates used for task evaluation are described in Tab.~\ref{tab:templates}.

\begin{table*}[t]
\caption{Prompt templates used for evaluation.}
\begin{tabular}{@{}l l p{0.55\linewidth}@{}}
\toprule
\textbf{Setting} & \textbf{Task} & \textbf{Template} \\
\midrule
\multirow{5}{*}{\makecell[l]{ICL\\(Sec.~\ref{sec:icl})}}
 & Zero-shot          & ``\texttt{Text: \{input\}\textbackslash nLabel:}'' \\
 & Antonym            & ``\texttt{Text: \{input\}\textbackslash nAntonym:}'' \\
 & Present to past    & ``\texttt{Text: \{input\}\textbackslash nPast:}'' \\
 & Person to language & ``\texttt{Text: \{input\}\textbackslash nLanguage:}'' \\
 & Country to capital & ``\texttt{Text: \{input\}\textbackslash nCapital:}'' \\
\midrule
\multirow{3}{*}{\makecell[l]{Reasoning\\(Sec.~\ref{sec:icl})}}
 & Zero-shot & ``\texttt{Question: \{question\}\textbackslash nAnswer:}'' \\
 & CoT       & ``\texttt{Question: \{question\}\textbackslash nAnswer: Let's think step by step.}'' \\
 & 5-shot    & ``\texttt{[5 samples]\textbackslash nQuestion: \{question\}\textbackslash nAnswer:}'' \\
\midrule
Sentiment (Sec.~\ref{sec:sentiment}) & SST2 & ``\texttt{Classify the following sentence. Return only one word: positive or negative\textbackslash nSentence: \{input\}\textbackslash nAnswer:}'' \\
\midrule
Refusal (Sec.~\ref{sec:refusal}) & Chat & \texttt{[} \texttt{\{"role": "system", "content": "Be a helpful and safe AI assistant. Avoid harmful or biased content. Explain unclear questions. Admit when you don’t know something."\},}\newline \hspace*{1em}\texttt{\{"role": "user", "content": \{sentence\}\}} \texttt{]} \\
\midrule
Stereotypical bias (Appx.~\ref{appendix:stereotype_bias}) & MCQ & ``\texttt{Choose the appropriate word.\textbackslash n"\{context\}"\textbackslash nA. \{choices[0]\}\textbackslash nB. \{choices[1]\}\textbackslash nC. \{choices[2]\}\textbackslash nAnswer:}'' \\
\bottomrule
\end{tabular}\label{tab:templates}
\end{table*}

\section{Effects of probes}
\label{appendix:effect_of_probe}
Concept vectors are obtained by training logistic regression probes. One might be concerned about the reliability of the probe, such as (1) \textbf{how many samples are needed before the induced behaviors become reliable}, and (2) \textbf{comparing logistic regression to alternative probes (\textit{e.g.,} Ridge Regression, K-Means clustering) or simpler contrastive methods (\textit{e.g.,} difference-in-means)}. We conduct experiments across multiple probes and analyze how the sample size affects the probe's reliability. For all experiments in this section, we use RAd and RAb with the truth direction and evaluate on TruthfulQA open-ended and multiple-choice tasks. For the sample size experiments, we vary the amount of data used to extract the probe by sampling from $3\% $ to $90\%$ of $\mathcal{D}_{\text{dev}}$. Tab.~\ref{tab:probe_methods} and Tab.~\ref{tab:probe_samples} show two findings. First, probe quality matters for inducing side behaviors: higher validation accuracy probes \textit{i.e.,} Ridge Regression with approximately $72.2\%$, induce stronger behavioral shifts than K-Means (with $56.0\%$). Second, the probe converges quickly with increasing sample size, indicating that the probe does not require large labeled datasets to elicit controllable side behaviors.

\begin{table*}[h]
\caption{Performance comparison of different probes for RAd and RAb. Val. Acc. denotes the validation accuracy of the probe.}
\centering
% \vspace{-0.5em}
\resizebox{0.925\linewidth}{!}{
\setlength{\tabcolsep}{4pt}
\begin{tabular}{llccccccccc}
\toprule
\multirow{2}{*}{\textbf{Model}} 
& \multirow{2}{*}{\textbf{Probe}} 
& \multirow{2}{*}{\textbf{Val. Acc.}}
& \multicolumn{4}{c}{\textbf{TruthfulQA open-ended}} 
& \multicolumn{2}{c}{\textbf{TruthfulQA MC}} 
& \multicolumn{2}{c}{\textbf{Unlearning tasks}} \\
\cmidrule(lr){4-7} \cmidrule(lr){8-9} \cmidrule(lr){10-11}
& & & BLEU & ROUGE-1 & ROUGE-2 & ROUGE-L & MC1 & MC2 & MMLU & WMDP \\
\midrule
Base & $-$ & $-$ & $47.0$ & $45.5$ & $37.9$ & $42.6$ & $39.0$ & $55.0$ & $58.4$ & $54.4$ \\
\midrule
\multirow{3}{*}{\textbf{RAd}} 
& Diff-in-means       & $-$ & $49.0${\footnotesize \textcolor{blue}{$+2.0$}} & $50.7${\footnotesize \textcolor{blue}{$+5.2$}} & $41.2${\footnotesize \textcolor{blue}{$+3.3$}} & $50.0${\footnotesize \textcolor{blue}{$+7.4$}} & $43.9${\footnotesize \textcolor{blue}{$+4.9$}} & $60.5${\footnotesize \textcolor{blue}{$+5.5$}} & $54.6$ & $28.8$ \\
& Ridge Regression    & $72.2$ & $51.0${\footnotesize \textcolor{blue}{$+4.0$}} & $54.4${\footnotesize \textcolor{blue}{$+8.9$}} & $45.3${\footnotesize \textcolor{blue}{$+7.4$}} & $52.0${\footnotesize \textcolor{blue}{$+9.4$}} & $41.2${\footnotesize \textcolor{blue}{$+2.2$}} & $57.7${\footnotesize \textcolor{blue}{$+2.7$}} & $54.8$ & $25.2$ \\
& K-Means             & $56.0$ & $50.0${\footnotesize \textcolor{blue}{$+3.0$}} & $48.8${\footnotesize \textcolor{blue}{$+3.3$}} & $39.2${\footnotesize \textcolor{blue}{$+1.3$}} & $46.3${\footnotesize \textcolor{blue}{$+3.7$}} & $38.3${\footnotesize \textcolor{purple}{$-0.7$}} & $55.7${\footnotesize \textcolor{blue}{$+0.7$}} & $55.7$ & $26.2$ \\
& Logistic Regression & $71.3$ & $47.7${\footnotesize \textcolor{blue}{$+0.7$}} & $53.9${\footnotesize \textcolor{blue}{$+8.4$}} & $40.9${\footnotesize \textcolor{blue}{+3.0}} & $51.9${\footnotesize \textcolor{blue}{$+9.3$}} & $44.9${\footnotesize \textcolor{blue}{+5.9}} & $62.3${\footnotesize \textcolor{blue}{$+7.3$}} & $54.9$ & $28.2$ \\
\midrule
\multirow{3}{*}{\textbf{RAb}} 
& Diff-in-means       & $-$ & $45.1${\footnotesize \textcolor{purple}{$-1.9$}} & $45.3${\footnotesize \textcolor{purple}{$-0.2$}} & $35.5${\footnotesize \textcolor{purple}{$-2.4$}} & $43.9${\footnotesize \textcolor{blue}{$+1.3$}} & $31.1${\footnotesize \textcolor{purple}{$-7.9$}} & $47.2${\footnotesize \textcolor{purple}{$-7.8$}} & $47.1$ & $27.7$ \\
& Ridge Regression    & $72.2$ & $45.3${\footnotesize \textcolor{purple}{$-1.7$}} & $45.3${\footnotesize \textcolor{purple}{$-0.2$}} & $36.3${\footnotesize \textcolor{purple}{$-1.6$}} & $45.6${\footnotesize \textcolor{blue}{$+3.0$}} & $31.7${\footnotesize \textcolor{purple}{$-7.3$}} & $47.9${\footnotesize \textcolor{purple}{$-7.1$}} & $54.0$ & $41.2$ \\
& K-Means             & $56.0$ & $41.9${\footnotesize \textcolor{purple}{$-5.1$}} & $41.2${\footnotesize \textcolor{purple}{$-4.3$}} & $24.5${\footnotesize \textcolor{purple}{$-13.4$}} & $38.7${\footnotesize \textcolor{purple}{$-3.9$}} & $27.8${\footnotesize \textcolor{purple}{$-11.2$}} & $43.9${\footnotesize \textcolor{purple}{$-11.1$}} & $45.3$ & $29.0$ \\
& Logistic Regression & $71.3$ & $41.1${\footnotesize \textcolor{purple}{$-5.9$}} & $41.9${\footnotesize \textcolor{purple}{$-3.6$}} & $31.6${\footnotesize \textcolor{purple}{$-6.3$}} & $40.9${\footnotesize \textcolor{purple}{$-1.7$}} & $26.1${\footnotesize \textcolor{purple}{$-12.9$}} & $40.0${\footnotesize \textcolor{purple}{$-15.0$}} & $52.0$ & $32.9$ \\
\bottomrule
\end{tabular}}
\label{tab:probe_methods}
% \vspace{-1.0em}
\end{table*}

\begin{table*}[h]
\caption{Performance comparison of RAd w/ truth and RAb w/ truth across different sample sizes.}
\centering
% \vspace{-0.5em}
\resizebox{0.92\linewidth}{!}{
\setlength{\tabcolsep}{4pt}
\begin{tabular}{llccccccccc}
\toprule
\multirow{2}{*}{\textbf{Model}} 
& \multirow{2}{*}{\textbf{Size}} 
& \multirow{2}{*}{\textbf{Val. Acc.}}
& \multicolumn{4}{c}{\textbf{TruthfulQA open-ended}} 
& \multicolumn{2}{c}{\textbf{TruthfulQA MC}} 
& \multicolumn{2}{c}{\textbf{Unlearning tasks}} \\
\cmidrule(lr){4-7} \cmidrule(lr){8-9} \cmidrule(lr){10-11}
& & & BLEU & ROUGE-1 & ROUGE-2 & ROUGE-L & MC1 & MC2 & MMLU & WMDP \\
\midrule
Base & $-$ & $-$ & $47.0$ & $45.5$ & $37.9$ & $42.6$ & $39.0$ & $55.0$ & $58.4$ & $54.4$ \\
\midrule
\multirow{5}{*}{\textbf{RAd}} 
& $3\%$   & $56.3$ & $55.9${\footnotesize \textcolor{blue}{$+8.9$}} & $56.1${\footnotesize \textcolor{blue}{$+10.6$}} & $48.3${\footnotesize \textcolor{blue}{$+10.4$}} & $55.1${\footnotesize \textcolor{blue}{$+12.5$}} & $41.7${\footnotesize \textcolor{blue}{$+2.7$}} & $58.7${\footnotesize \textcolor{blue}{$+3.7$}} & $55.3$ & $27.6$ \\
& $15\%$  & $69.8$ & $51.0${\footnotesize \textcolor{blue}{$+4.0$}} & $54.4${\footnotesize \textcolor{blue}{$+8.9$}} & $45.1${\footnotesize \textcolor{blue}{$+7.2$}} & $54.9${\footnotesize \textcolor{blue}{$+12.3$}} & $43.7${\footnotesize \textcolor{blue}{$+4.7$}} & $60.7${\footnotesize \textcolor{blue}{$+5.7$}} & $54.3$ & $27.1$ \\
& $30\%$  & $70.0$ & $49.5${\footnotesize \textcolor{blue}{$+2.5$}} & $52.5${\footnotesize \textcolor{blue}{$+7.0$}} & $39.7${\footnotesize \textcolor{blue}{$+1.8$}} & $51.2${\footnotesize \textcolor{blue}{$+8.6$}} & $46.0${\footnotesize \textcolor{blue}{$+7.0$}} & $62.9${\footnotesize \textcolor{blue}{$+7.9$}} & $54.1$ & $27.7$ \\
& $60\%$  & $72.0$ & $41.7${\footnotesize \textcolor{purple}{$-5.3$}} & $48.8${\footnotesize \textcolor{blue}{$+3.3$}} & $37.0${\footnotesize \textcolor{purple}{$-0.9$}} & $48.3${\footnotesize \textcolor{blue}{$+5.7$}} & $40.9${\footnotesize \textcolor{blue}{$+1.9$}} & $59.5${\footnotesize \textcolor{blue}{$+4.5$}} & $54.0$ & $28.2$ \\
& $90\%$  & $71.3$ & $54.9${\footnotesize \textcolor{blue}{$+7.9$}} & $58.1${\footnotesize \textcolor{blue}{$+12.6$}} & $48.3${\footnotesize \textcolor{blue}{$+10.4$}} & $57.4${\footnotesize \textcolor{blue}{$+14.8$}} & $44.1${\footnotesize \textcolor{blue}{$+5.1$}} & $61.1${\footnotesize \textcolor{blue}{$+6.1$}} & $54.8$ & $27.4$ \\
\midrule
\multirow{5}{*}{\textbf{RAb}} 
& $3\%$   & $56.3$ & $42.6${\footnotesize \textcolor{purple}{$-4.4$}} & $44.1${\footnotesize \textcolor{purple}{$-1.4$}} & $35.5${\footnotesize \textcolor{purple}{$-2.4$}} & $40.9${\footnotesize \textcolor{purple}{$-1.7$}} & $36.5${\footnotesize \textcolor{purple}{$-2.5$}} & $51.5${\footnotesize \textcolor{purple}{$-3.5$}} & $52.3$ & $37.3$ \\
& $15\%$  & $69.8$ & $44.4${\footnotesize \textcolor{purple}{$-2.6$}} & $40.9${\footnotesize \textcolor{purple}{$-4.6$}} & $31.9${\footnotesize \textcolor{purple}{$-6.0$}} & $41.7${\footnotesize \textcolor{purple}{$-0.9$}} & $25.1${\footnotesize \textcolor{purple}{$-13.9$}} & $39.5${\footnotesize \textcolor{purple}{$-15.5$}} & $51.2$ & $41.1$ \\
& $30\%$  & $70.0$ & $44.4${\footnotesize \textcolor{purple}{$-2.6$}} & $43.6${\footnotesize \textcolor{purple}{$-1.9$}} & $33.1${\footnotesize \textcolor{purple}{$-4.8$}} & $43.6${\footnotesize \textcolor{blue}{$+1.0$}} & $27.1${\footnotesize \textcolor{purple}{$-11.9$}} & $40.5${\footnotesize \textcolor{purple}{$-14.5$}} & $51.8$ & $41.1$ \\
& $60\%$  & $72.0$ & $39.2${\footnotesize \textcolor{purple}{$-7.8$}} & $39.5${\footnotesize \textcolor{purple}{$-6.0$}} & $32.8${\footnotesize \textcolor{purple}{$-5.1$}} & $39.5${\footnotesize \textcolor{purple}{$-3.1$}} & $27.9${\footnotesize \textcolor{purple}{$-11.1$}} & $42.5${\footnotesize \textcolor{purple}{$-12.5$}} & $52.8$ & $36.5$ \\
& $90\%$  & $71.3$ & $40.2${\footnotesize \textcolor{purple}{$-6.8$}} & $40.0${\footnotesize \textcolor{purple}{$-5.5$}} & $30.1${\footnotesize \textcolor{purple}{$-7.8$}} & $40.0${\footnotesize \textcolor{purple}{$-2.6$}} & $26.8${\footnotesize \textcolor{purple}{$-12.2$}} & $41.1${\footnotesize \textcolor{purple}{$-13.9$}} & $52.9$ & $36.6$ \\
\bottomrule
\end{tabular}}
% \vspace{-0.5em}
\label{tab:probe_samples}
\end{table*}

\section{Is the phenomenon specific to RAd and RAb?}
\label{appendix:rmu_concept}
\textbf{RMU with a concept vector.} If Hypothesis~\ref{hypothesis1} concerns RM in general, the observed side effect should hold for RMU~\cite{li2024wmdp}. RMU collapses the forget-samples' representations toward a random vector, which can naturally be replaced by a concept vector. We therefore compare RAd and RMU under the same setting. Specifically, for both methods, we use the truth concept vector as the target vector and evaluate the resulting models on both the TruthfulQA benchmarks and the unlearning benchmarks (MMLU and WMDP). We conduct the comparison on both Zephyr-7B and Mistral-7B. Tab.~\ref{tab:RMUvsRAd} shows that both methods can incorporate the truth concept vector to enhance truthfulness while performing unlearning. However, RAd consistently achieves larger improvements on TruthfulQA than RMU, while maintaining comparable unlearning performance.

\begin{table*}[h]
% \vspace{-0.75em}
\caption{\label{tab:RMUvsRAd} Performance of RAd and RMU models on WMDP, MMLU, and TruthfulQA benchmarks. \textcolor{blue}{Increases} and \textcolor{purple}{drops} are marked (compared to the base models)}.
\centering
% \vspace{-0.7em}
\setlength{\tabcolsep}{4pt}
\resizebox{0.92\textwidth}{!}{
\begin{tabular}{llcccccccc}
\toprule
\multirow{2}{*}{\textbf{}} & \multirow{2}{*}{\textbf{Models}} 
& \multicolumn{4}{c}{\textbf{TruthfulQA open-ended}} 
& \multicolumn{2}{c}{\textbf{TruthfulQA multiple-choice}} 
& \multicolumn{2}{c}{\textbf{Unlearning tasks}} \\
\cmidrule(lr){3-6} \cmidrule(lr){7-8} \cmidrule(lr){9-10}
& & BLEU & ROUGE-1  & ROUGE-2  & ROUGE-L  & MC1 & MC2 & MMLU & WMDP \\
\midrule
\multirow{5}{*}{\textbf{Zephyr-7B}} 
& Base model & $47.0$ & $45.5$ & $37.9$ & $42.6$ & $39.0$ & $55.0$ & $58.4$ & $54.4$\\
\cmidrule{2-10}
& RMU  & $47.5${\footnotesize\textcolor{blue}{$+0.5$}} & $49.0${\footnotesize\textcolor{blue}{$+3.5$}} & $41.1${\footnotesize\textcolor{blue}{$+3.2$}} & $43.7${\footnotesize\textcolor{blue}{$+1.1$}} & $37.4${\footnotesize\textcolor{purple}{$-1.6$}} & $55.1${\footnotesize\textcolor{blue}{$+0.1$}} & $55.3$ & $26.5$\\
& RAd w/ random  & $49.5${\footnotesize\textcolor{blue}{$+2.5$}} & $47.7${\footnotesize\textcolor{blue}{$+2.2$}} & $39.5${\footnotesize\textcolor{blue}{$+1.6$}} & $44.3${\footnotesize\textcolor{blue}{$+1.7$}} & $38.4${\footnotesize\textcolor{purple}{$-0.6$}} & $55.9${\footnotesize\textcolor{blue}{$+0.9$}} & $55.9$ & $25.6$\\\cmidrule{2-10}
&$\cellcolor{gray!15}$RMU w/ truth  & $\cellcolor{gray!15}$$\textbf{50.2}${\footnotesize\textcolor{blue}{$+3.2$}} & $\cellcolor{gray!15}$$\textbf{49.7}${\footnotesize\textcolor{blue}{$+4.2$}} & $\cellcolor{gray!15}$$\textbf{40.6}${\footnotesize\textcolor{blue}{$+2.7$}} & $\cellcolor{gray!15}$$\textbf{46.5}${\footnotesize\textcolor{blue}{$+3.9$}} & $\cellcolor{gray!15}$$\textbf{39.5}${\footnotesize\textcolor{blue}{$+0.5$}} & $\cellcolor{gray!15}$$\textbf{56.1}${\footnotesize\textcolor{blue}{$+1.1$}} & $\cellcolor{gray!15}$$53.8$ & $\cellcolor{gray!15}$$24.9$\\
& $\cellcolor{gray!15}$RAd w/ truth  & $\cellcolor{gray!15}$$\textbf{47.7}${\footnotesize\textcolor{blue}{$+0.7$}} & $\cellcolor{gray!15}$$\textbf{53.9}${\footnotesize\textcolor{blue}{$+8.4$}} & $\cellcolor{gray!15}$$\textbf{40.9}${\footnotesize\textcolor{blue}{$+3.0$}} & $\cellcolor{gray!15}$$\textbf{51.9}${\footnotesize\textcolor{blue}{$+9.3$}} & $\cellcolor{gray!15}$$\textbf{44.9}${\footnotesize\textcolor{blue}{$+5.9$}} & $\cellcolor{gray!15}$$\textbf{62.3}${\footnotesize\textcolor{blue}{$+7.3$}} & $\cellcolor{gray!15}$$54.9$ & $\cellcolor{gray!15}$$28.2$\\ 
\midrule

\multirow{5}{*}{\textbf{Mistral-7B}} 
& Base model & $40.6$ & $38.7$ & $35.5$ & $40.6$ & $28.2$ & $42.6$ & $59.6$ & $55.7$\\
\cmidrule{2-10}
& RMU  & $40.6${\footnotesize\textcolor{gray}{$+0.0$}} & $41.4${\footnotesize\textcolor{blue}{$+2.7$}} & $38.4${\footnotesize\textcolor{blue}{$+2.9$}} & $41.4${\footnotesize\textcolor{blue}{$+0.8$}} & $29.8${\footnotesize\textcolor{blue}{$+1.6$}} & $43.7${\footnotesize\textcolor{blue}{$+1.1$}} & $57.4$ & $26.6$\\
& RAd w/ random  & $40.4${\footnotesize\textcolor{purple}{$-0.2$}} & $39.9${\footnotesize\textcolor{blue}{$+1.2$}} & $38.2${\footnotesize\textcolor{blue}{$+2.7$}} & $40.4${\footnotesize\textcolor{purple}{$-0.2$}} & $28.6${\footnotesize\textcolor{blue}{$+0.4$}} & $42.9${\footnotesize\textcolor{blue}{$+0.3$}} & $53.6$ & $25.5$\\
\cmidrule{2-10}
&$\cellcolor{gray!15}$RMU w/ truth  & $\cellcolor{gray!15}$$\textbf{45.5}${\footnotesize\textcolor{blue}{$+4.9$}} & $\cellcolor{gray!15}$$\textbf{48.2}${\footnotesize\textcolor{blue}{$+10.0$}} & $\cellcolor{gray!15}$$\textbf{42.1}${\footnotesize\textcolor{blue}{$+6.6$}} & $\cellcolor{gray!15}$$\textbf{48.0}${\footnotesize\textcolor{blue}{$+7.4$}} & $\cellcolor{gray!15}$$\textbf{31.2}${\footnotesize\textcolor{blue}{$+3.0$}} & $\cellcolor{gray!15}$$\textbf{46.3}${\footnotesize\textcolor{blue}{$+3.7$}} & $\cellcolor{gray!15}$$53.6$ & $\cellcolor{gray!15}$$25.5$\\
&$\cellcolor{gray!15}$RAd w/ truth  & $\cellcolor{gray!15}$$\textbf{50.9}${\footnotesize\textcolor{blue}{$+10.3$}} & $\cellcolor{gray!15}$$\textbf{54.1}${\footnotesize\textcolor{blue}{$+15.4$}} & $\cellcolor{gray!15}$$\textbf{46.8}${\footnotesize\textcolor{blue}{$+11.3$}} & $\cellcolor{gray!15}$$\textbf{54.6}${\footnotesize\textcolor{blue}{$+14.0$}} & $\cellcolor{gray!15}$$\textbf{34.1}${\footnotesize\textcolor{blue}{$+5.9$}} & $\cellcolor{gray!15}$$\textbf{49.9}${\footnotesize\textcolor{blue}{$+7.3$}} & $\cellcolor{gray!15}$$53.0$ & $\cellcolor{gray!15}$$25.0$\\
\bottomrule
\end{tabular}}
% \vspace{-1.0em}
\end{table*}

\section{Performance of RAd and RAb at Deeper Layers}
\label{appendix:deeper_layer}
Following~\citet{li2024wmdp}, unlearning and concept vector construction are performed at layer $l=7$. However, concepts' representations can be distributed across other layers. A full layer-wise grid search is computationally expensive. Here, we conduct an additional experiment at layer $15$ and layer $30$ with two TruthfulQA tasks and use the same experimental protocol to verify whether the observed controllability phenomenon generalizes to deeper layers. We tune the coefficient $c$, while fixing the retain weight $\alpha_r = 1200.0$ for RAd and $\alpha_r = 20.0$ for RAb. At layer $15$, we grid search $c \in \{30, 35, 40, 45, 50 \}$ for RAd and $c \in \{ 200, 210, 220, 230, 240, 250 \}$ for RAb. At layer $30$, we grid search $c \in \{ 2500, 2600, 2700, 2800, 2900 \}$ for RAd and $c \in \{ 3000, 3200, 3400, 3600, 3800, 4000 \}$ for RAb. To this end, we set $c = 35$ for RAd and $c = 210$ for RAb at layer $15$, $c = 2600$ for RAd and $c = 3200$ for RAb at layer $30$.

\begin{table*}[h]
\centering
\caption{Performance of RAd and RAb with random vector and truth vector at layers $7$ and $15$. \textcolor{blue}{Increase} and \textcolor{purple}{drops} are marked compared to the base model.}
\resizebox{0.9\linewidth}{!}{
\begin{tabular}{c l c c c c c c c c}
\toprule
\multirow{2}{*}{\textbf{Layer}}
  & \multirow{2}{*}{\textbf{Model}}
  & \multicolumn{4}{c}{\textbf{TruthfulQA open-ended}} 
  & \multicolumn{2}{c}{\textbf{TruthfulQA MC}} 
  & \multicolumn{2}{c}{\textbf{Unlearning}} \\
\cmidrule(lr){3-6} \cmidrule(lr){7-8} \cmidrule(lr){9-10}
  & & BLEU & ROUGE-1 & ROUGE-2 & ROUGE-L & MC1 & MC2 & MMLU & WMDP\\
\midrule
$-$ & Base model & $47.0$ & $45.5$ & $37.9$ & $42.6$ & $39.0$ & $55.0$ & $58.4$ & $54.4$ \\
\midrule
\multirow{4}{*}{$7$}
  & RAd w/ random & $49.5${\footnotesize \textcolor{blue}{$+2.5$}} & $47.7${\footnotesize \textcolor{blue}{$+2.2$}} & $39.5${\footnotesize \textcolor{blue}{$+1.6$}} & $44.3${\footnotesize \textcolor{blue}{$+1.7$}} & $38.4${\footnotesize \textcolor{purple}{$-0.6$}} & $55.9${\footnotesize \textcolor{blue}{$+0.9$}} & $55.9$ & $25.6$ \\
  & RAb w/ random & $51.2${\footnotesize \textcolor{blue}{$+4.2$}} & $49.7${\footnotesize \textcolor{blue}{$+4.2$}} & $41.6${\footnotesize \textcolor{blue}{$+3.7$}} & $46.8${\footnotesize \textcolor{blue}{$+4.2$}} & $38.6${\footnotesize \textcolor{purple}{$-0.4$}} & $55.6${\footnotesize \textcolor{blue}{$+0.6$}} & $57.7$ & $50.2$ \\
  
  & RAd w/ truth  & $47.7${\footnotesize \textcolor{blue}{$+0.7$}} & $53.9${\footnotesize \textcolor{blue}{$+8.4$}} & $40.9${\footnotesize \textcolor{blue}{$+3.0$}} & $51.9${\footnotesize \textcolor{blue}{$+9.3$}} & $44.9${\footnotesize \textcolor{blue}{$+5.9$}} & $62.3${\footnotesize \textcolor{blue}{$+7.3$}} & $54.9$ & $28.2$ \\
  & RAb w/ truth  & $41.1${\footnotesize \textcolor{purple}{$-5.9$}} & $41.9${\footnotesize \textcolor{purple}{$-3.6$}} & $31.6${\footnotesize \textcolor{purple}{$-6.3$}} & $40.9${\footnotesize \textcolor{purple}{$-1.7$}} & $26.1${\footnotesize \textcolor{purple}{$-12.9$}} & $40.0${\footnotesize \textcolor{purple}{$-15.0$}} & $52.0$ & $32.9$ \\
\midrule
\multirow{4}{*}{$15$}
  & RAd w/ random & $49.5${\footnotesize \textcolor{blue}{$+2.5$}} & $48.0${\footnotesize \textcolor{blue}{$+2.5$}} & $41.2${\footnotesize \textcolor{blue}{$+3.3$}} & $46.8${\footnotesize \textcolor{blue}{$+4.2$}} & $37.8${\footnotesize \textcolor{purple}{$-1.2$}} & $54.3${\footnotesize \textcolor{purple}{$-0.7$}} & $55.5$ & $29.9$ \\
  & RAb w/ random & $46.3${\footnotesize \textcolor{purple}{$-0.7$}} & $45.6${\footnotesize \textcolor{blue}{$+0.1$}} & $37.5${\footnotesize \textcolor{purple}{$-0.4$}} & $45.3${\footnotesize \textcolor{blue}{$+2.7$}} & $37.7${\footnotesize \textcolor{purple}{$-1.3$}} & $54.3${\footnotesize \textcolor{purple}{$-0.7$}} & $54.2$ & $49.1$ \\
  & RAd w/ truth  & $48.5${\footnotesize \textcolor{blue}{$+1.5$}} & $46.1${\footnotesize \textcolor{blue}{$+0.6$}} & $39.5${\footnotesize \textcolor{blue}{$+1.6$}} & $45.6${\footnotesize \textcolor{blue}{$+3.0$}} & $42.0${\footnotesize \textcolor{blue}{$+3.0$}} & $57.2${\footnotesize \textcolor{blue}{$+2.2$}} & $55.6$ & $30.1$ \\
  & RAb w/ truth  & $36.8${\footnotesize \textcolor{purple}{$-10.2$}} & $45.1${\footnotesize \textcolor{purple}{$-0.4$}} & $13.7${\footnotesize \textcolor{purple}{$-24.2$}} & $46.6${\footnotesize \textcolor{blue}{$+4.0$}} & $23.4${\footnotesize \textcolor{purple}{$-15.6$}} & $46.1${\footnotesize \textcolor{purple}{$-8.9$}} & $40.1$ & $32.3$ \\
\midrule
\multirow{4}{*}{$30$}
  & RAd w/ random & $50.0${\footnotesize \textcolor{blue}{$+3.0$}} & $50.0${\footnotesize \textcolor{blue}{$+4.5$}} & $40.4${\footnotesize \textcolor{blue}{$+2.5$}} & $50.0${\footnotesize \textcolor{blue}{$+7.4$}} & $40.4${\footnotesize \textcolor{blue}{$+1.4$}} & $56.1${\footnotesize \textcolor{blue}{$+1.1$}} & $58.7$ & $47.9$ \\
  & RAb w/ random & $2.2${\footnotesize \textcolor{purple}{$-44.8$}} & $3.4${\footnotesize \textcolor{purple}{$-42.1$}} & $0.2${\footnotesize \textcolor{purple}{$-37.7$}} & $3.2${\footnotesize \textcolor{purple}{$-39.4$}} & $20.8${\footnotesize \textcolor{purple}{$-18.2$}} & $45.7${\footnotesize \textcolor{purple}{$-9.3$}} & $30.9$ & $25.7$ \\
  & RAd w/ truth & $47.5${\footnotesize \textcolor{blue}{$+0.5$}} & $44.9${\footnotesize \textcolor{purple}{$-0.6$}} & $37.7${\footnotesize \textcolor{purple}{$-0.2$}} & $45.8${\footnotesize \textcolor{blue}{$+3.2$}} & $38.4${\footnotesize \textcolor{purple}{$-0.6$}} & $56.2${\footnotesize \textcolor{blue}{$+1.2$}} & $57.2$ & $32.0$ \\
  & RAb w/ truth & $3.7${\footnotesize \textcolor{purple}{$-43.3$}} & $6.6${\footnotesize \textcolor{purple}{$-38.9$}} & $0.7${\footnotesize \textcolor{purple}{$-37.2$}} & $6.9${\footnotesize \textcolor{purple}{$-35.7$}} & $14.3${\footnotesize \textcolor{purple}{$-24.7$}} & $35.0${\footnotesize \textcolor{purple}{-20.0}} & $46.8$ & $36.3$ \\
\bottomrule
\end{tabular}}
\label{tab:layer_15_results}
\end{table*}

Tab.~\ref{tab:layer_15_results} shows that while the controllability phenomenon persists at layer $15$, it is attenuated at layer $30$ for RAd with truth vector. Furthermore, we observe that at layers $15$ and $30$, RAb with truth vector causes larger drops in MMLU when the direction is ablated. This may align with the observation that features in deep nets become more entangled in deeper layers. While early layers learn general features, deeper layers combine these features into complex and often highly entangled representations that are specialized for the final output task~\citep{yosinski2014transferable}. When a concept's representation in latent space is entangled with other meaningful concepts, applying RAb with that concept extracted from deeper layers risks more damage to general performance than extraction from earlier layers.

\section{Robustness of RAd and RAb models against knowledge recovery}
\label{appendix:robustness}

As discussed in Sec.~\ref{sec:mechanism}, the mechanism of RAd and RAb may not be true erasure. Consequently, RAd and RAb models are not robust to knowledge recovery~\citep{huunlearning, lucki2025an}. Unlearned knowledge can be resurfaced through relearning~\citep{li2024wmdp, lucki2025an}, targeted attacks~\citep{hu2025unlearning}. In this section, we evaluate the robustness of RAd and RAb models against these knowledge recovery attacks. Following~\citet{lucki2025an}, we employ five methods: Logitlens~\citep{logitlensblog}, finetuning, orthogonalization~\citep{arditi2024refusal}, enhanced GCG~\citep{lucki2025an}, and pruning~\citep{weiassessing}. In the following, we define the threat model and provide details of attack methods and the results of knowledge recovery attacks.
\subsection{Threat model}
Following~\citet{lucki2025an}, we consider a white-box scenario where a knowledge recovery attacker has full access to the base and unlearned model's parameters, allowing for modifications at inference time. We further assume that (a subset of) the unlearning dataset is exposed to the attacker.
\subsection{Attack methods}
\textbf{Logitlens\textnormal{~\citep{logitlensblog}}.} Using Logitlens, we project the final token's activations at each transformer layer onto the model's vocabulary to identify the answer token. Concretely, for WMDP QAs, we add a prefix to each question, extract the projected logits for the answer tokens ``\texttt{A},'' ``\texttt{B},'' ``\texttt{C},'' ``\texttt{D},'' and select the token with the highest probability as the prediction. The question prefix is: ``\texttt{Answer the following question with A, B, C, or D.\textbackslash n\textbackslash n}".

\textbf{Finetuning.} We consider three settings: (1) \emph{Forget}: finetuning the unlearned model using forget-samples from forget-sets, (2) \emph{Forget-relevant}: finetuning the unlearned model using forget-relevant samples from a closely related domain dataset, and (3) \emph{Forget-irrelevant}: finetuning the unlearned model using forget-irrelevant samples. We vary the sample count from $5$ to $1000$. LoRA~\citep{hu2022lora} is used for efficiency. Chat templates used for finetuning are specified in Tab.~\ref{tab:chat_templates}. 

We utilize the WMDP forget and retain splits as forget and forget-relevant datasets for both Biology and Cyber domains. Wikitext~\citep{merity2017pointer} serves as a forget-irrelevant dataset. Finetuning hyperparameters are specified in Tab.~\ref{tab:finetuning_hparams}. 

\textbf{Orthogonalization.} Orthogonalization ablates the ``unlearning directions'' in the unlearned models. The unlearning direction is defined as the difference-in-means~\citep{diffinmeans} between activations of the unlearned model and the base model.
\subsection{Experimental setup}
\begin{wraptable}{r}{0.3\textwidth}
\vspace{-3mm}
\caption{
Hyperparameters for finetuning as knowledge extraction.
}
\label{tab:finetuning_hparams}
\centering
\resizebox{0.9\linewidth}{!}{
\setlength{\tabcolsep}{6pt}
\begin{tabular}{lc}
\toprule
{\textbf{Hyperparameter}} 
& {\textbf{Value}} \\
    \midrule
    LoRA rank & $128$ \\
    LoRA target modules & all linear \\
    LoRA alpha & $16$ \\
    LoRA dropout & $0$ \\
    LoRA bias & none \\
    \midrule
    Maximum sequence length & $1024$ \\
    Epochs & $3$ \\
    Batch size & $1$ \\
    Gradient accumulation steps & $1$ \\
    Learning rate & $2e-4$ \\
    Learning rate scheduler & linear \\
    Warmup ratio & $0.05$ \\
    Optimizer & AdamW \\
    Weight decay & $0.01$ \\
    \bottomrule
\end{tabular}}
\vspace{-5mm}
\end{wraptable}
For forget and forget-relevant settings, we employ synthetic preference datasets~\citep{lucki2025an} to extract the unlearning direction of biology and cyber domains in the unlearned models. These are available at \url{https://huggingface.co/datasets/J4Q8/bio_forget_dpo} and \url{https://huggingface.co/datasets/J4Q8/cyber_forget_dpo}, respectively. For the forget-irrelevant setting, we employ the MMLU validation set to extract the unlearning direction. 

\begin{table*}[t]
\centering
\small
\caption{Chat templates for finetuning unlearned models.}
\renewcommand{\arraystretch}{1.15}
\begin{tabular}{@{}l p{0.72\linewidth}@{}}
\toprule
\textbf{Setting} & \textbf{Template} \\
\midrule
\makecell[l]{Forget \& forget-relevant\\(biology)} & \texttt{[}\newline
   \hspace*{1em}\texttt{\{"role": "system", "content": "You are a helpful assistant."\},}\newline
   \hspace*{1em}\texttt{\{"role": "user", "content": "Write a research article in the field of biology."\}}\newline
   \hspace*{1em}\texttt{\{"role": "assistant", "content": "Of course, here is a research article in the field of biology. \{text\}"\}}\newline
   \texttt{]} \\
\midrule
\makecell[l]{Forget \& forget-relevant\\(cyber)} & \texttt{[}\newline
   \hspace*{1em}\texttt{\{"role": "system", "content": "You are a helpful assistant."\},}\newline
   \hspace*{1em}\texttt{\{"role": "user", "content": "Write a github page about cybersecurity."\}}\newline
   \hspace*{1em}\texttt{\{"role": "assistant", "content": "Of course, here is a github page about cybersecurity. \{text\}"\}}\newline
   \texttt{]} \\
\midrule
\makecell[l]{Forget-irrelevant\\(Wikitext)} & \texttt{[}\newline
   \hspace*{1em}\texttt{\{"role": "system", "content": "You are a helpful assistant."\},}\newline
   \hspace*{1em}\texttt{\{"role": "user", "content": "Write a wikipedia article."\}}\newline
   \hspace*{1em}\texttt{\{"role": "assistant", "content": "Of course, here is a wikipedia article. \{text\}"\}}\newline
   \texttt{]} \\
\bottomrule
\end{tabular}\label{tab:chat_templates}
\end{table*}

\textbf{Enhanced GCG.} 
% GCG~\citep{zou2023universal} is reported ineffective against RMU~\citep{li2024wmdp, dang2025effects}, 
Enhanced GCG~\citep{lucki2025an} is a variant of Greedy Coordinate Gradient (GCG;~\citet{zou2023universal}), designed to attack unlearned models by injecting an optimized adversarial prefix into the input prompt at inference.  
% The method mutates random token position through swapping, insertion, or deletion~\citep{thompson2024flrt}, retaining only top-performing candidates per iteration. Its objective function combines feature-based distillation with cross-entropy loss, the latter utilizing loss clamping to reduce optimization effort on relatively well-solved tokens. All losses are computed relative to target strings generated by the base model on hazardous questions with a candidate prefix appended.
Following~\citet{lucki2025an}, the adversarial prefix is optimized for $1,500$ gradient update steps using a chat template, and $L_2$ distillation loss computed on activations at layers $5$, $6$, and $7$. The attack is performed using five domain-specific multiple-choice questions correctly answered by the base model. The universal adversarial prefix has over $100$ tokens.

\textbf{Pruning.} Pruning-based attack isolates parameters critical to unlearning.  We employ set difference pruning~\citep{weiassessing}, using SNIP score~\citep{leesnip} to quantify each parameter's influence on the forgetting objective and retaining objective. Parameters that rank in the top-$q$\% by forgetting influence but outside the top-$p$\% by retaining influence are pruned.

We perform a grid search for $p, q \in \{0.5, 1.0, 2.5, 5.0, 7.5\}$, and report the combination yielding the highest WMDP accuracy. We use $128$ samples of WMDP forget-sets and Wikitext to quantify the influence of neurons on forgetting and retaining, respectively.

\subsection{Attack results}

\begin{table*}[t]

\caption{Accuracy under attack of RAd and RAb models measured on WMDP-Bio,  WMDP-Cyber QAs, and MMLU. All attacks are conducted using the WMDP-Bio forget-set. For sentiment, experiments are conducted using the \texttt{neg$\to$pos} direction. $^{*}$For Logitlens, we report results of attacking the last layer. For finetuning, we report results of finetuning using $5$ forget-sample from WMDP-Bio.}
\vspace{-0.5em}
\label{tab:recovery_attack_bio}
\centering
\vspace{0.15cm}
\resizebox{\linewidth}{!}{
\setlength{\tabcolsep}{4pt}
\begin{tabular}{llccccccccc}
\toprule
\multirow{2}{*}{\textbf{Benchmark}}
& \multirow{2}{*}{\textbf{Knowledge Recovery}}
& \multirow{2}{*}{\textbf{Base model}} 
& \multicolumn{4}{c}{\textbf{RAd models}}
& \multicolumn{4}{c}{\textbf{RAb models}} \\
\cmidrule(lr){4-7} \cmidrule(lr){8-11}
& & & \textbf{random} & \textbf{truthfulness} & \textbf{sentiment} & \textbf{refusal}
& \textbf{random} & \textbf{truthfulness} & \textbf{sentiment} & \textbf{refusal} \\

\midrule
\multirow{6}{*}{WMDP-Bio ($\downarrow$)}
& No attack          & $63.9$ & $26.8$ & $29.7$ & $26.5$ & $26.2$ & $60.5$ & $39.8$ & $38.8$ & $48.3$ \\
\cmidrule{2-11}
& Logitlens$^*$         & $-$    & $26.8$ & $28.0$ & $25.8$ & $26.6$ & $61.0$ & $30.2$ & $34.6$ & $45.2$ \\
& Finetuning$^*$         & $-$    & $59.0$ & $25.3$ & $29.1$ & $44.8$ & $62.9$ & $63.8$ & $58.1$ & $61.5$ \\
& Orthogonalization  & $-$    & $62.8$ & $62.8$ & $63.8$ & $62.5$ & $61.4$ & $54.4$ & $50.7$ & $60.6$ \\
& Enhanced GCG       & $-$    & $30.1$& $33.1$ & $26.3$ & $41.4$ & $59.7$ & $44.4$ & $42.4$ & $39.0$ \\
& Pruning            & $-$    & $57.2$ & $56.1$& $49.9$ & $47.6$ & $53.7$ & $54.0$ & $51.8$ & $53.6$ \\

\midrule
\multirow{6}{*}{WMDP-Cyber ($\downarrow$)}
& No attack          & $43.3$ & $25.3$ & $26.2$ & $25.7$ & $27.2$ & $40.6$ & $28.9$ & $33.1$ & $25.6$ \\
\cmidrule{2-11}
& Finetuning$^*$         & $-$    & $33.6$ & $24.8$ & $25.1$ & $26.5$ & $42.4$ & $38.6$ & $40.5$ & $34.7$ \\
& Orthogonalization  & $-$    & $41.2$ & $40.1$ & $42.1$ & $42.3$ & $39.1$ & $39.8$ & $37.2$ & $31.5$ \\
& Enhanced GCG       & $-$    & $25.4$ & $26.4$ & $27.0$ & $25.5$ & $38.4$ & $28.1$ & $34.3$ & $27.8$ \\
& Pruning            & $-$    & $38.7$ & $39.4$ & $25.4$ & $25.6$ & $41.7$ & $36.0$ & $40.1$ & $31.4$ \\

\midrule
\multirow{6}{*}{MMLU ($\uparrow$)}
& No attack          & $58.4$ & $55.9$ & $54.9$ & $54.8$ & $51.7$ & $57.7$ & $52.0$ & $49.5$ & $54.2$ \\
\cmidrule{2-11}
% & Logitlens         & $-$   & $-$ & $-$ & $-$ & $-$ & $-$ & $-$ & $-$ & $-$ \\
& Finetuning$^*$         & $-$   & $57.5$ & $47.4$ & $56.3$ & $57.6$ & $58.5$ & $57.8$ & $57.0$ & $57.7$ \\
& Orthogonalization  & $-$   & $57.4$ & $57.6$ & $58.1$ & $58.1$ & $56.2$ & $51.2$ & $46.5$ & $57.0$ \\
& Enhanced GCG       & $-$   & $56.1$ & $54.5$ & $53.4$ & $51.2$ & $58.1$ & $52.1$ & $49.5$ & $54.2$ \\
& Pruning            & $-$   & $56.5$ & $56.4$ & $54.6$ & $49.5$ & $57.1$ & $55.2$ & $53.2$ & $55.3$ \\

\bottomrule
\end{tabular}}
% \vspace{-0.5em}
\end{table*}

\begin{table*}[h]
\caption{Accuracy under attack of RAd and RAb models measured on WMDP-Bio,  WMDP-Cyber QAs, and MMLU. All attacks are conducted using the WMDP-Cyber forget-set. For sentiment, experiments are conducted using the \texttt{neg$\to$pos} direction. $^{*}$For Logitlens, we report results of attacking the last layer. For finetuning, we report results of finetuning using $5$ forget-sample from WMDP-Cyber.}
\vspace{-0.2em}
\label{tab:recovery_attack_cyber}
\centering
\resizebox{\linewidth}{!}{
\setlength{\tabcolsep}{4pt}
\begin{tabular}{llccccccccc}
\toprule
\multirow{2}{*}{\textbf{Benchmark}}
& \multirow{2}{*}{\textbf{Attack}}
& \multirow{2}{*}{\textbf{Base model}} 
& \multicolumn{4}{c}{\textbf{RAd}}
& \multicolumn{4}{c}{\textbf{RAb}} \\
\cmidrule(lr){4-7} \cmidrule(lr){8-11}
& & & \textbf{random} & \textbf{truthfulness} & \textbf{sentiment} & \textbf{refusal}
& \textbf{random} & \textbf{truthfulness} & \textbf{sentiment} & \textbf{refusal} \\

\midrule
\multirow{6}{*}{WMDP-Cyber ($\downarrow$)}
& No attack          & $43.3$ & $25.3$ & $26.2$ & $25.7$ & $27.2$ & $40.6$ & $28.9$ & $33.1$ & $25.6$ \\
\cmidrule{2-11}
& Logitlens$^{*}$        & $-$    & $25.1$ & $26.2$ & $25.6$ & $26.7$ & $40.6$ & $27.1$ & $32.3$ & $27.1$ \\
& Finetuning$^{*}$         & $-$    & $42.3$ & $28.1$ & $25.9$ & $27.5$ & $41.8$ & $40.6$ & $39.8$ & $37.7$ \\
& Orthogonalization  & $-$    & $41.1$ & $40.6$ & $41.5$ & $41.0$ & $39.0$ & $42.2$ & $33.2$ & $41.5$ \\
& Enhanced GCG       & $-$    & $24.4$ & $26.6$ & $24.6$ & $25.8$ & $38.4$ & $30.2$ & $34.1$ & $29.9$ \\
& Pruning            & $-$    & $40.4$ & $39.1$ & $33.5$ & $25.7$ & $40.0$ & $37.8$ & $38.3$ & $34.0$ \\

\midrule
\multirow{6}{*}{WMDP-Bio ($\downarrow$)}
& No attack          & $63.9$ & $26.8$ & $29.7$ & $26.5$ & $26.2$ & $60.5$ & $39.8$ & $38.8$ & $48.3$ \\
\cmidrule{2-11}
% & Logitlens         & $-$    & $-$ & $-$ & $-$ & $-$ & $-$ & $-$ & $-$ & $-$ \\
& Finetuning        & $-$    & $58.1$ & $34.3$ & $27.6$ & $28.1$ & $63.5$ & $63.0$ & $53.2$ & $61.6$ \\
& Orthogonalization  & $-$    & $63.0$ & $62.1$ & $64.1$ & $62.6$ & $62.5$ & $59.3$ & $34.2$ & $62.2$ \\
& Enhanced GCG       & $-$    & $28.7$ & $33.9$ & $26.6$ & $25.8$ & $60.4$ & $48.2$ & $40.1$ & $52.2$ \\
& Pruning            & $-$    & $57.9$ & $56.7$ & $29.5$ & $30.1$ & $61.9$ & $59.1$ & $51.5$ & $55.4$ \\

\midrule
\multirow{6}{*}{MMLU ($\uparrow$)}
& No attack          & $58.4$ & $55.9$ & $54.9$ & $54.8$ & $51.7$ & $57.7$ & $52.0$ & $49.5$ & $54.2$ \\
\cmidrule{2-11}
% & Logitlens         & $-$   & $-$ & $-$ & $-$ & $-$ & $-$ & $-$ & $-$ & $-$ \\
& Finetuning        & $-$   & $58.2$ & $56.7$ & $56.0$ & $55.8$ & $58.6$ & $57.8$ & $55.5$ & $57.5$ \\
& Orthogonalization  & $-$   & $57.6$ & $58.0$ & $58.3$ & $58.2$ & $56.1$ & $54.7$ & $36.9$ & $58.0$ \\
& Enhanced GCG       & $-$   & $56.1$ & $54.5$ & $53.4$ & $51.2$ & $58.1$ & $52.1$ & $49.5$ & $54.2$ \\
& Pruning            & $-$   & $57.0$ & $56.6$ & $50.2$ & $49.3$ & $57.4$ & $55.5$ & $52.4$ & $53.5$ \\

\bottomrule

\end{tabular}}
\vspace{4pt}
\end{table*}

Tab.~\ref{tab:recovery_attack_bio} reports accuracy under attack (AuA) when knowledge recovery attacks are conducted using the WMDP-Bio forget-set (see Tab.~\ref{tab:recovery_attack_cyber} for analogous results of attacks using the WMDP-Cyber forget-set). Overall, \textbf{unlearned models are vulnerable to knowledge recovery, regardless of concept vectors.}
Attacks that directly modify model parameters, such as finetuning, orthogonalization, and pruning, can substantially restore forgotten knowledge, often recovering performance to near the base model's accuracy. In contrast, Logitlens and enhanced GCG are generally less effective. This is expected given the underlying mechanisms of RAd and RAb, which manipulate the forget-samples' representations. Logitlens relies on mapping these representations to the vocabulary space; when the representations are altered or suppressed, Logitlens fails to surface the forgotten knowledge. Enhanced GCG relies on gradient signals to identify token substitutions in the prefix that increase the probability of a target output; however, when representations are manipulated, the attacker is likely to receive uninformative gradient signals from the unlearned models~\citep{dang2025effects}. Furthermore, attacks targeting the Biology domain can also induce knowledge recovery in the Cyber domain.

\begin{figure}[t]
    \centering
    \begin{subfigure}[t]{0.49\linewidth}
        \centering
        \includegraphics[width=\linewidth]{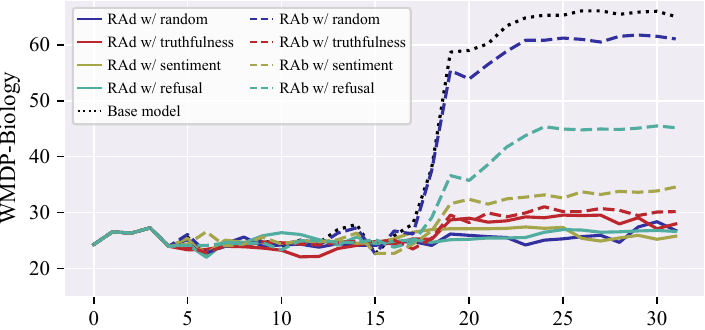}
        \caption{WMDP-Bio QA set.}
        \label{fig:logitlens_bio}
    \end{subfigure}
    \hfill
    \begin{subfigure}[t]{0.49\linewidth}
        \centering
        \includegraphics[width=\linewidth]{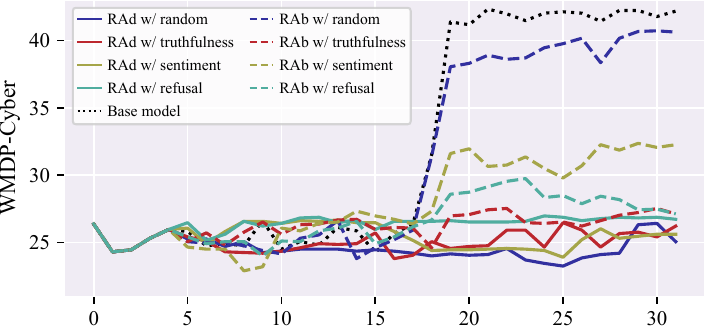}
        \caption{WMDP-Cyber QA set.}
        \label{fig:logitlens_cyber}
    \end{subfigure}
    \caption{Layer-wise knowledge recovery attack performance of Logitlens on the WMDP-Cyber and WMDP-Bio QA sets.}
    \label{fig:logitlens}
\end{figure}

\begin{figure}[t]
    \centering
    \begin{subfigure}[t]{0.49\linewidth}
        \centering
        \includegraphics[width=\linewidth]{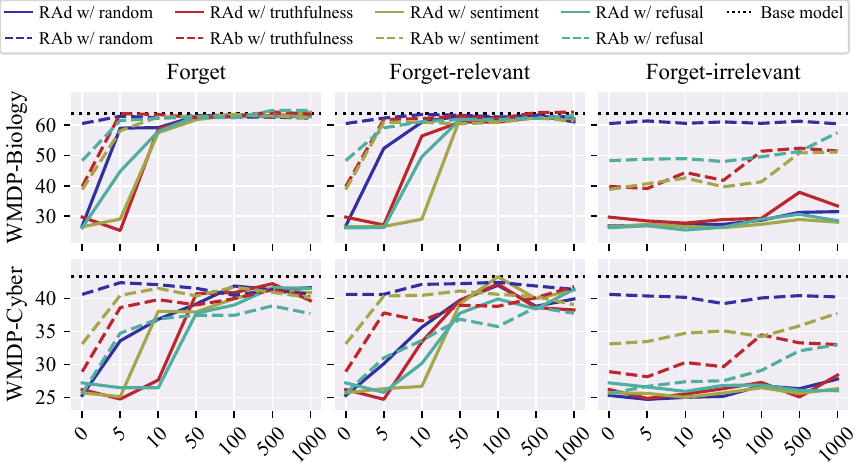}
        \label{fig:finetuning_attack_bio}
        \caption{Finetuning on WMDP-Bio forget-samples (forget), WMDP-Bio retain-samples (forget-relevant), and Wikitext samples (forget-irrelevant).}
    \end{subfigure}
    \hfill
    \begin{subfigure}[t]{0.49\linewidth}
        \centering
        \includegraphics[width=\linewidth]{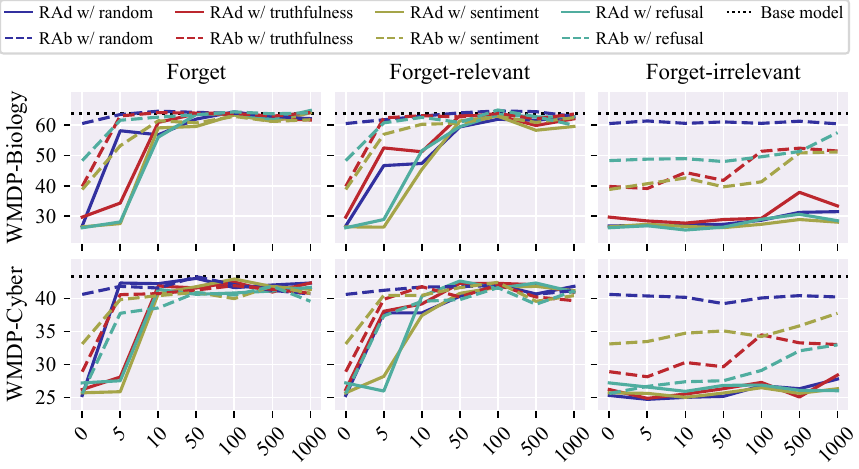}
        \label{fig:finetuning_attack_cyber}
        \caption{Finetuning on WMDP-Cyber forget-samples (forget), WMDP-Cyber retain-samples (forget-relevant), and Wikitext samples (forget-irrelevant).}
    \end{subfigure}
    \caption{Finetuning on forget or forget-relevant samples of one WMDP domain recovers forgotten knowledge in the other domain. }
    \label{fig:finetuning_attack}
    \vspace{-0.5em}
\end{figure}

\textbf{Ablation studies on Logitlens and finetuning.} For Logitlens, we perform attacks across layers. While RAd models remain robust across all layers, RAb models show vulnerability at middle layers (see Fig.~\ref{fig:logitlens}). 
% \begin{figure}[t]
%     \centering
%     \includegraphics[width=0.5\linewidth]{assets/logitlens_wmdp-bio.pdf}
%     \caption{Layer-wise knowledge recovery attack performance of Logitlens on the WMDP-Biology QA set.}
%     \label{fig:logitlens_bio}
% \end{figure}

% \begin{figure}[t]
%     \centering
%     \includegraphics[width=0.5\linewidth]{assets/finetuning_wmdp_bio.pdf}
%     \caption{Finetuning on WMDP-Biology forget-samples (forget), WMDP-Biology retain-samples (forget-relevant), and Wikitext samples (forget-irrelevant). Finetuning on WMDP-Biology forget or forget-relevant samples recovers forgotten knowledge in the WMDP-Cyber domain.}
%     \label{fig:finetuning_attack_bio}
% \end{figure}
Fig.~\ref{fig:finetuning_attack} shows that forgotten knowledge is fully recovered when unlearned models are finetuned on a small number of forget or forget-relevant samples.
RAd models appear more robust than RAb models, whereas finetuning on forget-irrelevant samples fails to recover the forgotten knowledge.

\section{Experiments on MUSE}
\label{appendix:muse}

\textbf{Experimental setup.} We evaluate RAd and RAb on MUSE-News and MUSE-Books~\citep{shi2025muse}. We employ the base models provided by~\citet{shi2025muse} and fine-tune them for $T = 2,000$ steps. We fix the forget and retain weights to $\alpha_f = 1.0$ and $\alpha_r = 10.0$, and perform a grid search for the coefficient $c$. We report $c = 45.0$ for RAd on MUSE-News, $c = 400.0$ for RAb on MUSE-News, $c = 35.0$ for RAd on MUSE-Books, and $c = 450.0$ for RAb on MUSE-Books.

\textbf{Results.} Experimental results are shown in Tab.~\ref{tab:truthfulqa_muse} and Tab.~\ref{tab:truthfulqa_muse_books}. On TruthfulQA, unlearning via RAd with truth vector consistently improves performance across both open-ended and multiple-choice settings, outperforming the base models and unlearning via RAd with random direction. Furthermore, unlearning via RAb with truth direction exhibits an ability to suppress truthful responses, as evidenced by marked declines in performance.

% On MUSE-News and MUSE-Books, both RAd and RAb exhibit mixed unlearning performance. On MUSE-Books, RAb with truthfulness direction archives the best balance between forgetting and retention. However, on MUSE-News, it reduces the knowledge memorization of the retain-set. In comparison, RAd does not archive balance between forgetting and retention on both datasets.

% However, it over-unlearns on MUSE-News, as indicated by $\textnormal{PrivLeak} \gg 0$. In comparison, RAd tends to under-unlearn on MUSE-Books, and over-unlearn on MUSE-News.

\begin{table*}[h]
\caption{Performance comparison on TruthfulQA and MUSE-News benchmarks.}
\label{tab:truthfulqa_muse}
\vspace{-0.5em}
\centering
\resizebox{\linewidth}{!}{
\setlength{\tabcolsep}{4pt}
\begin{tabular}{lcccccccccc}
\toprule
\multirow{2}{*}{\textbf{Models}} 
& \multicolumn{4}{c}{\textbf{TruthfulQA open-ended}} 
& \multicolumn{2}{c}{\textbf{TruthfulQA multiple-choice}} 
& \multicolumn{4}{c}{\textbf{MUSE-News}} \\
\cmidrule(lr){2-5}
\cmidrule(lr){6-7}
\cmidrule(lr){8-11}
& BLEU & R-1 & R-2 & R-L 
& MC1 & MC2 
& $\text{VerbMem}_{f}$($\downarrow$) 
& $\text{KnowMem}_{f}$($\downarrow$) 
& $\text{PrivLeak}$ 
& $\text{KnowMem}_{r}$($\uparrow$) \\
\midrule

Base model 
& $40.0$ & $37.0$ & $28.2$ & $37.3$ 
& $26.9$ & $44.2$ 
& $57.6$ & $64.4$ 
& $-99.8$ & $53.6$ \\
\midrule
RAd w/ random 
& $40.9${\footnotesize\textcolor{blue}{$+0.9$}}
& $37.7${\footnotesize\textcolor{blue}{$+0.7$}}
& $28.9${\footnotesize\textcolor{blue}{$+0.7$}}
& $37.3${\footnotesize\textcolor{blue}{$+0.0$}}
& $24.4${\footnotesize\textcolor{purple}{$-2.5$}}
& $42.2${\footnotesize\textcolor{purple}{$-2.0$}}
& $5.4$ & $46.8$ 
& $80.4$ & $46.3$ \\

RAd w/ truth 
& $44.4${\footnotesize\textcolor{blue}{$+4.4$}}
& $40.9${\footnotesize\textcolor{blue}{$+3.9$}}
& $30.6${\footnotesize\textcolor{blue}{$+2.4$}}
& $40.0${\footnotesize\textcolor{blue}{$+2.7$}}
& $27.5${\footnotesize\textcolor{blue}{$+0.6$}}
& $45.6${\footnotesize\textcolor{blue}{$+1.4$}}
& $10.0$ & $33.0$ 
& $76.8$ & $43.1$ \\
\midrule
RAb w/ random 
& $35.8${\footnotesize\textcolor{purple}{$-4.2$}}
& $33.8${\footnotesize\textcolor{purple}{$-3.2$}}
& $27.0${\footnotesize\textcolor{purple}{$-1.2$}}
& $34.6${\footnotesize\textcolor{purple}{$-2.7$}}
& $23.5${\footnotesize\textcolor{purple}{$-3.4$}}
& $39.5${\footnotesize\textcolor{purple}{$-4.7$}}
& $12.5$ & $53.4$ 
& $65.9$ & $48.8$ \\

RAb w/ truth 
& $40.9${\footnotesize\textcolor{blue}{$+0.9$}}
& $37.7${\footnotesize\textcolor{blue}{$+0.7$}}
& $29.4${\footnotesize\textcolor{blue}{$+1.2$}}
& $37.5${\footnotesize\textcolor{blue}{$+0.2$}}
& $14.8${\footnotesize\textcolor{purple}{$-12.1$}}
& $25.8${\footnotesize\textcolor{purple}{$-18.4$}}
& $4.2$ & $9.8$ 
& $59.7$ & $42.7$ \\

\bottomrule
\end{tabular}}
\vspace{-0.5em}
\end{table*}

\begin{table*}[h]
\caption{Performance comparison on TruthfulQA and MUSE-Books benchmarks.}
\vspace{-0.5em}
\label{tab:truthfulqa_muse_books}
\centering
\resizebox{\linewidth}{!}{
\setlength{\tabcolsep}{4pt}
\begin{tabular}{lcccccccccc}
\toprule
\multirow{2}{*}{\textbf{Models}} 
& \multicolumn{4}{c}{\textbf{TruthfulQA open-ended}} 
& \multicolumn{2}{c}{\textbf{TruthfulQA multiple-choice}} 
& \multicolumn{4}{c}{\textbf{MUSE-Books}} \\
\cmidrule(lr){2-5}
\cmidrule(lr){6-7}
\cmidrule(lr){8-11}
& BLEU & R-1 & R-2 & R-L 
& MC1 & MC2 
& $\text{VerbMem}_f$($\downarrow$) 
& $\text{KnowMem}_f$($\downarrow$) 
& $\text{PrivLeak}$ 
& $\text{KnowMem}_r$($\uparrow$) \\
\midrule

Base model 
& $29.9$ & $27.2$ & $18.4$ & $27.0$ 
& $21.4$ & $34.3$ 
& $99.7$ & $46.4$ 
& $-57.3$ & $67.7$ \\

% Retrained model 
% & $-$ & $-$ & $-$ & $-$ 
% & $-$ & $-$ 
% & $14.5$ & $29.1$ 
% & $-0.1$ & $68.5$ \\

\midrule
RAd w/ random 
& $37.0${\footnotesize\textcolor{blue}{$+7.1$}}
& $35.8${\footnotesize\textcolor{blue}{$+8.6$}}
& $19.9${\footnotesize\textcolor{blue}{$+1.5$}}
& $37.3${\footnotesize\textcolor{blue}{$+10.3$}}
& $21.3${\footnotesize\textcolor{purple}{$-0.1$}}
& $36.0${\footnotesize\textcolor{blue}{$+1.7$}}
& $11.0$ & $48.5$ 
& $-47.0$ & $65.8$ \\

RAd w/ truth 
& $39.2${\footnotesize\textcolor{blue}{$+9.3$}}
& $38.5${\footnotesize\textcolor{blue}{$+11.3$}}
& $26.5${\footnotesize\textcolor{blue}{$+8.1$}}
& $38.0${\footnotesize\textcolor{blue}{$+11.0$}}
& $24.2${\footnotesize\textcolor{blue}{$+2.8$}}
& $41.3${\footnotesize\textcolor{blue}{$+7.0$}}
& $11.6$ & $41.9$ 
& $-46.4$ & $56.1$ \\
\midrule
RAb w/ random 
& $27.2${\footnotesize\textcolor{purple}{$-2.7$}}
& $30.4${\footnotesize\textcolor{blue}{$+3.2$}}
& $6.6${\footnotesize\textcolor{purple}{$-11.8$}}
& $29.2${\footnotesize\textcolor{blue}{$+2.2$}}
& $26.2${\footnotesize\textcolor{blue}{$+4.8$}}
& $52.2${\footnotesize\textcolor{blue}{$+17.9$}}
& $1.3$ & $2.6$ 
& $27.3$ & $2.1$ \\

RAb w/ truth 
& $23.8${\footnotesize\textcolor{purple}{$-6.1$}}
& $28.9${\footnotesize\textcolor{blue}{$+1.7$}}
& $0.7${\footnotesize\textcolor{purple}{$-17.7$}}
& $30.1${\footnotesize\textcolor{blue}{$+3.1$}}
& $15.3${\footnotesize\textcolor{purple}{$-6.1$}}
& $34.9${\footnotesize\textcolor{blue}{$+0.6$}}
& $6.6$ & $3.2$ 
& $2.9$ & $60.6$ \\

\bottomrule
\end{tabular}}
\vspace{-0.5em}
\end{table*}

\section{Additional results}
\label{appendix:additional_results}
\subsection{Analysis on Generated Outputs of Unlearned Models}
\label{sec:grammar}

Unlearning evaluations that primarily rely on accuracy seem too coarse to capture the full extent of unlearning effectiveness in terms of model outputs' grammatical correctness and coherence. We further conduct a text quality analysis on generated outputs from unlearned models under two scenarios: (1) RAd models when using a concept vector versus a random vector, and (2) RAd versus other methods. We employ Qwen2.5-32B-Instruct as an LLM-as-a-judge and perform pairwise comparison between generated texts from unlearned models (see Fig.~\ref{fig:llm-as-a-judge-prompt} for details of the prompt). Win rates are shown in blue or red, while draw rates are shown in gray. 
\begin{figure}[H]
\centering
\includegraphics[width=0.7\linewidth]{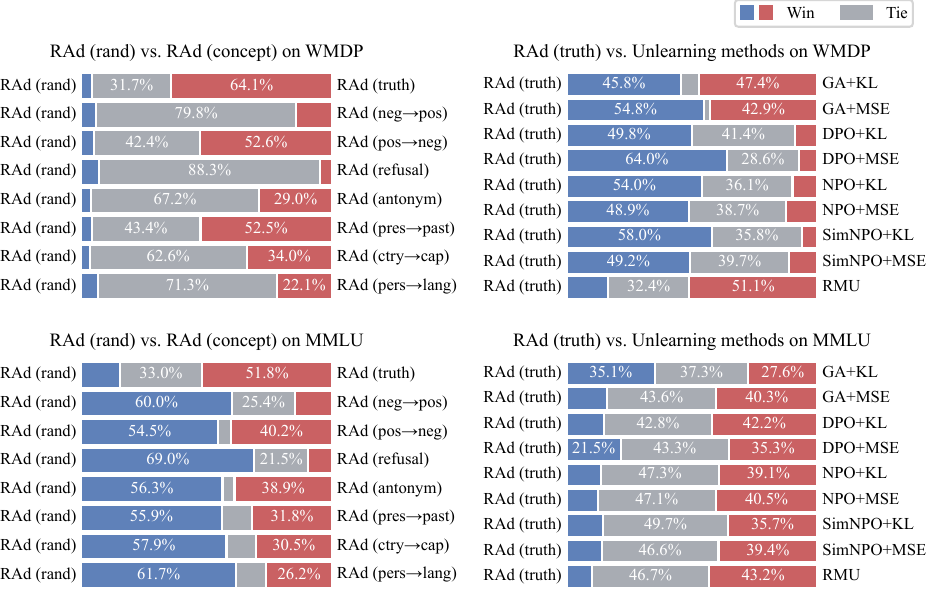}
\caption{\textbf{Left.} Win rate of RAd w/ random over RAd w/ concept. \textbf{Right.} Win rate of RAd w/ truth over other unlearning baselines.}
\label{fig:win_rate}
% \vspace{-1.5em}
\end{figure}

\textbf{Trade-off between forget and retain text quality.} Fig.~\ref{fig:win_rate} reveals a trade-off: RAd with concept vector often produces more coherent, grammatically correct generations than RAd with a random direction and other baselines on WMDP, but it degrades the quality of outputs on MMLU. We intuitively explain that steering forget-sample's representations toward a structured, meaningful region of the concept's representation space, rather than toward random noise, leads the unlearned models to produce more coherent, grammatically correct texts in response to forget inputs. However, the concept direction may be entangled with existing concepts in the retain domains, causing representational shifts in the retain domains that degrade the retain outputs' quality. 

\begin{figure}[h]
\vspace{-1.0em}
    \begin{tcolorbox}[boxrule=0pt, colframe=white, fontupper=\normalfont\small\ttfamily]
        Please act as an impartial judge and evaluate the quality of the responses provided by two AI assistants to a user question. \\
        
        You will be given responses from assistant A and assistant B. Your job is to choose the assistant whose response is more grammatically correct. A grammatically correct response should contain valid, meaningful natural-language text. A grammatically incorrect response may contain primarily symbols, formatting artifacts, repeated punctuation, random characters, or nonsensical text. If no response is more useful or meaningful than the other, it is a tie. During your evaluation, please ignore factual correctness of the responses. \\
        
        Avoid any position biases and ensure that the order in which the responses were presented does not influence your decision. Do not allow the length of the responses to influence your evaluation. Do not favor certain names of the assistants. Be as objective as possible. After providing your explanation, output your final verdict by strictly following this format: "[[A]]" if assistant A is better, "[[B]]" if assistant B is better, and "[[C]]" for a tie. \\
        
        [The Start of Assistant A’s Response] \\
        \{response\_a\}
        
        [The End of Assistant A’s Response] \\
        
        [The Start of Assistant B’s Response] \\
        \{response\_b\}
        
        [The End of Assistant B’s Response] \\
        
        Return your verdict in the following JSON format: \\
        \{\\
            \hspace*{2em} "verdict": "[[A]]" or "[[B]]" or "[[C]]", \\
        	\hspace*{2em} "explanation": "<brief explanation of your evaluation>" \\
        \}.
    \end{tcolorbox}
    \caption{Prompt used for the LLM-as-a-judge pairwise comparison of grammatical correctness. The LLM judge is tasked with determining which answer is more \emph{grammatically correct}. We ask the LLM judge to respond in a specified JSON format to facilitate verdict extraction. During comparison, we load the LLM judge with 4-bit GPTQ~\citep{frantar2023optq}. We use greedy decoding and set the maximum number of generated tokens to $200$.}
    \label{fig:llm-as-a-judge-prompt}
    \vspace{-1.5em}
\end{figure}

\subsection{Alignment between random vector and concept vectors}
\label{appendix:alignment_random_concept}
\begin{figure}[h]
    % \vspace{-1.0em}
    \centering
    \includegraphics[width=0.45\linewidth]{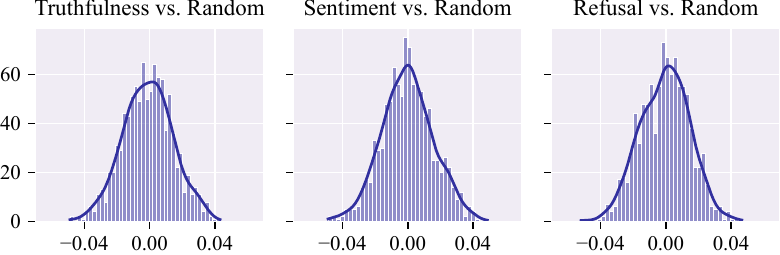}
    \caption{Alignment between random and concept directions.}
    \label{fig:concept_vector_alignment}
    % \vspace{-3.5em}
\end{figure}
We empirically study the alignment between random vectors and concept directions for truthfulness, sentiment, and refusal. Fig.~\ref{fig:concept_vector_alignment} reports the cosine similarity between random vectors and the concept directions. The similarities are small and concentrated around zero.

\subsection{Robustness across Runs}

\begin{figure}[H]
    \centering
    \includegraphics[width=\linewidth]{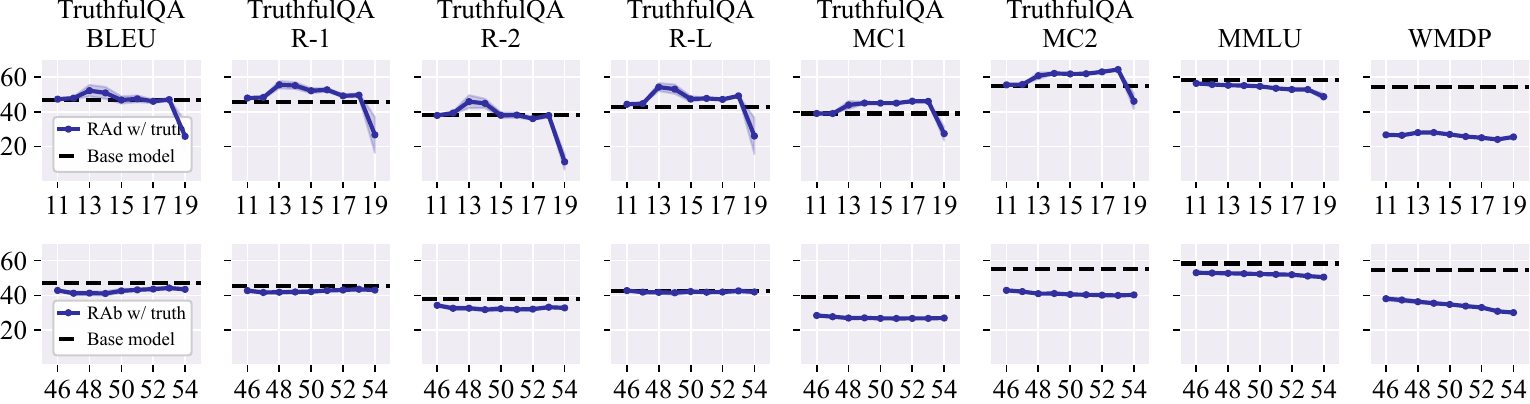}
    \caption{Performance of RAd and RAb Zephyr-7B model on truthfulness across scaling coefficients $c$. The solid lines show the mean performance across five runs, and the shaded regions indicate one standard deviation.}
    \label{fig:varying_coeff}
\end{figure}

\begin{figure}[h]
    \centering
    \includegraphics[width=\linewidth]{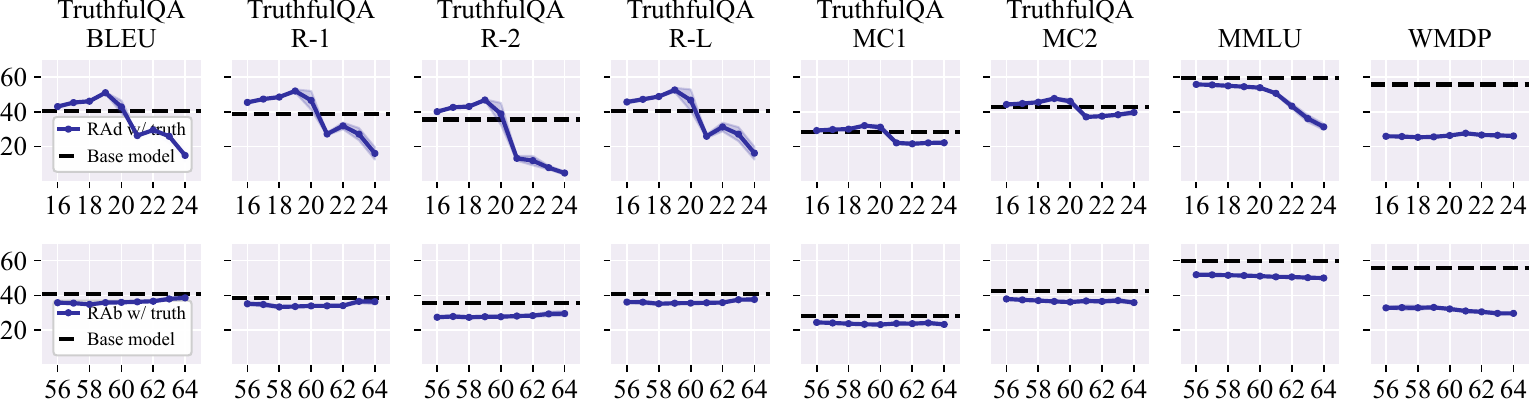}
    \caption{Performance of RAd and RAb Mistral-7B model on truthfulness across scaling coefficients $c$. The solid lines show the mean performance across five runs, and the shaded regions indicate one standard deviation.}
    \label{fig:varying_coeff_mistral}
\end{figure}

\section{Proof of Proposition~\ref{proposition1}}
\label{appendix:theoretical_analysis}
\label{appendix:proposition}
A key component in our analysis is Lévy's Lemma~\citep{milman1986asymptotic, ledoux2001concentration, vershynin2018high}, which states that when a point $\mathbf{x}$ is selected from a high dimensional hypersphere at random and $f(\mathbf{x})$ does not vary too rapidly, then $f(\mathbf{x})$ is highly concentrated around its expected value $\mathbb{E}[f(\mathbf{x})]$ with high probability. 

\begin{lemma} [Lévy's Lemma]
    Suppose $f$: $\mathbb{S}^{d-1} \to \mathbb{R}$ is $L$-lipschitz w.r.t. Euclidean on the unit hypersphere. Then, a point $\mathbf{x}$ is drawn uniformly from $ \mathbb{S}^{d-1}$ at random, for any $\epsilon > 0$, 
    \begin{align}
        \mathbb{P}[|f(\mathbf{x}) - \mathbb{E}[f(\mathbf{x})]| > \epsilon] \leq 2 \exp\left(\frac{-(d-1)\epsilon^2}{2L^2}\right)
    \end{align}
    \label{levy_lemma}
\end{lemma}

We apply Lévy's Lemma to the function $f (\cdot) = \langle\cdot, \bar\lambda_W\rangle$, which yields Proposition~\ref{proposition1}.

\goldba*

\begin{proof}
    For any $\mathbf{u} \in \mathbb{S}^{d-1}$ and $\mathbf{w}\in \mathbb{S}^{d-1}$, if $f(\cdot) = \langle\cdot, \bar\lambda_W\rangle$ then $f$ is $1$-Lipschitz ($L = 1$):
    \begin{align}
        |f(\mathbf{u}) - f(\mathbf{w})| &=|\langle\mathbf{u},\bar\lambda_W\rangle - \langle\mathbf{w},\bar\lambda_W\rangle| \\&= |\langle\mathbf{u}-\mathbf{w}, \bar\lambda_W\rangle| 
    \end{align}
    By the Cauchy-Schwarz inequality:
    \begin{align}
        |f(\mathbf{u}) - f(\mathbf{w})| &\leq ||\bar\lambda_W||_2 ||\mathbf{u} - \mathbf{w}||_2 \\&= 1\cdot||\mathbf{u} - \mathbf{w}||_2 
    \end{align}
    
    Expectation of $f(\mathbf{u})$: $\mathbb{E}[f(\mathbf{u})] = \mathbb{E}_{\mathbf{u}\sim \mathbb{S}^{d-1}}\langle \mathbf{u},\bar\lambda_W\rangle = \langle \mathbb{E}_{\mathbf{u}\sim \mathbb{S}^{d-1}}[\mathbf{u}],\bar\lambda_W\rangle = 0$. By Lévy's Lemma, we obtain
    \begin{align}
        \mathbb{P}[|f(\mathbf{u}) - \mathbb{E}[f(\mathbf{u})]| &> \epsilon] \leq 2 \exp\left(-\frac{(d-1)\epsilon^2}{2L^2}\right) \\
        \mathbb{P}[|\langle\mathbf{u},\bar\lambda_W\rangle| &\leq \epsilon] \geq 1-2 \exp\left(-\frac{(d-1) \epsilon^2}{2}\right)
    \end{align}
    For this bound to be non-trivial, it requires $1 - 2\exp\left(-\frac{(d-1)\epsilon^2}{2}\right) > 0$. Hence, for any $\epsilon > \sqrt{\frac{2 \ln2}{d-1}}$, we conclude $\mathbb{P}[|\langle\mathbf{u},\bar\lambda_W\rangle| \leq \epsilon] \geq 1-2 \exp\left(-\frac{(d-1) \epsilon^2}{2}\right)$.
\end{proof}

\section{Limitation and future work}
\label{appendix:limitation}
% \textbf{Mechanism of RAd \& RAb: superficial masking or true erasure?} Despite years of research, the ``forgetting'' mechanism in LLM unlearning remains a controversial subject~\citep{liu2025rethinking, cooper2025machine, hu2025unlearning, yu2025impossibility, triantafillou2026your, mavrothalassitis2026ascent, huu2025improving}. 
% In our study, we acknowledge that RAd and RAb may fall into the category of \textit{superficial masking}, \textit{i.e.,} redirecting the forget-samples' representations to suppress the model's ability to access target knowledge and may leave the target knowledge preserved, rather than achieving \textit{true erasure}, \textit{i.e.,} surgically erasing target knowledge from the model's weights. 
% % For applications demanding verifiable forgetting, relying on representational shifts that can be reversed by knowledge recovery attacks is fundamentally insufficient~\cite{lucki2025an}. Truly erasing knowledge remains a challenging problem for the LLM unlearning community.
% Nevertheless, we clarify that our central objective is not to claim true erasure. Rather, our work aims to reveal and characterize a previously unexplored phenomenon: that the choice of target vector systematically elicits controllable side behaviors and capabilities corresponding to the concept. We hope this characterization provides a useful lens for future work on principled unlearning methods.

\textbf{Model scale.} Due to computational constraints, experiments are conducted on 7-8B models with updates to a subset of model components, which risks missing interesting observations for larger models. 

\textbf{Reliance on the Linear Representation Hypothesis.} The effectiveness of RAd and RAb rests on the Linear Representation Hypothesis that concepts are encoded linearly and can be effectively approximated as a one-dimensional vector. While empirically supported by current literature, this may not hold for complex, multi-aspect, or highly entangled concepts where the linear approximation is overly simplistic.

\textbf{Hyperparameter search.} Hyperparameters are selected via grid search. Developing principled approaches to selecting the layer and hyperparameters is left for future work.

\textbf{Future work.} Investigating whether other LLM unlearning methods, such as GA, NPO, SimNPO, and DPO, exhibit similar side effects is another interesting direction for future work.
\end{document}